\documentclass[10pt]{article} 

\usepackage[preprint]{tmlr}

\usepackage[T1]{fontenc}
\usepackage{lmodern}
\usepackage{graphicx}
\usepackage{booktabs}
\usepackage{array}
\usepackage{amsmath,amssymb,amsfonts}
\usepackage{algorithm}
\makeatletter\let\tmlrAND\AND\let\AND\@undefined\makeatother
\usepackage{algorithmic}
\usepackage{multirow}
\usepackage{placeins}
\usepackage{xcolor}
\usepackage{amsthm}
\newtheorem{theorem}{Theorem}
\newtheorem{lemma}{Lemma}
\newtheorem{proposition}{Proposition}
\newtheorem{corollary}{Corollary}
\theoremstyle{definition}
\newtheorem{assumption}{Assumption}
\theoremstyle{remark}
\newtheorem{remark}{Remark}
\usepackage[hidelinks]{hyperref}
\usepackage{url}
\usepackage[capitalize]{cleveref}

\newcommand{\eg}{\emph{e.g.}}
\newcommand{\ie}{\emph{i.e.}}

\newcommand{\etal}{\emph{et al.}}
\newcommand{\scs}[1]{\textsuperscript{\scriptsize #1}} 

\crefname{section}{Sec.}{Secs.}
\crefname{table}{Tab.}{Tabs.}
\crefname{figure}{Fig.}{Figs.}
\crefname{equation}{Eq.}{Eqs.}
\crefname{algorithm}{Alg.}{Algs.}
\crefname{proposition}{Proposition}{Propositions}
\crefname{theorem}{Theorem}{Theorems}
\crefname{lemma}{Lemma}{Lemmas}
\crefname{assumption}{Assumption}{Assumptions}
\crefname{corollary}{Corollary}{Corollaries}

\title{The Scissors Effect: When Resize-Based Input\\
Diversity Helps or Hurts Transfer Attacks}

\author{%
  \name Yuhang Jiang \email jyhtjtj@gmail.com \\
  \addr University of Trento
  \tmlrAND
  \name Xiaojing Chen \email chenxiaojing0909@ahu.edu.cn \\
  \addr Anhui University
}

\def\month{09}
\def\year{2026}
\def\openreview{\url{https://openreview.net/forum?id=b4pCcgJM0M}}

\begin{document}

\maketitle

\begin{abstract}
Input Diversity (DI), a random resize and pad applied at each attack iteration,
is a near-default ingredient of transfer-based attacks, widely assumed to improve
transferability. We show this assumption is regime-dependent and, for
adversarially trained surrogates, often reversed. Holding the attack fixed and
varying only the surrogate, raising the DI probability improves transfer from
standard surrogates but degrades it from robust ones: the two response curves
separate like a pair of scissors, a pattern we call the \emph{Scissors Effect}.
On ImageNet, blind DI costs a robust source $10.3$ percentage points of attack
success across four
architecturally diverse targets; the effect is several times smaller at
$32{\times}32$. Direct measurement supports a bias--variance account: DI
displaces the gradient by a comparable amount on both groups but reduces its
variance only where the gradient is noisy, and robust surrogates have little
noise left to average away. A gradient-consistency probe, frozen and hashed
before the runs, predicts the sign of the effect on seven unseen surrogates, and
we report where it fails alongside where it works. The practical consequence
holds independently of the mechanism: leaving DI enabled by default understates
the attack a robust surrogate can mount, and so overstates the robustness of the
model being evaluated. Code: \url{https://github.com/Avalon-S/ScissorsEffect}.
\end{abstract}

\section{Introduction}
\label{sec:intro}

Constructing transferable adversarial examples, perturbations that fool unseen
target models, is a central tool for evaluating the safety of black-box deep
learning systems. Since the discovery of
transferability~\citep{szegedy2013intriguing, goodfellow2015explaining, papernot2016transferability},
many techniques have been proposed to enhance it. Among them, Input Diversity
(DI)~\citep{xie2019improving}, which applies random resizing and padding at each
attack iteration, is now enabled by default in most of them. Later methods such
as Admix~\citep{wang2021admix}, SSA~\citep{long2022frequency} and
SIA~\citep{wang2023structure} stack further input transformations on top of it,
under the working assumption that ``more diversity is always better.''

In this paper we show that this assumption is regime-dependent, and for robustly
trained surrogates it is often reversed. We systematically study transferability
from robust surrogates (\ie, models trained with adversarial
training~\citep{madry2018towards} or related robustification) and uncover a
phenomenon we call the \emph{Scissors Effect} (\cref{fig:scissors}). Throughout,
a \emph{difference} between two attack success rates is reported in percentage
points (pp), and \% is reserved for rates themselves and for relative changes.
Holding the
attack and target fixed and varying only the diversity probability $p$:
\begin{itemize}
    \item For \textbf{standard} surrogates, increasing diversity improves
    transferability, which is what the default assumption predicts. The gain is
    large on ImageNet ($+14.9$ pp at $p{=}1$) and small at $32{\times}32$, where
    it appears at DI's conventional transform probability and has gone by
    $p{=}1$.
    \item For \textbf{robust} surrogates, increasing diversity ($p\to1$) often
    degrades transferability, in many cases performing worse than an
    attack with no diversity at all.
\end{itemize}
The two response curves separate in opposite directions as $p$ increases, like
the blades of a pair of scissors.

\begin{figure}[t]
\centering
\includegraphics[width=1.0\linewidth]{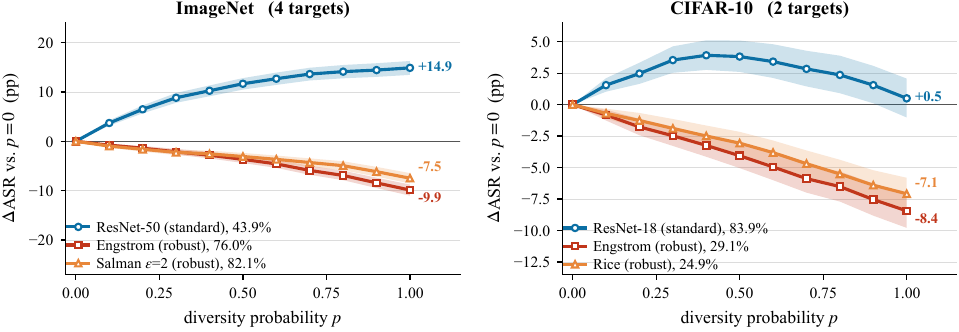}
\caption{The Scissors Effect. Change in transfer ASR relative to $p{=}0$; legend
entries give each surrogate's ASR at $p{=}0$. Both panels share one DI
implementation, $N{=}1{,}000$ images, 5 seeds, $r{=}0.9$, the all-images basis,
and sources disjoint from every target. Bands are 95\% image-paired bootstrap
intervals. \emph{Left:} ImageNet, four standard targets (InceptionV3, ViT-B/16,
Swin-B, ConvNeXt-B). \emph{Right:} CIFAR-10, two standard targets (VGG-16,
DenseNet-121); the pool's robust target is excluded and reported in
\cref{tab:blade_gap}. This sweep is a separate run from \cref{tab:imagenet}; the
two agree to within $0.4$ pp.}
\label{fig:scissors}
\end{figure}

\textbf{Why this matters.}
Robust pretrained models are increasingly used as surrogates because of their
stronger, perceptually aligned features~\citep{salman2020adversarially}. In that
setting, leaving DI on by default quietly weakens the attack: on ImageNet it
costs $10.3$ pp of ASR on average for the robust source. The consequence is not
confined to attackers. A defense evaluated with a robust surrogate and DI left on
is credited with robustness that the attack, not the defense, gave away.

That robustness changes a surrogate's gradient regime, and hence its transfer
behavior, is itself known: \citet{springer2021little} show that ``a little
robustness'' improves targeted transfer, and \citet{zhang2024why} explain how
robustness strength trades off model smoothness against gradient similarity. Our
contribution is not to re-discover that robustness matters, but to characterize
its DI-specific manifestation: which component of input diversity is responsible,
how the effect emerges as robustness increases, and a single gradient-geometry
probe that indicates when DI should be disabled.

\noindent\textbf{Three claims, held to different standards.} This paper makes an
empirical claim, proposes a mechanism for it, and derives a decision rule from
that mechanism. The three are not equally well supported, and we separate them
rather than let the first carry the other two.

\begin{description}
\item[C1 (empirical).] Resize-based input diversity improves transfer from
standard surrogates and degrades it from adversarially trained ones. We establish
this on ImageNet and CIFAR-10; on CIFAR-100, registered in advance as a scope
test, the effect is within $1.5$ pp of zero and the direction does not hold
(\cref{sec:heldout}). This is the paper's primary claim.
\item[C2 (mechanism).] The degradation arises because DI's resize step displaces
the gradient on every surrogate while buying a variance reduction only where the
gradient is noisy; robust gradients have little noise left to average away.
We offer this as a bounded candidate, not as an established mechanism.
\item[C3 (decision rule).] Local Gradient Consistency (LGC), a cheap probe
measured on the clean input, separates the two regimes well enough to decide
whether to enable DI. We implement it as CG-DI, a
training-free switch, and treat it as a falsification test of C2's regime
picture, reporting where it fails.
\end{description}

\Cref{tab:claims} states, for each claim, which result supports it directly,
which is only consistent with it, and what remains assumed. The scope of C1 is
the $L_\infty$ adversarially trained surrogates, two datasets and DI
implementation described in \cref{sec:setup}; we do not claim it for randomized
smoothing, $L_2$ robustness, or transform families we did not test.

\begin{table}[t!]
\caption{The three claims and the status of the evidence for each. ``Directly
supports'' means the experiment measures the quantity the claim is about;
``consistent with'' means it would look the same under other explanations;
``assumed or open'' is what the claim needs that we have not established.}
\label{tab:claims}
\centering
\small
\renewcommand{\arraystretch}{1.15}
\begin{tabular}{@{}>{\raggedright\arraybackslash}p{0.10\linewidth}>{\raggedright\arraybackslash}p{0.28\linewidth}>{\raggedright\arraybackslash}p{0.26\linewidth}>{\raggedright\arraybackslash}p{0.26\linewidth}@{}}
\toprule
& Directly supports & Consistent with & Assumed or open \\
\midrule
C1\newline\emph{empirical} &
Fig.~\ref{fig:scissors} and \cref{tab:imagenet} (8/8 source--target intervals
exclude zero); the ten-attack panel under dependence-aware statistics
(\cref{tab:panel_stats}); the controlled $\epsilon$-sweep
(\cref{tab:eps_sweep}); the cross-recipe panel (\cref{sec:recipe_panel}) &
The blade-gap ordering across datasets (\cref{tab:blade_gap}), which does not
isolate resolution from target type or model pool &
Generality beyond $L_\infty$ AT surrogates and the two datasets tested;
on CIFAR-100 only the ordering of the group means survives, not a sign reversal:
both are positive at $r{=}0.9$ \\
\midrule
C2\newline\emph{mechanism} &
The bias--variance decomposition measured per surrogate
(\cref{tab:biasvar}); the $2\times2$ factorial separating resize from
translation, with resize about four times the larger (\cref{tab:transform}) &
Clean-point gradient alignment (\cref{tab:alignment}), whose robust column is a
consistent direction rather than a per-pair effect; the frequency analysis; the
trajectory association (\cref{tab:trajectory}) &
That a fixed linear operator describes the DI gradient. It does not: the
linearization residual is $1.04$--$1.28$, so \cref{prop:scissors} shows a
crossover is \emph{possible} under a simplified model, not that it is the
operating mechanism \\
\midrule
C3\newline\emph{rule} &
The held-out test on seven unseen ImageNet surrogates spanning two unseen
training recipes, with the predictions frozen and locally hashed before the runs
(\cref{tab:heldout}) &
The LGC--effect correlation (\cref{tab:correlation}), which is a two-group
separation rather than a graded relation &
Known failures: on CIFAR-100, where the effect itself is at most $1.5$ pp, the rule
is correct on 3 of 7 and we claim no dataset transfer; per-image routing costs
$6.4$--$7.4$ pp on two surrogates; the CIFAR-10 threshold classifies 2 of 4
standard surrogates as robust-like; Salman $\epsilon{=}8$; and no measured
advantage over reading training provenance \\
\bottomrule
\end{tabular}
\end{table}

\section{Related Work}
\label{sec:related}

\subsection{Transfer-Based Attacks and Input Diversity}
\label{sec:rw_transfer}
Adversarial attacks generate perturbations $x_{adv}=x+\delta$ that maximize a
loss $\mathcal{L}(f(x_{adv}),y)$ subject to $\|\delta\|_\infty\le\epsilon$. Basic
gradient-based attacks such as FGSM and PGD~\citep{goodfellow2015explaining,
madry2018towards} achieve high white-box success but overfit to the surrogate's
local curvature, limiting transferability. To address this, MI-FGSM
~\citep{dong2018boosting} adds momentum to stabilize update directions, and
NI-FGSM~\citep{lin2020nesterov} couples Nesterov acceleration with scale
invariance. A second, highly successful family augments the input at each
iteration. Input Diversity (DI-FGSM)~\citep{xie2019improving} applies, with
probability $p$, a random \emph{resize} to a fraction $r\in(0,1]$ of the original
side length, followed by a random padding back to the original size that acts as
a \emph{translation}. Smaller $r$ is more aggressive, so the attack
optimizes the expected loss over a distribution of transformed inputs.
\citet{xie2019improving} themselves sweep $p$ and report white-box success
falling as black-box success rises. That is a trade-off between two evaluation
settings for one attack, whereas ours is between two surrogate regimes with both
branches measured as black-box transfer ASR, and their sweep does not group
surrogates by training regime. TI-FGSM
~\citep{dong2019evading} convolves the gradient with a translation-invariance
kernel. Building on DI, Admix~\citep{wang2021admix} mixes in other images,
Spectrum Simulation (SSA)~\citep{long2022frequency} augments in the frequency
domain, and SIA~\citep{wang2023structure}, GRA~\citep{zhu2023boosting},
PGN~\citep{ge2023boosting}, BSR~\citep{wang2024boosting}, and
AdaMSI~\citep{long2024convergence} stack progressively richer structural
transformations. Recent work further optimizes transferability via Bayesian
priors~\citep{fan2025transferable}, model quantization~\citep{yang2024quantization},
or logit calibration with truncated feature mixing~\citep{weng2025improving}.
Across this literature, input transformations are treated as broadly beneficial
regardless of the surrogate's training regime; we show that this breaks down once
the surrogate is robust, with \emph{resize} the dominant component of the harm.

\subsection{Input-Transformation Defenses}
\label{sec:rw_defense}
Input transformations also appear on the defense side: methods such as JPEG
compression and resizing aim to purify adversarial perturbations
~\citep{guo2018countering}. Paradoxically, attacking such defenses motivated
making the attack itself transformation-robust, which is part of why DI and its
descendants were adopted so widely. Our findings caution that this default is
surrogate-dependent: when a robust surrogate is used to evaluate a
defense, leaving DI on can understate attack strength.

\subsection{Robustness, Gradient Geometry, and Transferability}
\label{sec:rw_robust}
Adversarial Training (AT)~\citep{madry2018towards} and variants such as TRADES
~\citep{zhang2019theoretically} and ensemble AT~\citep{tramer2018ensemble} are
the standard route to robustness. AT models learn perceptually aligned features
and ``interpretable'' gradients~\citep{ilyas2019adversarial,
salman2020adversarially}. A line of work directly studies how robustness affects
transferability. \citet{springer2021little} show that surrogates with a small
amount of robustness yield better targeted transfer, and
\citet{zhang2024why} give a more complete account: robustness strength trades off
model \emph{smoothness} against \emph{gradient similarity}, and they conjecture that
the data-distribution shift induced by stronger adversarial training degrades
gradient similarity.

Our work operates in the same conceptual space, in that robustness changes surrogate
gradient properties and those properties govern transfer behavior, but makes a
different, complementary cut. (i)~\textbf{Object of study.} Prior work explains
\emph{base} transferability as a function of robustness; we study the
orthogonal axis of how a specific transfer-enhancing transformation (DI,
and within it resize) interacts with that regime. In our controlled
$\epsilon$-sweep (\cref{sec:eps_sweep}) the base MI-FGSM transferability
independently reproduces the ``little-robustness'' peak of
\citet{springer2021little, zhang2024why}, and the DI effect rides on top of it.
(ii)~\textbf{Metric.} LGC measures gradient stability under \emph{input-space}
perturbations at the scale of DI's transformation magnitude; it is closest to a
localized form of gradient similarity, and is complementary to Zhang
\etal's parameter-space smoothness rather than a replacement. We therefore frame
the Scissors Effect as a DI-specific specialization of the broader
robustness--transferability trade-off, not a standalone re-explanation of it.
While prior work~\citep{wang2020high, xu2019frequency} links neural networks to
frequency-dependent behavior, we explicitly measure the frequency content and
local consistency of gradients and tie them to the efficacy of input
diversity.

\section{The Scissors Effect: Discovery and Verification}
\label{sec:scissors}

\subsection{Experimental Setup}
\label{sec:setup}

\textbf{Threat model.} We study the standard black-box \emph{transfer} setting:
the attacker has white-box access to a \emph{surrogate} (source) model but
no access to the target: no parameters, no gradients, and no queries. The
attacker crafts an $L_\infty$-bounded perturbation on the surrogate and transfers
it to the unseen target; success is measured by the target's misclassification.
Attacks are untargeted unless stated otherwise (the targeted case is in
\cref{sec:targeted}). We ask whether, given such a surrogate, adding input
diversity to the surrogate-side optimization helps or hurts the transferred
attack. The surrogate's training regime, which the attacker may not even know, is
therefore the central variable.

\textbf{Datasets.} We use 5{,}000 test images from CIFAR-10~\citep{krizhevsky2009learning} and
5{,}000 validation images from ImageNet~\citep{deng2009imagenet}. Sample sizes vary by experiment and are stated in each
table/figure caption; analyses use $N{=}500$--$1{,}000$ images with multiple
seeds, drawn from those pools. In every table, bold marks the best value within
its comparison group, and a caption says so explicitly wherever it marks
something else.

\textbf{Models.} We consider Standard and Robust source (surrogate) models on
each dataset. \emph{CIFAR-10:} 4 standard (ResNet-18/50~\citep{he2016deep},
VGG16~\citep{simonyan2014very}, \mbox{DenseNet121~\citep{huang2017densely}}), all
trained by us on CIFAR-10 and reaching $93.8$--$95.4\%$ clean accuracy, and 9 robust (Engstrom
~\citep{engstrom2019robustness}, Rice~\citep{rice2020overfitting},
Gowal~\citep{gowal2020uncovering}, Carmon~\citep{carmon2019unlabeled},
\mbox{Peng~\citep{peng2023robust}}, \mbox{Wang~\citep{wang2023better}}, TRADES/Zhang
~\citep{zhang2019theoretically}, Sehwag$\times2$~\citep{sehwag2022proxy}).
\emph{ImageNet:} 6 standard (ResNet50, ViT-B/16~\citep{dosovitskiy2021image},
DenseNet121, InceptionV3,
Swin-B, ConvNeXt-B~\citep{szegedy2016rethinking, liu2021swin, liu2022convnet})
and 3 robust (Engstrom, Salman~\citep{salman2020adversarially},
Mo~\citep{mo2022adversarial}).

\textbf{Targets, and their separation from sources.} Every result in this paper
uses a target that does not appear in its own source pool, which on CIFAR-10 is
easy to violate because the standard RobustBench defended models are natural
choices for both roles. On ImageNet the targets are
four architectures held out of the source pool (InceptionV3, ViT-B/16, Swin-B,
ConvNeXt-B). On CIFAR-10 the correlation
pool of 13 surrogates transfers to RobustBench's naturally trained
WRN-28-10~\citep{zagoruyko2016wide, croce2021robustbench},
\cref{fig:scissors} transfers to our VGG-16 and DenseNet-121, and the CG-DI
evaluation transfers to that WRN-28-10 and to Carmon, neither of which appears in
that evaluation's own surrogate pool. Our released scripts abort on a source--target overlap rather than
reporting the resulting white-box row as transfer.
The full per-model architecture and training-recipe
table is given in \cref{tab:arch} (\cref{sec:arch}). Our robust
sources span two architecture families (ResNet-50 and ViT-B) and two training
recipes (PGD-AT and ViT-AT), so the effect cannot be attributed to a single
architecture. Throughout, ``robust'' denotes specifically $L_\infty$
adversarially trained surrogates (PGD-AT, TRADES, and their variants); we do not
test randomized-smoothing or $L_2$-robust models, whose gradient geometry could
sit elsewhere on the LGC axis, and we scope our claims to the $L_\infty$-AT family
accordingly.

\textbf{Source--target overlap.} Our ImageNet analyses deliberately use disjoint
architecture families for source and target (\eg, ResNet50 $\to$
ViT-B/Swin-B/ConvNeXt-B; Mo2022 ViT-B $\to$ ConvNeXt-B), so the effect is not an
artifact of family similarity. We verify this directly in
\cref{sec:self_transfer} with a same-family ($\mathrm{RN50}\to\mathrm{RN50}$)
control. Ensemble surrogates are out of scope and left to future work.

\textbf{Metric and attack.} We report Attack Success Rate
$\mathrm{ASR}=1-\text{robust accuracy}$, and write
$D=\mathrm{ASR}(\text{with DI})-\mathrm{ASR}(\text{without DI})$ in percentage
points for the DI effect. Unless otherwise stated we use MI-FGSM
~\citep{dong2018boosting} with $T{=}10$ iterations, step size $\alpha{=}2/255$,
and budget $\epsilon{=}8/255$ (CIFAR-10) or $\epsilon{=}16/255$ (ImageNet).

\textbf{How to read the tables.} ASR can be computed over all evaluated images,
over the subset the target classifies correctly, or over the subset both source
and target classify correctly. These differ by whole percentage points, so every
table names the basis it uses and no claim mixes them. Which basis is
appropriate depends on the question rather than on the kind of table: a claim
about how an attack performs is stated over all images, while one about gradient
geometry excludes images the models get wrong anyway, since there is nothing
there to compare. Where a number appears on a basis other than its table's, as
in the uncertainty analysis for \cref{tab:imagenet}, the text says so at the
point of use. Throughout, $N$ is the number of images
and ``seeds'' the number of end-to-end re-runs with fresh transform draws; seeds
capture transform randomness only, so we report uncertainty over images
separately rather than reading it off the seed spread. Intervals are 95\%
bootstrap intervals whose resampling unit each caption states, and where a claim
is about targets in general we resample targets, which gives the wider interval.
Tables with no interval column report point estimates, and we do not read
significance into differences within them.

\subsection{The Scissors Effect}
\label{sec:scissors_discovery}

We vary the input diversity probability $p\in[0,1]$ and measure transfer ASR.
\Cref{fig:scissors} shows the two surrogate types separating in opposite
directions on both datasets, under one DI implementation, $N{=}1{,}000$ images
and 5 seeds. On ImageNet the standard source (ResNet-50, $43.9\%$ at $p{=}0$)
gains $+14.9$ pp by $p{=}1$, while the robust sources lose $-9.9$ (Engstrom,
$76.0\%$) and $-7.5$ (Salman $\epsilon{=}2$, $82.1\%$). On CIFAR-10 the robust
sources lose $-8.4$ (Engstrom, $29.1\%$) and $-7.1$ (Rice, $24.9\%$), while the
standard source (ResNet-18, $83.9\%$) ends flat at $+0.5$. The CIFAR-10 standard
curve is the one case where the endpoint misleads: it rises to $+3.9$ at
$p{=}0.4$ and comes back down, so a measurement taken only at $p{=}1$ reads as no
effect where the conventional operating point shows a real one. The ImageNet
curves are monotone and have no such turn.

\Cref{tab:imagenet} breaks the ImageNet result down by target. For the robust source
(Engstrom), DI degrades transfer to every target architecture: ViT ($-6.8$ pp),
Swin ($-13.1$ pp), and ConvNeXt ($-16.2$ pp). For the standard source
(ResNet50) DI yields large gains: ViT ($+14.3$ pp), Swin ($+13.2$ pp), ConvNeXt
($+16.3$ pp). On average, blind DI costs the robust source $10.3$ pp of ASR.

To keep the magnitude in perspective: robust sources transfer so well to begin
with (avg.\ $76.0\%$ vs.\ $44.0\%$ for the standard source) that even after this
harm they remain the stronger surrogate ($65.7\%$ with blind DI vs.\ the
standard source's $58.6\%$). DI does not make a robust surrogate weak; it erodes
the lead of the strongest surrogate. The consequence for defense evaluation is
still real: leaving DI on by default underestimates the attack a
robust surrogate can mount, and thus overestimates the target's robustness.

\begin{table}[t!]
\caption{Transfer ASR (\%) on ImageNet ($N{=}1{,}000$, 5 seeds, all-images basis).
DI-FGSM is at $p{=}1$ and reported as mean$\pm$std over seeds; MI-FGSM
($p{=}0$) is deterministic.}
\label{tab:imagenet}
\centering
\small
\begin{tabular}{llccccc}
\toprule
Source & Method & Avg & IncV3 & ViT & Swin & ConvNeXt \\
\midrule
Robust & MI-FGSM & \textbf{76.0} & \textbf{86.0} & \textbf{82.3} & \textbf{68.6} & \textbf{66.9} \\
 & DI-FGSM & 65.7$\pm$0.3 & 80.9$\pm$0.1 & 75.5$\pm$0.3 & 55.5$\pm$0.8 & 50.7$\pm$0.9 \\
\midrule
Standard & MI-FGSM & 44.0 & 51.9 & 34.9 & 40.9 & 48.3 \\
 & DI-FGSM & \textbf{58.6$\pm$0.5} & \textbf{66.5$\pm$0.8} & \textbf{49.2$\pm$1.1} & \textbf{54.1$\pm$0.7} & \textbf{64.6$\pm$0.7} \\
\bottomrule
\end{tabular}
\end{table}

The pattern is not limited to CNN targets: ViTs and hybrid architectures (Swin,
ConvNeXt) exhibit it too. One interpretation is that robustness aligns gradients
with semantic directions that transfer across inductive biases, making DI's
spatial distortion consistently harmful for robust surrogates; we offer this as a
reading, not a measured claim, since our experiments establish the consistent
\emph{direction} across targets but not a shared-feature mechanism across
architectures.

\textbf{Uncertainty.} \Cref{tab:imagenet} reports means over seeds, which says
nothing about whether an individual source--target effect is resolved. We
therefore logged per-image outcomes for this table. The effects below are
computed on the images each target classifies correctly, so their magnitudes run
larger than the all-images differences in \cref{tab:imagenet} and should be read
against each other rather than against the table. Resampling images, all eight
source--target intervals exclude zero: the robust source ranges from $-6.13$
$[-8.13,-4.21]$ on InceptionV3 to $-18.97$ $[-21.56,-16.43]$ on ConvNeXt-B, and
the standard source from $+15.98$ $[+13.62,+18.45]$ to $+19.26$ $[+16.53,+22.05]$.
Treating targets rather than images as the resampling unit, which is the right
unit for a claim about targets in general, a cluster bootstrap over the four gives
$-12.07$ $[-16.94,-7.20]$ and $+17.96$ $[+16.69,+19.06]$. The robust-side interval
is much the wider of the two because the effect varies by target from $-6$ to
$-19$, so we claim a consistent direction across targets and not a common
magnitude.

\textbf{Scope on CIFAR-10.} On CIFAR-10 the effect is smaller and depends on how
much transform is applied. At the $p{=}1$ used in this table the standard side is
flat, which is why a measurement taken only there reads as a one-sided effect. At
DI's conventional $p{=}0.5$ both sides move and every interval excludes zero
(\cref{tab:convp}). The magnitudes are a few points rather than the ten to twenty
seen on ImageNet, but the sign structure is the same.

\begin{table}[t!]
\caption{CIFAR-10 at DI's conventional transform probability ($p{=}0.5$,
$r{=}0.9$, $N{=}1{,}000$, 5 seeds, all-images basis). Each surrogate is
transferred to three targets it does not appear among, and $D$ is in pp, averaged over
them; intervals are 95\% image-paired bootstrap. Compare $p{=}1$, where the
standard side ends flat (\cref{fig:scissors}).}
\label{tab:convp}
\centering
\small
\begin{tabular}{llcc}
\toprule
Surrogate & Type & $D$ ($p{=}0.5$) & 95\% CI \\
\midrule
ResNet-18 & Standard & $+2.83$ & $[+1.79, +3.87]$ \\
ResNet-50 & Standard & $+2.14$ & $[+1.38, +2.91]$ \\
VGG-16 & Standard & $+3.85$ & $[+2.53, +5.20]$ \\
DenseNet-121 & Standard & $+3.54$ & $[+2.61, +4.51]$ \\
\midrule
Engstrom & Robust & $-4.32$ & $[-5.23, -3.43]$ \\
Rice & Robust & $-3.63$ & $[-4.49, -2.78]$ \\
\bottomrule
\end{tabular}
\end{table}

\subsection{How far apart the blades are, and what that ordering does not settle}
\label{sec:blade_gap}

The natural summary statistic is the distance between the two blades, the mean
effect on standard surrogates minus the mean effect on robust ones.
\Cref{tab:blade_gap} collects it for the three datasets on a single evaluation
basis. The ordering is the same in every combination of transform probability,
surrogate pool and resize rate we measured: ImageNet, then CIFAR-10, then
CIFAR-100. Its \emph{size} is not. The ImageNet-to-CIFAR-10 ratio runs from
$2.0\times$ at $p{=}0.5$ on matched pools to $3.2\times$ when CIFAR-10 is scored
on the thirteen-surrogate pool at $p{=}0.8$, so we report an ordering and not a
scaling law. The ImageNet row differs from \cref{tab:imagenet}'s $24.9$ for the
same reason: that table averages the robust side over Engstrom alone, and
$23.6$ here is the same quantity measured on the pool of \cref{fig:scissors}.

\begin{table}[t!]
\caption{Distance between the two blades. $D=\mathrm{ASR}(\mathrm{DI})-\mathrm{ASR}(\text{no DI})$
in percentage points, averaged over targets within a surrogate and then over
surrogates within a type; gap $=\bar D_{\text{std}}-\bar D_{\text{rob}}$.
All entries use the all-images basis, $N{=}1{,}000$, 5 seeds and $r{=}0.9$
unless stated. Rows 1--2 share a script and operator, so they are directly
comparable; rows 3--5 use their own pools and are not.}
\label{tab:blade_gap}
\centering
\small
\setlength{\tabcolsep}{4pt}
\begin{tabular}{llllcc|ccc}
\toprule
& & & & \multicolumn{2}{c|}{$p{=}0.8$} & \multicolumn{3}{c}{gap} \\
Dataset & Res. & Pool (std/rob) & Targets & $\bar D_{\text{std}}$ & $\bar D_{\text{rob}}$ & $p{=}0.5$ & $p{=}0.8$ & $p{=}1$ \\
\midrule
ImageNet & 224 & 1 / 2 & 4 standard & $+14.1$ & $-5.9$ & 15.0 & 20.0 & 23.6 \\
CIFAR-10 & 32 & 1 / 2 & 2 standard & $+2.4$ & $-6.0$ & 7.4 & 8.4 & 8.3 \\
\midrule
CIFAR-10$^{\dagger}$ & 32 & 4 / 9 & 1 standard & $-0.7$ & $-7.0$ & --- & 6.3 & --- \\
CIFAR-100 & 32 & 1 / 6 & 3 robust & $+0.9$ & $+0.3$ & --- & 0.6 & 0.7 \\
CIFAR-100, $r{=}0.6$ & 32 & 1 / 6 & 3 robust & $+1.6$ & $-0.2$ & --- & 1.9 & 2.6 \\
\bottomrule
\end{tabular}

\vspace{2pt}
{\footnotesize $^{\dagger}$The thirteen-surrogate pool of \cref{tab:correlation},
transferred to a naturally trained WRN-28-10 outside the pool.}
\end{table}

Two features of the table block a resolution-only reading. First, the two
$32{\times}32$ datasets do not agree, and which one is larger depends on the
resize rate: at $r{=}0.9$ CIFAR-10 has the wider gap (6.3 against 0.6), at
$r{=}0.6$ CIFAR-100 does (1.9 against 1.2). Second, the CIFAR-100 pool
transfers only to \emph{robust} targets, while the ImageNet and CIFAR-10 rows
transfer to standard ones, and that difference is not incidental: within the
CIFAR-10 pool the one robust target is exactly where the harm disappears
(at $p{=}0.8$ Engstrom scores $-6.8$ to VGG-16 and $+0.1$ to Carmon), and
including it shrinks the CIFAR-10 gap from 8.4 to 5.9. Target robustness is therefore a candidate
explanation for CIFAR-100's small gap on the same footing as resolution, and our
design does not separate the two. We read the ordering as consistent with a
resolution account, since $32{\times}32$ images leave little room for
resize-induced frequency distortion, and state that the settings differ in
dataset, model pool, target type and transform at once and so do not isolate it.

The gap is also non-monotone in the resize rate. On the thirteen-surrogate pool
it rises from 1.2 at $r{=}0.6$ to a maximum of 6.7 at $r{=}0.95$ and falls to
3.8 at $r{=}0.97$, because the standard side is harmed at both extremes while the
robust side stays near $-7$ throughout. A single rate cannot summarize the
effect, which is why \cref{tab:blade_gap} names one.

\textbf{Robust surrogates are not one class.} The paper's framing treats
adversarially trained surrogates as a group, and the group mean hides a spread
wider than the standard side's entire range. On the thirteen-surrogate pool at
$r{=}0.9$ the nine robust surrogates run from $-4.24$ (Gowal) to $-10.00$
(Sehwag, proxy distribution), mean $-7.02$, s.d.\ $2.27$: a factor of $2.4$
between the smallest and largest harm. The four standard surrogates span
$-3.16$ to $+1.64$ over the same measurement. Every robust surrogate is harmed
and no standard one is harmed by more than the mildest robust one, so the
separation survives; what does not survive is any account of the spread. The
natural one attributes it to the training recipe, with classical PGD-AT harmed
most and extra-data methods least. Against a target outside the pool that
ordering breaks: the largest effect belongs to a proxy-distribution model and the
second largest to a proxy-distribution ResNet-18, while the extra-data models
occupy the bottom four places. We report the spread and offer no account of it.

\subsection{Generalization to Modern Attacks}
\label{sec:modern_attacks}

Does the Scissors Effect extend beyond MI-FGSM? We evaluate 10 methods spanning
2018--2024: 3 classical momentum-based (MI-FGSM, NI-FGSM~\citep{lin2020nesterov},
VMI-FGSM~\citep{wang2021enhancing}) and 7 modern input-transformation methods
(Admix~\citep{wang2021admix}, SSA~\citep{long2022frequency},
SIA~\citep{wang2023structure}, GRA~\citep{zhu2023boosting},
PGN~\citep{ge2023boosting}, BSR~\citep{wang2024boosting},
AdaMSI~\citep{long2024convergence}), implemented via
TransferAttack~\citep{transferattack2023}.\footnote{Our core DI/MI-FGSM
experiments (\cref{tab:imagenet}) use \texttt{torchattacks}
~\citep{torchattacks2020}; TransferAttack is used only for the modern attacks
here. The two libraries differ in default iteration count, step size, and input
pre-processing, so \emph{absolute} ASR is not comparable across them: the same
nominal Engstrom$\to$Swin-B MI-FGSM reads $68.6\%$ in \cref{tab:imagenet}
(torchattacks) but $42.2\%$ here (TransferAttack), and the standard source
$40.9\%$ vs.\ $23.9\%$. Only \emph{within}-table $\Delta$ columns, which hold the
library and setup fixed, should be compared across rows; the Scissors direction
is what each table establishes, and it is consistent in both.}

\begin{table}[t!]
\caption{Generalization to 10 attacks (Engstrom / ResNet50 $\to$ Swin-B;
$N{=}1{,}000$, 5 seeds, all-images basis). $\Delta$ is the DI effect $D$ in
percentage points, computed from unrounded values, so it can differ by $0.1$ from
the difference of the two rounded columns. The ten attacks share
images, surrogate and target, so this tally is not ten independent observations;
\cref{tab:panel_stats} analyses it as a dependent panel on a different basis,
and its levels are not comparable to this table's.}
\label{tab:modern}
\centering
\small
\setlength{\tabcolsep}{3pt}
\begin{tabular}{ll|ccc|ccc}
\toprule
& & \multicolumn{3}{c|}{\textbf{Robust}} & \multicolumn{3}{c}{\textbf{Standard}} \\
Method & Venue & Base & +DI & $\Delta$ & Base & +DI & $\Delta$ \\
\midrule
MI-FGSM & CVPR'18 & 42.2 & 40.5 & $-1.7$ & 23.9 & 24.1 & $+0.2$ \\
NI-FGSM & ICLR'20 & 30.7 & 30.0 & $-0.7$ & 23.2 & 23.6 & $+0.4$ \\
VMI-FGSM & CVPR'21 & 42.6 & 42.0 & $-0.6$ & 22.8 & 24.7 & $+1.9$ \\
\midrule
Admix & ICCV'21 & 44.3 & 41.1 & $-3.2$ & 25.6 & 27.3 & $+1.7$ \\
SSA & ECCV'22 & 41.8 & 26.2 & $\mathbf{-15.6}$ & 22.9 & 22.3 & $-0.6$ \\
SIA & ICCV'23 & 46.1 & 40.4 & $-5.6$ & 25.2 & 26.3 & $+1.1$ \\
GRA & ICCV'23 & 37.2 & 36.6 & $-0.6$ & 22.6 & 23.6 & $+1.0$ \\
PGN & NeurIPS'23 & 33.9 & 32.2 & $-1.8$ & 21.8 & 23.2 & $+1.5$ \\
BSR & CVPR'24 & 41.2 & 36.1 & $-5.1$ & 25.5 & 25.8 & $+0.4$ \\
AdaMSI & AAAI'24 & 43.1 & 41.3 & $-1.8$ & 23.7 & 24.1 & $+0.4$ \\
\bottomrule
\end{tabular}
\end{table}

\Cref{tab:modern} shows that DI degrades all 10 methods on the robust surrogate
and benefits 9/10 on the standard one. This is a dependent panel, not a sign
test: the ten attacks share a surrogate, a target and an image set, and overlap
heavily in algorithm family, so they are not ten independent observations and a
binomial null does not apply to the tally. We therefore log per-image outcomes,
repeat the panel against a second target, and analyse it with statistics that
respect the dependence: an image-paired bootstrap per attack, a random-effects
pooling across attacks, and an image-block sign-flip permutation test.

\begin{table}[t!]
\caption{The ten-attack panel with dependence taken into account
($N{=}1{,}000$, 5 seeds; ASR over the source-and-target clean-correct subset, effects in pp).
Seeds are averaged per image, so the resampling unit is the image throughout.
Per-attack intervals are a $10{,}000$-replicate image-paired bootstrap; the
pooled $D$ is a DerSimonian--Laird random-effects mean weighted by inverse
bootstrap variance plus the estimated between-attack variance. $\bar\rho$ is the
mean pairwise correlation of the per-image effect vectors, and the effective
count is $10/(1+9\bar\rho)$. The permutation test ($10{,}000$ draws) assigns one
sign per image, shared across all ten attacks.}
\label{tab:panel_stats}
\centering
\small
\begin{tabular}{llcccc}
\toprule
Target & Source & pooled $D$ & 95\% CI & attacks resolved & $\bar\rho$ (eff.\ $n$) \\
\midrule
Swin-B & Engstrom & $-3.92$ & $[-4.81, -3.04]$ & 4/10 & $0.116$ ($4.9$) \\
Swin-B & ResNet-50 & $+0.83$ & $[+0.27, +1.40]$ & 4/10 & $0.127$ ($4.7$) \\
\midrule
ConvNeXt-B & Engstrom & $-6.25$ & $[-7.30, -5.23]$ & 9/10 & $0.150$ ($4.3$) \\
ConvNeXt-B & ResNet-50 & $+1.38$ & $[+0.68, +2.06]$ & 5/10 & $0.152$ ($4.2$) \\
\bottomrule
\end{tabular}
\end{table}

\Cref{tab:panel_stats} gives the dependence-aware result. The two-sided conclusion holds on both
targets, with the permutation test giving $p<10^{-4}$ on three of the four cells
and $p{=}0.003$ on the fourth. The correlation structure costs about half the
nominal sample: ten attacks behave like four to five independent ones. At the
level of individual attacks the picture is weaker and we state it: on the target
reported in \cref{tab:modern} only 4 of 10 are individually resolved on each
side, so the smallest effects ($-0.7$, $-0.6$) are not replications; against
ConvNeXt-B, 9 of 10 are resolved on the robust side.

SSA is the largest robust-side effect and is individually resolved. Our theory
suggests reading it as a frequency-domain regime shift, but a pre-registered test
measured the effective gradient regime under SSA directly and found it moves in
the opposite direction (\cref{sec:ssa}). We therefore report SSA as an unresolved
interaction between two transform families rather than as a prediction of the
model.

\subsection{How the Effect Depends on Robustness: a Controlled \texorpdfstring{$\epsilon$}{epsilon}-Sweep}
\label{sec:eps_sweep}

The contrast above pits ``standard'' against ``robust'' as two clean categories,
but the scientific question is not binary: \emph{how does the DI effect evolve as
surrogate robustness increases?} Multiple robust models differ along many axes at
once (architecture, recipe, budget), so we isolate robustness strength with a
controlled sweep. We fix the architecture (ResNet-50) and training recipe
(PGD-AT) and vary only the training budget $\epsilon_{\text{train}}$, using
Salman \etal's $L_\infty$ checkpoints~\citep{salman2020adversarially} for
$\epsilon_{\text{train}}\in\{0.5,1,2,4,8\}/255$ plus the standard ($\epsilon{=}0$)
model. Target is Swin-B, $N{=}500$, 3 seeds, $\epsilon_{\text{attack}}{=}16/255$.

\begin{table}[t!]
\caption{Controlled robustness-strength sweep (fixed ResNet-50 + PGD-AT, varying
only $\epsilon_{\text{train}}$; target Swin-B, $N{=}500$, 3 seeds, all-images
basis). $D$ is the DI
effect (DI$-$MI) in pp. DI flips from beneficial to harmful at the smallest non-zero
robustness; harm peaks near $\epsilon_{\text{train}}{=}2/255$. The MI-FGSM column
independently reproduces the little-robustness transferability peak of
\citet{springer2021little, zhang2024why}.}
\label{tab:eps_sweep}
\centering
\small
\begin{tabular}{lccc}
\toprule
$\epsilon_{\text{train}}$ ($\times255$) & MI-FGSM & DI-FGSM & $D=$ DI$-$MI \\
\midrule
0 (Standard) & 41.2 & 54.2 & $\mathbf{+13.0}$ \\
0.5 & \textbf{88.4} & 82.9 & $-5.5$ \ \ ($\leftarrow$ crossover) \\
1 & 86.0 & 79.1 & $-6.9$ \\
2 & 78.4 & 65.9 & $\mathbf{-12.5}$ \ (max harm) \\
4 & 62.2 & 51.3 & $-10.9$ \\
8 & 44.4 & 37.8 & $-6.6$ \\
\bottomrule
\end{tabular}
\end{table}

\Cref{tab:eps_sweep} shows the harm is not a binary switch but emerges
gradually: $D$ crosses from $+13.0$ pp to negative at
$\epsilon_{\text{train}}\in(0,0.5]/255$, so DI becomes harmful already in the
``little-robustness'' regime, then deepens to a peak near
$\epsilon_{\text{train}}{=}2/255$ before tapering as overall transferability falls. The MI-FGSM curve
itself peaks at $\epsilon_{\text{train}}\in[0.5,1]/255$ ($88.4\%/86.0\%$),
independently reproducing the little-robustness transferability peak of
\citet{springer2021little, zhang2024why}: our setup operates in the same regime
they study, and the DI effect $D$ is the orthogonal axis we contribute. The same
sweep against a second target (ViT-B) gives the same qualitative picture, with
one difference worth stating: there the sign change is not sharply located. $D$
is $-0.7$ at $\epsilon_{\text{train}}{=}0.5$ and $+0.1$ at $1$, both smaller than
the image-level sampling uncertainty at $N{=}500$ ($\approx\!1.3$ pp for an ASR
near $90\%$), and only from $\epsilon_{\text{train}}{=}2$ onward is $D$ clearly
negative, deepening to $-9.7$ at $8$ (\cref{tab:eps_sweep_vitb},
\cref{sec:eps_vitb}). We therefore locate the transition as ``somewhere at or
below $\epsilon_{\text{train}}{=}2$'' rather than in a specific interval.

\textbf{Across training methods, not just strength.} The sweep isolates
robustness strength at a fixed recipe (Salman ResNet-50, PGD-AT); to check
the effect is not specific to that family, we run a cross-recipe panel of five
robust ImageNet surrogates spanning four adversarial-training recipes and four
architectures: PGD-AT (Engstrom, Salman; ResNet-50), ViT-aware AT (Mo2022;
ViT-B), ARES adversarial training (ConvNeXt-B), and ConvStem AT (Singh; ViT-B).
Each is attacked against two targets (Swin-B and ConvNeXt-B), with
each surrogate's LGC measured alongside ($N{=}500$, 3 seeds;
\cref{tab:recipe_panel}, \cref{sec:recipe_panel}). DI harms every robust
surrogate on all nine cross-family source--target pairs ($D$ from $-4.5$ to
$-18.7$), while the standard surrogate benefits ($+13.0$, $+15.9$); the
\emph{sign} of $D$ tracks the LGC regime, not the recipe or architecture, with every
LGC${>}0.92$ pair harmed and both LGC${<}0.92$ pairs benefiting. The
CIFAR-10 robust pool additionally covers TRADES, semi-supervised,
diffusion-augmented, and proxy-distribution training (\cref{tab:arch}), all
high-LGC and all DI-harmed (\cref{tab:di_ablation}). The Scissors Effect thus
tracks the resulting gradient regime, not the particular objective that
produced it. We read this as a replication of the \emph{sign} and its LGC-tracking
across recipes and a second target, not as a second controlled strength sweep (no
public fixed-backbone $\epsilon$-spectrum exists outside Salman's).

\subsection{Source--Target Overlap: a Same-Family Control}
\label{sec:self_transfer}

To rule out that the Scissors Effect is driven by architecture-family similarity
between source and target, we run a same-family control in which both source and
target are ResNet-50 (3 seeds, $N{=}500$, target ASR on held-out ResNet-50
checkpoints). The direction is governed by the source regime, not by
family overlap. A standard ResNet-50 source benefits from DI even against
ResNet-50 robust targets ($D$ between $+1.1$ and $+1.3$ pp), while a robust
ResNet-50 source is harmed even against a standard ResNet-50 target ($-11.6$ pp).
These mirror the cross-family baseline ($+13.0$ and $-12.5$ pp against Swin-B),
so the effect is a property of the surrogate's gradient geometry rather than of
source--target similarity. Full
numbers are in \cref{tab:self_transfer} (\cref{sec:self_transfer_app}).

\section{Mechanism Analysis: Why Does This Happen?}
\label{sec:mechanism}

Having established the Scissors Effect, we now investigate its mechanism. We
propose the \emph{Gradient Consistency Hypothesis}: the divergent effects of DI
stem from differences in gradient geometry between standard and robust models.

\subsection{The Gradient Consistency Hypothesis}
\label{sec:lgc}

\textbf{Intuition.} A model's input gradient can be either locally \emph{stable},
pointing in essentially the same direction under tiny input perturbations, or
\emph{noisy}, shifting under small perturbations. Standard models, trained on
clean data, tend to have noisy, high-frequency gradients
~\citep{ilyas2019adversarial, wang2020high}. Robust models, forced to learn features stable
under perturbation, have locally consistent gradients aligned with semantically
meaningful directions. DI averages the gradient over randomly transformed views;
this denoises a noisy gradient but perturbs an already-stable one.

We formalize stability through Local Gradient Consistency (LGC), the average
cosine similarity between the gradient at a clean input $x$ and at slightly
perturbed neighbors $x'_k=x+\xi_k$:
\begin{equation}
    \text{LGC}(x) = \frac{1}{K} \sum_{k=1}^{K} \cos\!\big(\nabla_x \mathcal{L}(f(x), y),\ \nabla_{x'_k} \mathcal{L}(f(x'_k), y)\big),
    \label{eq:lgc}
\end{equation}
where $\xi_k \sim \mathcal{U}(-\epsilon_{chk}, \epsilon_{chk})$ and
$\epsilon_{chk}{=}1/255$ (smaller than the attack step, probing the local
landscape). LGC is itself a sampled quantity, so a surrogate's value moves by
about $\pm0.02$ between runs that draw different images and probes; ResNet-50
appears at $0.63$--$0.65$ across the tables below for that reason, and we quote
each table's own measurement rather than a single canonical number.

\noindent\textbf{Empirical Validation.} \Cref{fig:grads} visualizes gradient
patterns. Robust models (\eg, Engstrom) have LGC $\approx 0.98$--$1.00$ with
spatially smooth, semantically aligned gradients; standard models vary by
architecture and dataset (ResNet18: $0.85$ on CIFAR-10; ResNet50: $0.64$ on
ImageNet) and exhibit high-frequency gradient noise that benefits from DI's
smoothing.

\begin{figure}[t!]
\centering
\begin{minipage}[t]{0.48\linewidth}
\centering
\includegraphics[width=\linewidth]{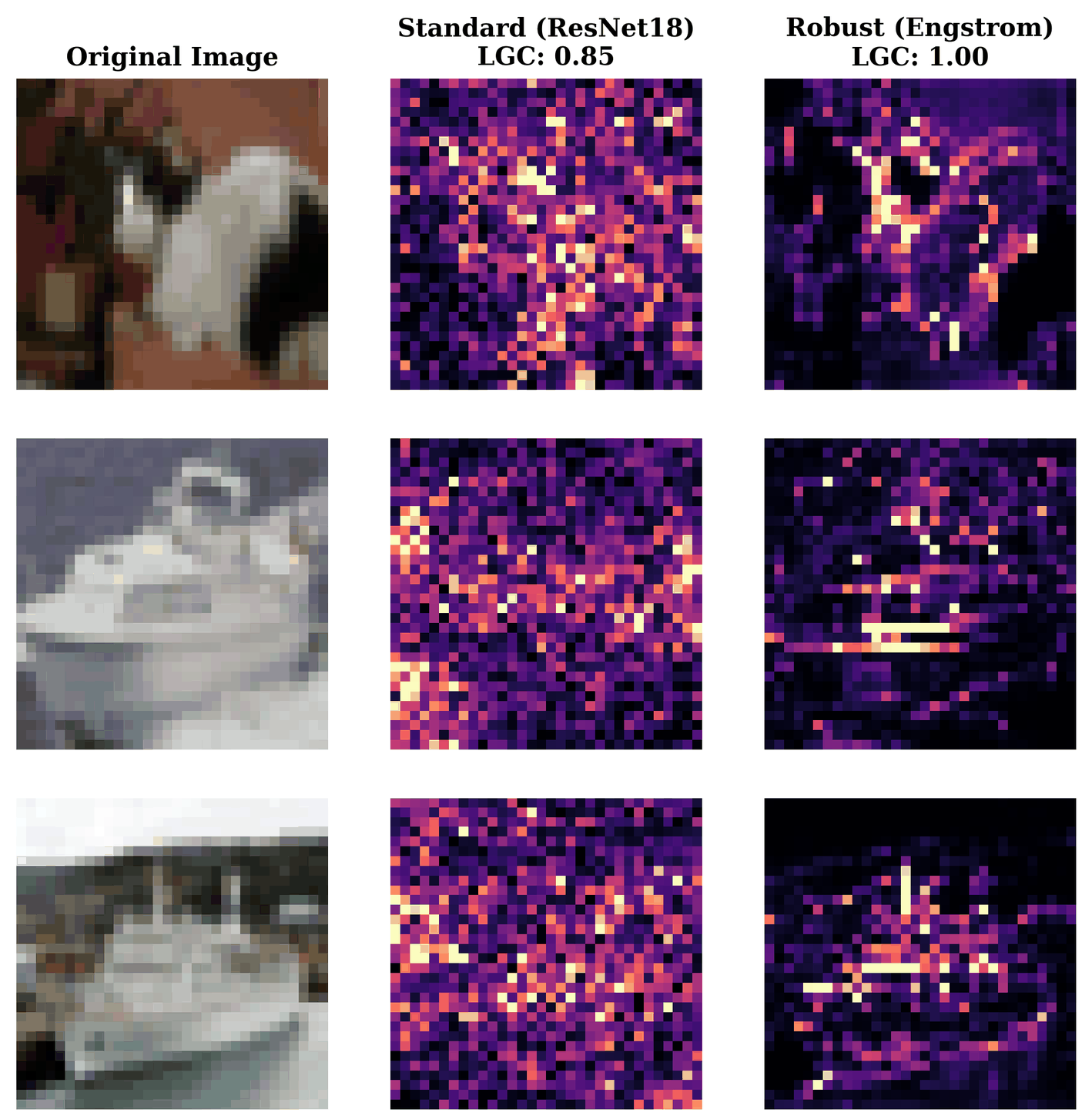}
\end{minipage}
\hfill
\begin{minipage}[t]{0.48\linewidth}
\centering
\includegraphics[width=\linewidth]{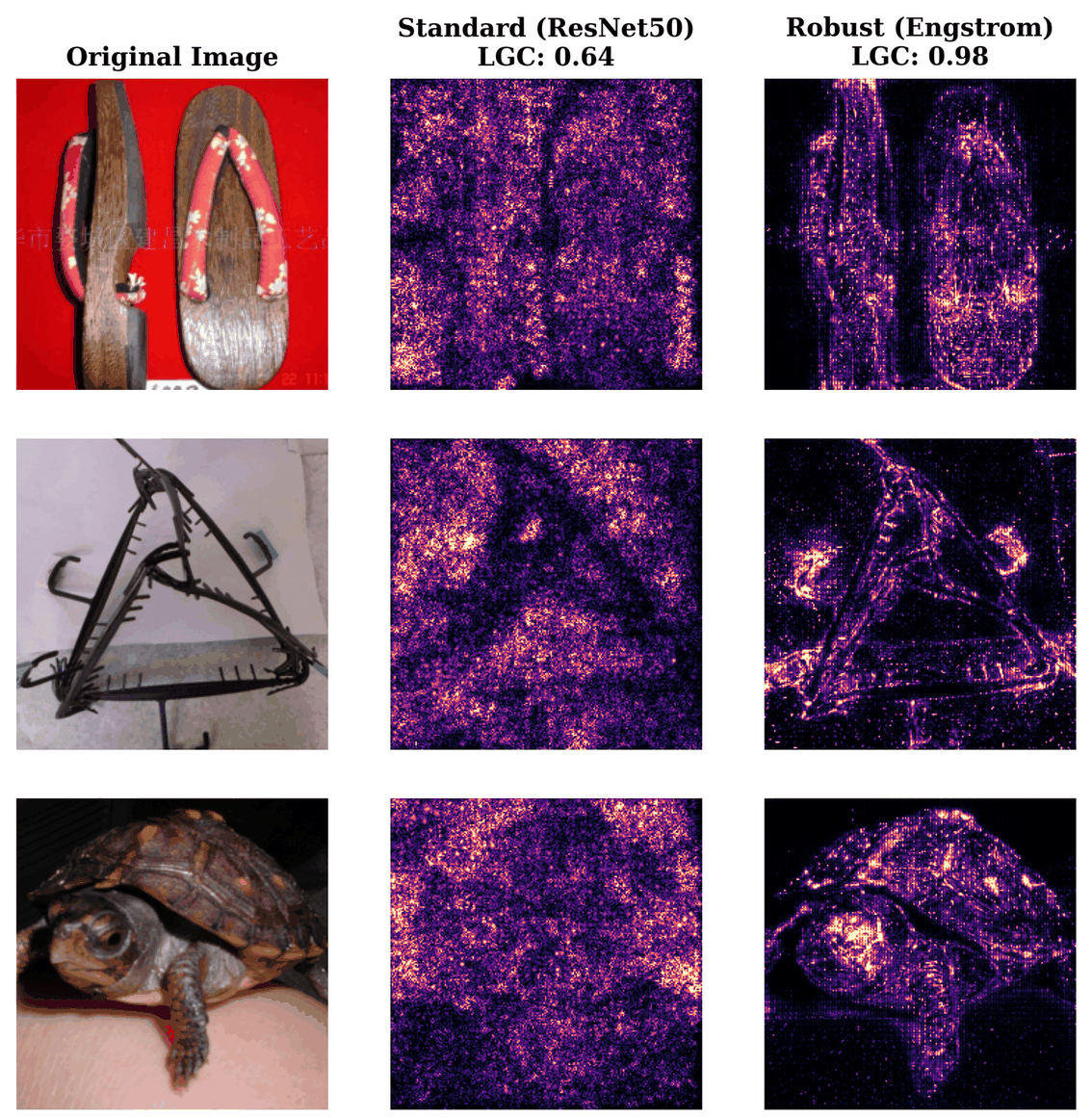}
\end{minipage}
\caption{Gradient consistency across datasets. (\emph{Left}) CIFAR-10: standard
(ResNet18) gradients show scattered high-frequency noise; robust (Engstrom)
gradients are spatially consistent. (\emph{Right}) ImageNet: the same pattern
holds at higher resolution.}
\label{fig:grads}
\end{figure}

\noindent\textbf{What Eq.~\ref{eq:lgc} measures, and what the theory uses.} These
are not the same quantity, and keeping them apart matters for how far the theory
can be pushed.
\Cref{eq:lgc} compares the clean gradient against perturbed ones, so the clean
point appears on both sides and is not itself a draw from any distribution. The
theorem's quantity is the gradient consistency ratio
$\mathrm{GCR}=\rho/(1+\rho)$, with $\rho$ the gradient SNR. GCR is directly
measurable without any model, but it is a ratio of moments and not an average of
cosines: for \emph{two independent} probes it is
$\mathbb{E}\langle g,g'\rangle$ normalized by
$\sqrt{\mathbb{E}\|g\|^2\,\mathbb{E}\|g'\|^2}$ (\cref{lem:lgc}), which equals
$\rho/(1+\rho)$ by construction. A mean cosine differs from it by a Jensen gap,
which is one reason we do not read \cref{eq:lgc} as an estimator of it. We write
$\mathrm{LGC}$ only for the Eq.~\ref{eq:lgc} statistic that CG-DI computes, and
$\mathrm{GCR}$ for the theorem's.

Measured on 15 surrogates at $\epsilon_{chk}{=}1/255$, the two-probe estimate and
an independent moment estimator of $\rho/(1+\rho)$ agree to within $0.005$ on 12
of 15. Eq.~\ref{eq:lgc} sits above GCR on 14 of the 15, by a margin that shrinks
as SNR grows: $+0.152$ for ResNet-50 ($\rho{=}1.0$), $+0.008$ for Engstrom
($\rho{=}37$), $+0.0001$ for Mo2022 ($\rho{=}6500$). The exception is ViT-B/16,
where Eq.~\ref{eq:lgc} falls $0.019$ \emph{below} GCR, so the ordering is a
measured tendency and not a bound: the two differ by a Jensen gap, which carries
no sign in general. We report the measured relation and assert no approximation
formula. The practical consequence is that
CG-DI's threshold is calibrated on Eq.~\ref{eq:lgc} and cannot be read as a
threshold on GCR, which is why the three thresholds in this paper carry distinct
symbols (\cref{tab:thresholds}).

\begin{table}[t!]
\caption{The three thresholds appearing in this paper. They apply to different
quantities and only one of them is a claim about real surrogates; the numerical
closeness of the first two is a coincidence, not a prediction.}
\label{tab:thresholds}
\centering
\small
\begin{tabular}{@{}llp{0.24\linewidth}p{0.36\linewidth}@{}}
\toprule
Symbol & Value & Applies to & Status \\
\midrule
$\tau_{\text{op}}$ & $0.92$ & Eq.~\ref{eq:lgc} statistic & operational; frozen, held out \\
$\tau_{\text{sim}}$ & $0.905$ & GCR, fixed-operator toy & illustrative; a simulation constant \\
$\widehat\tau_{\text{E6}}$ & $0.011$--$0.586$ & GCR, plug-in estimate & 6 of 7 finite (Swin-B none); not attack-matched; correct on 2 of 7 \\
\bottomrule
\end{tabular}
\end{table}

The closeness of $\tau_{\text{op}}$ and $\tau_{\text{sim}}$ is a coincidence
between a value calibrated on one statistic and a constant chosen to illustrate
another.

We further validate this in the frequency domain. Using the 2D FFT of the
batch-averaged gradient, we compute a High-Frequency (HF) Ratio (inspired
by~\citealp{wang2020high}), the fraction of log-magnitude spectral content above
50\% of the maximum spatial frequency (formal definition in
\cref{sec:spectral}). \Cref{tab:frequency} reports both metrics for 12 ImageNet
models.

\begin{table}[t!]
\caption{Gradient geometry of 12 ImageNet models ($N{=}200$ images, unfiltered;
$K{=}5$ probes at $\sigma{=}1/255$). HF is the share of the radially averaged
\emph{log-magnitude} spectrum above the mid-frequency cutoff
(\cref{sec:spectral}): a ratio of log-magnitudes, not of energies. No target
enters, so no basis applies.}
\label{tab:frequency}
\centering
\small
\begin{tabular}{llcc}
\toprule
Model & Type & HF Ratio & LGC \\
\midrule
ResNet18 & Standard & 0.48 & 0.83 \\
VGG16 & Standard & 0.50 & 0.84 \\
ResNet50 & Standard & 0.55 & 0.64 \\
DenseNet121 & Standard & 0.42 & 0.74 \\
InceptionV3 & Standard & 0.39 & 0.81 \\
ViT-B/16 & Standard & 0.37 & 0.90 \\
Swin-B & Standard & 0.55 & 0.36 \\
ConvNeXt-B & Standard & 0.57 & 0.87 \\
\midrule
Engstrom & Robust & 0.31 & 0.98 \\
Salman2020 & Robust & 0.35 & 0.98 \\
Mo2022 & Robust & \textbf{0.17} & \textbf{1.00} \\
\midrule
CLIP~\citep{radford2021learning} & VLM & 0.47 & 0.72 \\
\bottomrule
\end{tabular}
\end{table}

Standard models show HF Ratio $\approx 0.37$--$0.57$ versus robust $\approx
0.17$--$0.35$ ($\sim$1.7$\times$ lower), confirming the spectral gap. Mo2022
(ViT-B + AT) achieves the lowest HF ratio ($0.17$) and near-perfect LGC ($1.00$).
Swin-B has low LGC ($0.36$) despite standard training, likely because shifted-window
attention creates local gradient discontinuities, yet CG-DI still correctly
assigns $p{=}0.8$ since $0.36 \ll \tau$.

\subsection{A Population-Level Signature: Gradient Alignment Under DI}
\label{sec:alignment}

LGC and the frequency analysis characterize the surrogate in isolation. To test
the mechanism directly, namely that the same DI operation moves a standard
surrogate's gradient \emph{toward} transferable target directions but a robust
surrogate's \emph{away}, we measure source--target gradient alignment with and
without DI ($N{=}500$ ImageNet images, 3 seeds, 20 Expectation-over-Transformation
(EOT)~\citep{athalye2018synthesizing} samples, on source-and-target clean-correct
subsets). The design is one standard and three robust sources against three
ImageNet targets, giving 3 standard and 9 robust source--target pairs. The
attack-relevant metric is the
change in sign-alignment $D_{\text{sign}} := \Delta\cos\!\big(\mathrm{sign}(g_{\text{src}}),\,
\mathrm{sign}(g_{\text{tgt}})\big)$ induced by DI, which directly tracks the
sign-update step of (MI/DI)-FGSM.

\begin{table}[t!]
\caption{Change in source--target sign-alignment induced by DI ($N{=}500$
ImageNet images, 3 seeds, 20 EOT draws, source and target clean-correct basis).
$D_{\text{sign}}$ is averaged over
images within a pair and then over pairs; ``resolved'' counts pairs whose
per-image mean exceeds twice its standard error over images. The magnitude column
is the median over pairs of the per-image median
$\|g_{\text{DI}}\|/\|g_{\text{no-DI}}\|$, aggregated by median because the mean
of a ratio is dominated by near-zero denominators.}
\label{tab:alignment}
\centering
\small
\begin{tabular}{lcccccc}
\toprule
& & \multicolumn{2}{c}{$r{=}0.9$} & \multicolumn{2}{c}{$r{=}0.6$} & median \\
Source type & pairs & $D_{\text{sign}}$ & resolved & $D_{\text{sign}}$ & resolved & ratio ($r{=}0.6$) \\
\midrule
Standard (ResNet50) & 3 & $+0.0053$ & 3/3 & $+0.0056$ & 3/3 & $0.24$ \\
Robust (Engstrom, Salman, Mo2022) & 9 & $-0.0004$ & 2/9 & $-0.0009$ & 4/9 & $1.29$ \\
\bottomrule
\end{tabular}
\end{table}

\Cref{tab:alignment} shows an asymmetry, and the two columns carry very different
evidential weight. On the standard side all three pairs move the same way and each
one is individually resolved, from $+0.0036$ to $+0.0085$ (8.5 to 15.5 standard
errors). On the robust side the sign is consistent but the magnitudes largely are
not: 8 of the 9 pairs are negative at both resize rates, yet only 2 of 9 clear
twice their standard error at $r{=}0.9$ and 4 of 9 at $r{=}0.6$. The one positive
robust pair, Engstrom$\to$Swin-B at $+0.00007$, is $0.3$ standard errors from
zero rather than a counter-example of any size. We therefore report the robust
column as a consistent direction across pairs and not as a per-pair effect, and
we do not read $-0.0004$ as an account of a ten-to-sixteen-point ASR swing.

The magnitude column is the stronger of the two. DI shrinks the standard gradient
to $0.24$ of its untransformed norm at $r{=}0.6$, consistent with removing
high-frequency content that averaging can cancel. Robust gradients are unchanged
or slightly larger ($1.29$): already low-frequency, they give the resize little
noise to remove, so it adds its own transform variance instead. The
candidate driver is the asymmetric variance reduction: DI's averaging buys a
large variance gain on noisy standard gradients and almost none on robust ones,
while displacing the gradient by a comparable amount in both. \Cref{tab:biasvar} measures
those two quantities directly rather than inferring them from a norm ratio, which
can also move through cancellation or resize-Jacobian scaling.

\textbf{Along the trajectory, not at the clean point.} The measurement above is
taken at the clean input, and at that point the signal does not reach individual
images: images DI flips have no distinguishable sign-alignment shift from those
it leaves alone (\cref{sec:alignment_mediation}). That null is a timescale
mismatch rather than an absence, since a clean-point statistic is being asked
about a ten-step outcome. Recording sign-alignment, target loss and target margin
at every step instead, the picture changes. At the population level the sign of
the DI-induced change in trajectory-mean alignment matches the sign of the change
in target loss and of the ASR change on all eight source--target pairs, with no
exceptions. Per image, three trajectory statistics separate the images DI flips
from those it does not (\cref{tab:trajectory}):

\begin{table}[t!]
\caption{Per-image association between trajectory statistics and whether DI flips
that image's outcome, over 8 source--target pairs ($N{=}500$, 3 seeds, 10 attack
steps, source and target clean-correct basis). The unit of analysis is the image;
``significant'' counts pairs at $p<0.05$, and $|r|$ is the point-biserial
correlation between the statistic and the flip indicator, computed within a pair
and summarized across pairs. The clean-point control of
\cref{sec:alignment_mediation} is included for comparison.}
\label{tab:trajectory}
\centering
\small
\begin{tabular}{lccc}
\toprule
Statistic & significant & median $|r|$ & range \\
\midrule
Clean-point sign alignment & 0/8 & $0.02$ & --- \\
\midrule
Trajectory-mean alignment & 8/8 & $0.22$ & $0.11$--$0.31$ \\
Target-loss integral & 7/8 & $0.20$ & $0.08$--$0.35$ \\
Target-margin integral & \textbf{8/8} & $\mathbf{0.34}$ & $0.20$--$0.47$ \\
\bottomrule
\end{tabular}
\end{table}

The strongest separation is on Engstrom~$\to$~ConvNeXt-B, where the probability
that DI hurts an image rises from $0.04$ in the lowest quintile of the
target-margin integral to $0.34$ in the highest. An $|r|$ of $0.34$ is $12\%$ of
the variance, so this is a partial association and we describe it as one: the
mechanism is distributional, and the quantitative claim rests on the
variance-reduction asymmetry rather than on the sub-percent mean alignment change.

\subsection{Component Decomposition: Resize vs.\ Translation}
\label{sec:decomposition}

To isolate the harmful part of DI, we decompose it into \emph{resize} (which
alters spatial frequency content via interpolation) and \emph{translation} (which
shifts the image while preserving frequency content).

\begin{table}[t!]
\caption{Two-factor decomposition of the DI operator ($N{=}1{,}000$, 5 seeds,
$p{=}1$; 6 sources $\times$ 3 targets, so 6 standard and 12 robust source--target
pairs; all-images basis, effects in pp). ME is a main effect and the interaction
is $(\text{both}-\text{resize})-(\text{translation}-\text{none})$, averaged over
pairs within a source type. Intervals are 95\% image-paired bootstraps with the
four cells paired within an image. The ``both'' cell is the composition of the
two factors, not the DI operator. Restricted to $r>1$; the $r{=}0.9$ operator is
bridged in the text.}
\label{tab:transform}
\centering
\small
\begin{tabular}{lcccc}
\toprule
\multirow{2}{*}{Source type} & \multicolumn{4}{c}{ASR (\%) in the cells} \\
\cmidrule(l){2-5}
& none & resize & transl. & both \\
\midrule
Standard (6 pairs) & 47.3 & 53.0 & 57.1 & 60.2 \\
Robust (12 pairs) & 73.5 & 64.0 & 71.4 & 61.5 \\
\midrule
\multirow{2}{*}{Effect} & \multicolumn{2}{c}{Standard} & \multicolumn{2}{c}{Robust} \\
\cmidrule(lr){2-3}\cmidrule(l){4-5}
& estimate & 95\% CI & estimate & 95\% CI \\
\midrule
resize (ME) & $+4.38$ & $[+3.77,+4.99]$ & $\mathbf{-9.71}$ & $[-10.43,-9.01]$ \\
translation (ME) & $+8.47$ & $[+7.71,+9.24]$ & $-2.32$ & $[-2.84,-1.84]$ \\
interaction & $-2.65$ & $[-3.54,-1.77]$ & $-0.47$ & $[-1.07,+0.13]$ \\
\cmidrule(l){1-5}
both cells & $+12.85$ & $[+11.67,+14.06]$ & $-12.03$ & $[-12.96,-11.14]$ \\
DI operator & $+11.49$ & $[+10.35,+12.63]$ & $-14.19$ & $[-15.22,-13.22]$ \\
\bottomrule
\end{tabular}
\end{table}

\Cref{tab:transform} shows an asymmetry between the two factors on robust
surrogates. Resize carries a main effect of $-9.71$ pp $[-10.43,-9.01]$ against
translation's $-2.32$ $[-2.84,-1.84]$, a factor of four, and the interaction is
the one term whose interval contains zero: $-0.47$ $[-1.07,+0.13]$, with
$|\mathrm{INT}|<2$ on 11 of the 12 robust pairs.\footnote{This is not
TI-FGSM's kernel-based translation invariance, which smooths gradients
aggressively and causes larger harm ($-18.4$ pp); see
\cref{sec:translation_impact}.} We state the resize effect as a
\emph{magnitude} rather than as an attribution: it equals 81\% of the two-factor
composition ($-12.03$) and 68\% of the full operator's effect ($-14.19$). These
two denominators are different quantities and a single percentage cannot stand
for both, since the full DI operator is not the composition of the two factors we
manipulate.

The two identities that define the design,
$\mathrm{ME}_{\text{res}}+\mathrm{ME}_{\text{tr}}=\Delta_{\text{both}}$ and
$\mathrm{SE}_{\text{res}}+\mathrm{SE}_{\text{tr}}+\mathrm{INT}=\Delta_{\text{both}}$,
hold to machine precision at every level of the aggregation, which is what lets
the effects be bootstrapped image by image with the four cells paired.

On the standard side both factors help and the interaction is larger relative to
them, and unlike the robust side it is resolved away from zero: $-2.65$
$[-3.54,-1.77]$ against main effects of $+4.38$ and $+8.47$, with
$|\mathrm{INT}|<2$ on only 3 of 6 pairs. We therefore do not read the standard
column as separable.

The factorial is defined only for $r>1$: when $r<1$ the pad offset exists solely
as a consequence of the shrink, so translation cannot be varied on its own. We
therefore ran the $r{=}0.9$ operator of \cref{tab:imagenet} as a bridge cell on
the same sources, targets and images. It gives $-9.07$ pp $[-9.89,-8.27]$ on
robust surrogates against $-14.19$ $[-15.22,-13.22]$ for the $r{=}1.1$ operator
decomposed here. The two intervals are disjoint, so that gap is an operator
difference and not sampling.
The bilinear/bicubic interpolation in resize acts as a low-pass filter: for
standard models with noisy gradients this smoothing helps, but for robust models
whose gradients are already low-frequency and aligned it introduces
bias without variance reduction. An interpolation-mode ablation
(\cref{sec:interp_ablation}) confirms the \emph{direction} of the effect is invariant
across bilinear, bicubic, and antialiased resize, while its \emph{magnitude} scales
with the filter's low-pass strength (bicubic, a gentler rolloff, roughly halves the
harm): the dependence the low-pass-bias account predicts, and the opposite of what a
sign-agnostic interpolation artifact would produce. This
is also where the frequency analysis (\cref{tab:frequency,sec:spectral}) connects
to the theory of \cref{sec:bias_variance}: the spectra are what establish that
resize is a genuine low-pass contraction (so the theorem's bias factors $c,g_r<1$),
whereas a small centered translation acts as $R\approx I$ (so $c,g_r\approx1$ and
no bias). The frequency story thus explains why the bias term is large for
resize and negligible for translation.

\subsection{SSA: An Unresolved Interaction}
\label{sec:ssa}

SSA is the largest robust-side effect in \cref{tab:modern} ($-15.6$ pp) and the
only method whose standard-side effect is not clearly positive ($-0.6$ pp). A
natural account suggests itself: SSA is the only frequency-domain method in the
panel, and modelled as a low-pass spectral pre-filter it should raise a
surrogate's \emph{effective} gradient consistency, pushing it rightward along the
crossover and deeper into the DI-harmful regime. We tested that account
prospectively and it does not survive.

The test was pre-registered in two stages. Stage one measured the effective
gradient regime under each method directly, as the consistency of two independent
probes of the method's own gradient estimate; the prediction was written down and
hashed before stage two ran. The prediction was that SSA raises effective GCR
relative to no pre-method, and that the spatial methods do not.

Neither part holds. Across three sources and eight pre-methods, all 24
measurements \emph{lower} effective GCR, none raise it. On Engstrom the
no-pre-method value is $0.974$ and every method falls below it: Admix $0.583$,
BSR $0.278$, SIA $0.266$, and SSA between $0.418$ and $0.133$ depending on its
spectral-scaling strength. The pattern repeats on Salman $\epsilon{=}2$ and on
ResNet-50. So the direction is opposite to what the account required, and SSA is
not distinguished from the spatial methods by this measurement at all: it sits
inside their range, and where it sits is set by its own strength parameter.

That parameter is also why this account cannot be offered as support even where
it appears to fit. It is free, so the model can accommodate either sign of a
near-zero standard-side effect after the fact, which makes the agreement
uninformative. We therefore do not apply the corollary to SSA, and report SSA as
an unresolved interaction between two transform families. Our own sweep has a
limitation worth stating: its DI arm averages 20 draws of the composed transform
per step, so it is not the single-draw DI operator used elsewhere in the paper.
That is why the falsification rests on the stage-one gradient measurement rather
than on the sweep's ASR.

What does survive is the descriptive part. An SSA-based transform decomposition
analogous to \cref{tab:transform} (target Swin-B, $N{=}500$, 3 seeds; the SSA base
here is MI-FGSM with a 20-sample SSA ensemble, which is stronger than the
TransferAttack default and so reaches higher absolute ASR) shows the same
resize-driven asymmetry as the DI operator alone (\cref{tab:ssa}).

\begin{table}[t!]
\caption{SSA $\times$ transform decomposition (target Swin-B, $N{=}500$, 3 seeds,
all-images basis). Resize is the dominant harm on the robust
surrogate and translation the smaller one, matching \cref{tab:transform}. The
per-step gradient is averaged over 20 draws of the composed transform, SSA's own
ensemble size, so this is not the single-draw operator used elsewhere.
ASR (\%).}
\label{tab:ssa}
\centering
\small
\begin{tabular}{llcccc}
\toprule
Dataset & Surrogate & SSA-only & +Resize & +Translation & +Full DI \\
\midrule
\multirow{2}{*}{ImageNet} & Engstrom (Rob) & 54.1 & \textbf{44.5} ($-9.6$) & 51.7 ($-2.4$) & 47.5 ($-6.6$) \\
 & ResNet50 (Std) & 73.8 & 79.8 ($+6.0$) & 79.3 ($+5.5$) & 81.6 ($+7.8$) \\
\midrule
\multirow{2}{*}{CIFAR-10} & Engstrom (Rob) & 23.2 & \textbf{20.3} ($-2.9$) & 23.7 ($+0.5$) & 20.8 ($-2.5$) \\
 & Standard (Std) & 12.1 & 13.5 ($+1.5$) & 12.5 ($+0.5$) & 13.7 ($+1.6$) \\
\bottomrule
\end{tabular}
\end{table}

On the robust surrogate, SSA combined with resize alone causes the dominant harm
($-9.6$ pp on ImageNet; $-2.9$ pp on CIFAR-10, where it exceeds full DI's $-2.5$ pp),
while SSA with translation is approximately neutral. That is the same
resize-driven asymmetry as \cref{tab:transform}, now under a different base
attack, which is what makes it worth reporting. We do not offer a mechanism for
why the two compound: the natural spectral-overlap account is the one the
stage-one measurement of \cref{sec:ssa} already contradicts.

\subsection{A Bias--Variance Perspective}
\label{sec:bias_variance}

The effect can be read through a bias--variance lens. DI averages gradients over
transformed views, reducing the variance of the gradient estimate at the cost of
bias (the averaged direction may deviate from the original optimum). For standard
models (low LGC), the estimate has high variance from noisy, high-frequency loss
landscapes, so $\mathbb{E}_{T}[\nabla\mathcal{L}(T(x))]$ acts as beneficial
variance reduction. For robust models (high LGC), the landscape is smooth with
intrinsically low gradient variance ($\mathrm{Var}[\nabla\mathcal{L}]\approx0$),
so variance reduction is unnecessary; meanwhile resize shifts the spatial
frequency content and introduces systematic bias that degrades an already
well-aligned direction. LGC thus proxies the tradeoff: high LGC indicates low
variance where avoiding bias matters; low LGC indicates high variance where
averaging helps.

That is an argument, so we measured it. For each surrogate we draw $m{=}20$ EOT
samples of the DI gradient and estimate, per image, the variance across draws,
the variance of the averaged direction $V/m$, the signal $\|\mu\|^2$ and the
per-coordinate noise $\sigma^2$ from independent probes. Two quantities follow:
the variance of the DI-averaged direction relative to that of a single clean
gradient, $(V/m)/(n\sigma^2)$, and the bias, $\|\bar g-g\|/\|g\|$. Constants are
formed as ratios of expectations over images rather than as means of per-image
ratios; the raw second moments are released.

\begin{table}[t!]
\caption{Direct bias--variance measurement at $r{=}0.9$, $m{=}20$ EOT draws, on
the source clean-correct subset of a $500$-image ImageNet pool; $N$ is the number
retained per surrogate. The $m$ draws take the place of seeds here, and no target
enters, so no transfer basis applies. Constants are formed as ratios of
expectations over images rather than means of per-image ratios, which a near-zero
denominator would dominate; intervals are a $10{,}000$-replicate bootstrap that
resamples images and recomputes the ratio.
$\rho=\|\mu\|^2/(n\sigma^2)$ is the measured gradient SNR.}
\label{tab:biasvar}
\centering
\small
\begin{tabular}{lrrcc}
\toprule
Surrogate & $N$ & SNR $\rho$ & variance ratio [95\% CI] & bias [95\% CI] \\
\midrule
\multicolumn{5}{l}{\textit{standard}} \\
Swin-B & $409$ & $0.2$ & $0.036$ $[0.022,0.061]$ & $0.98$ $[0.96,1.01]$ \\
ResNet-50 & $404$ & $1.0$ & $0.057$ $[0.049,0.066]$ & $0.93$ $[0.93,0.94]$ \\
DenseNet-121 & $370$ & $1.7$ & $0.131$ $[0.101,0.179]$ & $0.94$ $[0.92,0.98]$ \\
InceptionV3 & $380$ & $3.1$ & $0.434$ $[0.215,0.731]$ & $1.23$ $[0.97,1.61]$ \\
ConvNeXt-B & $422$ & $4.7$ & $0.229$ $[0.175,0.305]$ & $0.90$ $[0.87,0.93]$ \\
ViT-B/16 & $404$ & $10.1$ & $0.516$ $[0.352,0.771]$ & $0.90$ $[0.88,0.92]$ \\
\midrule
\multicolumn{5}{l}{\textit{robust}} \\
Salman $\epsilon{=}8$ & $290$ & $0.4$ & $0.027$ $[0.017,0.057]$ & $0.95$ $[0.94,0.97]$ \\
Engstrom & $325$ & $37.3$ & $1.384$ $[1.045,1.870]$ & $0.92$ $[0.88,0.99]$ \\
Salman $\epsilon{=}4$ & $334$ & $39.6$ & $1.433$ $[1.089,1.935]$ & $0.96$ $[0.94,1.00]$ \\
Salman $\epsilon{=}2$ & $353$ & $51.6$ & $1.788$ $[1.424,2.319]$ & $0.96$ $[0.91,1.03]$ \\
Salman $\epsilon{=}1$ & $365$ & $62.6$ & $1.841$ $[1.520,2.278]$ & $0.89$ $[0.86,0.92]$ \\
Salman $\epsilon{=}0.5$ & $367$ & $68.6$ & $3.165$ $[2.433,4.217]$ & $1.03$ $[0.94,1.15]$ \\
ARES ConvNeXt-AT & $379$ & $815.3$ & $11.39$ $[8.74,15.55]$ & $0.71$ $[0.67,0.77]$ \\
Singh ConvStem-AT & $380$ & $1517.4$ & $17.41$ $[12.26,25.77]$ & $0.61$ $[0.56,0.66]$ \\
Mo2022 & $350$ & $6501.9$ & $134.5$ $[118.1,158.3]$ & $0.59$ $[0.57,0.62]$ \\
\bottomrule
\end{tabular}
\end{table}

\Cref{tab:biasvar} supports the reading more directly than a norm ratio can. DI
reduces the variance of the attack direction on all six standard surrogates and
increases it on eight of the nine robust ones, while the bias it pays is
comparable across the two groups (mean $0.98$ against $0.85$). Each of those
fourteen variance ratios is individually resolved: the six standard intervals lie
entirely below $1$ and the eight robust ones entirely above it. The single
unresolved case is Salman $\epsilon{=}8/255$, whose interval $[0.017,0.057]$ sits
below $1$ with the standard surrogates rather than with its own group, which is
the corner case discussed at the end of this subsection. At the group level
the asymmetry therefore sits on the variance side and not the displacement side.
These constants are measured with $m{=}20$ draws while the attack averages one
per step, so this is a statement about the two groups' gradient geometry and not
a causal decomposition of the attack outcome. The
variance ratio rises with the measured SNR across three orders of magnitude,
with local inversions among the low-SNR surrogates, which is the mechanism stated
quantitatively: a surrogate with little gradient noise has nothing for averaging
to remove, so the transform's own randomness dominates. Mo2022 is the limiting
case, with an SNR of $6500$ and essentially no noise to average.

The single robust surrogate that behaves like a standard one, Salman
$\epsilon{=}8/255$, does so for a reason the table makes visible: its gradient
SNR is $0.4$, in the range of the standard surrogates rather than of the other
robust ones, which run from $37$ upward. Adversarial training at that budget on a
ResNet-50 does not produce the gradient regime the rest of the robust column
shares. It is also the one checkpoint where CG-DI's fixed threshold fails
(\cref{sec:corner_case}), and this is why.

That checkpoint also marks the boundary of the bias--variance reading itself, and
we would rather state that than leave it for a reader to assemble. Its measured
constants, a variance ratio of $0.027$ and a bias of $0.95$, are those of a
surrogate DI should help; DI instead costs it $-6.6$ pp
(\cref{tab:eps_sweep}). Whatever decides the outcome there is not captured by the
two quantities this section measures. \Cref{prop:general_target} offers one
candidate, since a surrogate whose resize moves its signal away from the target's
vulnerable direction ($\Gamma<1$) is harmed regardless of its variance, but we
have not measured $\Gamma$ for this checkpoint and do not claim it as the
explanation.

This intuition can be made precise. Modeling the surrogate gradient as signal
plus isotropic noise and resize-DI as a linear operator $R$ with $m$-fold
averaging (assumptions in \cref{sec:theory}), the alignment between the attack
direction and the transferable target direction admits a closed form, and DI
helps exactly below a threshold on the gradient regime.

\begin{proposition}[Scissors crossover; proof in \cref{sec:theory}]
\label{prop:scissors}
Under \cref{ass:model,ass:di,ass:fom}, let $\rho=\|\mu\|^2/(n\sigma^2)$ be the
gradient signal-to-noise ratio, so that $\mathrm{GCR}=\rho/(1+\rho)$. Write
$c=\cos(R\mu,\mu)\in(0,1]$, $g_r=\|R\mu\|^2/\|\mu\|^2$, and
$\kappa=\mathrm{tr}(RR^{\top})/(nm)$, and assume $\kappa/g_r<c^2<1$. Then resize-based DI
strictly improves the signal-to-RMS alignment if and only if
$\mathrm{GCR}<\tau^\star$, and strictly degrades it if and only if
$\mathrm{GCR}>\tau^\star$, where
\begin{equation}
\tau^\star=\frac{\rho^\star}{1+\rho^\star},\qquad
\rho^\star=\frac{c^2-\kappa/g_r}{1-c^2}.
\label{eq:threshold}
\end{equation}
In particular, as $\mathrm{GCR}\to1$ (a fully robust surrogate), DI necessarily
hurts.
\end{proposition}

\noindent The proposition is a
\emph{surrogate-side} bias--variance crossover: a single threshold in the gradient
regime separates where averaging (DI) helps from where its resize bias dominates.
The bias enters only through $c,g_r<1$, \ie\ through $R\neq I$, the resize
component; a small, centered translation corresponds to $R\approx I$, leaving the
bias factor near unity, which matches translation being the smaller of the two
main effects in \cref{tab:transform}. The bridge from this surrogate-side statement to
\emph{transfer} is \cref{ass:fom}, which takes the target's vulnerable direction to
be aligned with the surrogate signal $\mu$. We do not prove that assumption; it is
the standard robust-features-transfer premise~\citep{salman2020adversarially} and
is what our gradient-alignment measurement (\cref{tab:alignment}) tests directly.
In fact the theorem does not even require exact alignment: \cref{prop:general_target}
generalizes it to an \emph{arbitrary} target direction $u$ and shows DI still
necessarily hurts as $\mathrm{GCR}\to1$ precisely when resize reduces alignment with
$u$ ($\gamma_R<\gamma_\mu$), a condition needing no architectural match and matching
the sign of the $D_{\text{sign}}$ that \cref{sec:alignment} measures. The proposition
therefore makes the bias--variance intuition explicit and testable, though its two
parts carry different weight. The fully robust limit ($\mathrm{GCR}\to1$, i.e.\
$\sigma\to0$) is close to definitional: with no gradient variance left to reduce,
any direction-altering operator can only hurt. The non-trivial content is therefore the
existence and uniqueness of the interior crossover, whose location the theorem
does not fix. What the proposition does not do is settle, from first principles,
whether a surrogate's gradient aligns with a target's transferable directions, which
is an empirical matter (\cref{sec:alignment}). All algebra is symbolically verified,
and a numerical simulation (\cref{fig:theory_mc}) exhibits the crossover near
$\mathrm{GCR}\approx0.9$ for the illustrative constants used there
(\cref{sec:theory}). The fixed-operator constants $c,g_r,\kappa$ can be estimated
on real gradients: of the 15 surrogates in \cref{tab:biasvar}, 13 admit a resize
geometry with $g_r\le1$, and on 12 of those the crossover condition
$\kappa/g_r<c^2$ also holds. The exception is Swin-B, whose clean gradients are
the least aligned in the set. The model can therefore be instantiated on real
surrogates, though not on every one.

The estimator we froze is a different one, and the two should not be
read as a single predictor. DI draws a fresh transform at each step, so the
matching statement is \cref{prop:random_transform} rather than
\cref{prop:scissors}, and it is that proposition's constants $c,g_r,a,b$ from
which the pre-registered plug-in threshold $\widehat\tau_{\text{E6}}$
(\cref{tab:thresholds}) was computed. That
predictor reproduces the DI direction on the eight surrogates it was formed from,
but that is a fit rather than a test: on seven surrogates it had not seen it is
correct on 2. Its summaries were also formed at $m{=}20$ with the transform
always applied, which is not the single-draw, $p{=}0.8$ operator the attack uses,
so what failed is the bridge from the model to attack outcomes rather than either
proposition (\cref{sec:measured_constants}). Either way we offer no per-surrogate
predictor, and CG-DI's $\tau$ stays empirical.

\begin{figure}[t]
\centering
\includegraphics[width=0.56\linewidth]{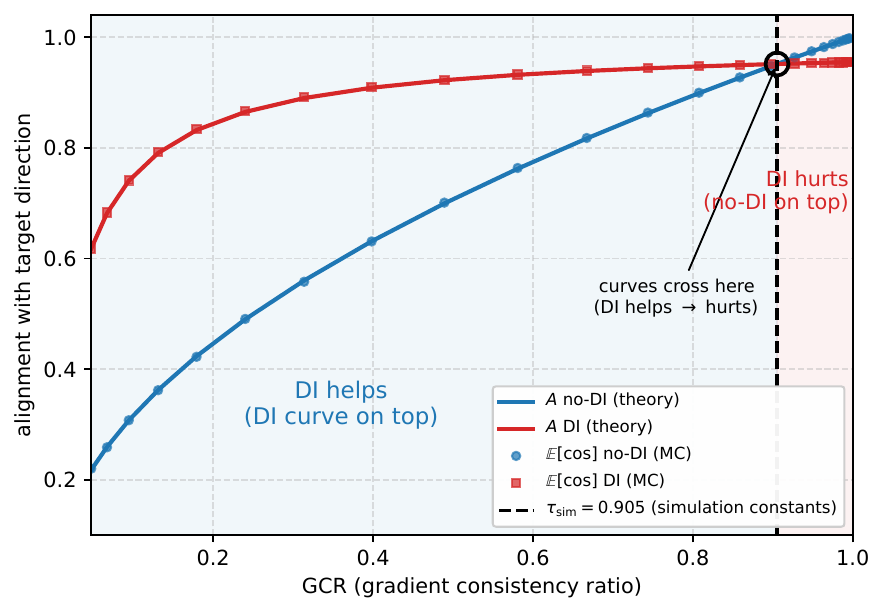}
\caption{The theorem, visualized. As the gradient regime sweeps from standard
(low GCR) to robust (high GCR), the DI effect on alignment changes sign at a single
crossover, and DI necessarily hurts as $\mathrm{GCR}\to1$. The crossover location
($\mathrm{GCR}\approx0.9$ here) is set by the constants $c,g_r,\kappa$, which are
illustrative simulation values rather than measurements. Lines are the
$A$-metric, markers the measured $\mathbb{E}[\cos]$. Setup in \cref{sec:theory}.}
\label{fig:theory_mc}
\end{figure}

\subsection{From Explanation to Indication: LGC as a Regime Identifier}
\label{sec:prediction}

Beyond the binary direction, does LGC quantify DI sensitivity? We find a
clear two-tier picture.
LGC is a reliable \textbf{coarse} regime indicator: on ImageNet the standard
surrogates span $0.36$--$0.90$ and the robust ones sit at $0.98$ or above, a
separation with a wide margin at $\tau_{\text{op}}{=}0.92$, and that separation is
all CG-DI uses. The margin is narrower elsewhere. On CIFAR-10 the standard
surrogates run to $0.937$, close enough to the threshold to misroute
(\cref{sec:heldout}), and among the robust ImageNet models Salman $\epsilon{=}8/255$ sits
at $0.80$, below both $\tau_{\text{op}}$ and $\tau_{\text{sim}}$. Whether LGC
also gives \emph{fine-grained} resolution within a
regime is a weaker claim we do not make: the correlation analysis below shows LGC
saturates across the robust region, so we read the quantitative correlations as
\emph{suggestive} of this regime picture rather than a per-surrogate predictor of
harm.

\textbf{How the two sides depend on the resize rate.} Sweeping $r$ over the full
thirteen-surrogate pool (\cref{tab:di_ablation}) separates two behaviors. The
robust side barely moves: its mean effect stays
between $-7.0$ and $-8.1$ across the whole range, and there is \emph{no ``safe''
near-identity resize}, since even $r{=}0.97$, a single-pixel shrink on
$32{\times}32$ images, already costs $-4.7$ to $-10.8$ depending on the
surrogate. Any interpolation perturbs an already-aligned gradient, and making the
interpolation gentler does not help.

The standard side is where $r$ matters, and it is harmed at both extremes: mean
$-3.85$ at $r{=}0.97$, $-0.68$ at $r{=}0.9$, and $-6.92$ at $r{=}0.6$. The
consequence is that the separation between the two groups is not monotone in $r$
either. It is widest near $r{=}0.95$ and collapses to $1.2$ pp at $r{=}0.6$, where
resize is aggressive enough to harm everything (\cref{sec:blade_gap}). The harm
therefore does not simply amplify with aggressiveness: the robust side's minimum
sits at the default rate, and both smaller and larger distortions cost more.

\begin{table}[t!]
\caption{Resize rate vs.\ harm on CIFAR-10 ($p{=}0.8$; $\Delta$ASR in pp vs.\ $p{=}0$,
$N{=}1{,}000$, 5 seeds, all-images basis), transferred to a naturally trained
WRN-28-10 held out of the surrogate pool.}
\label{tab:di_ablation}
\centering
\small
\begin{tabular}{lcccccc}
\toprule
Surrogate & $r$=0.97 & $r$=0.95 & $r$=0.9 & $r$=0.8 & $r$=0.7 & $r$=0.6 \\
\midrule
\multicolumn{7}{l}{\emph{Robust}} \\
Rice & $-10.80$ & $-9.40$ & $-9.32$ & $-9.86$ & $-10.38$ & $-11.00$ \\
Sehwag (proxy) & $-10.68$ & $-10.06$ & $-10.00$ & $-9.82$ & $-10.00$ & $-10.46$ \\
Engstrom & $-9.54$ & $-9.32$ & $-9.06$ & $-9.30$ & $-9.90$ & $-10.54$ \\
Sehwag-R18 (proxy) & $-9.38$ & $-8.24$ & $-8.62$ & $-8.92$ & $-9.10$ & $-9.08$ \\
Zhang (TRADES) & $-7.24$ & $-6.42$ & $-6.72$ & $-7.22$ & $-7.80$ & $-8.42$ \\
Carmon & $-6.68$ & $-6.56$ & $-5.80$ & $-5.90$ & $-6.32$ & $-6.84$ \\
Peng & $-5.04$ & $-4.86$ & $-4.54$ & $-4.38$ & $-5.06$ & $-5.58$ \\
Gowal & $-5.02$ & $-5.20$ & $-4.24$ & $-4.86$ & $-5.80$ & $-5.82$ \\
Wang & $-4.66$ & $-4.68$ & $-4.84$ & $-4.58$ & $-5.34$ & $-5.30$ \\
\emph{mean} & $-7.67$ & $-7.19$ & $\mathbf{-7.02}$ & $-7.20$ & $-7.74$ & $-8.12$ \\
\midrule
\multicolumn{7}{l}{\emph{Standard}} \\
DenseNet121 & $-2.08$ & $+1.14$ & $+1.64$ & $-0.76$ & $-3.22$ & $-4.66$ \\
VGG16 & $-4.12$ & $-0.28$ & $+0.06$ & $-1.26$ & $-3.82$ & $-7.10$ \\
ResNet50 & $-5.10$ & $-1.12$ & $-1.24$ & $-2.24$ & $-4.18$ & $-5.70$ \\
ResNet18 & $-4.10$ & $-1.54$ & $-3.16$ & $-6.48$ & $-9.08$ & $-10.22$ \\
\emph{mean} & $-3.85$ & $\mathbf{-0.45}$ & $-0.68$ & $-2.68$ & $-5.08$ & $-6.92$ \\
\midrule
gap & 3.82 & \textbf{6.74} & 6.34 & 4.52 & 2.67 & 1.20 \\
\bottomrule
\end{tabular}
\end{table}
\begin{figure}[t!]
\centering
\includegraphics[width=1.0\linewidth]{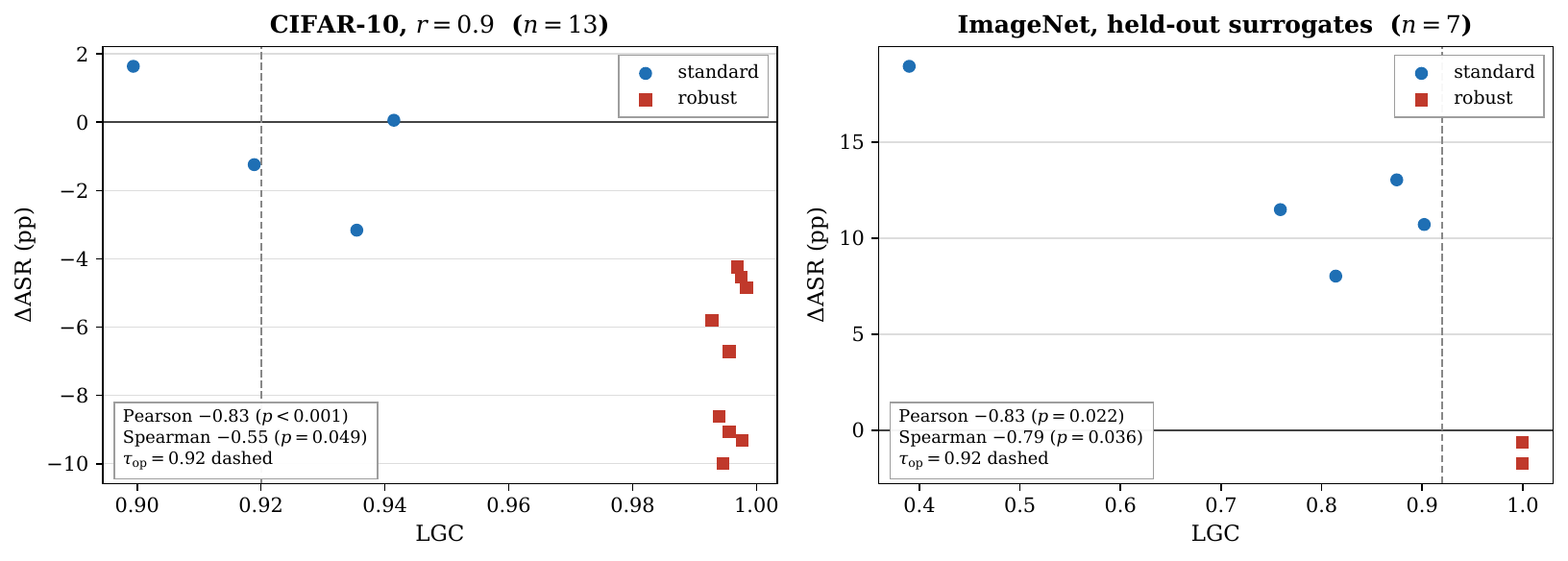}
\caption{LGC against the DI effect. Both panels use DI at $p{=}0.8$, the
all-images basis, and a target set disjoint from the surrogate pool, so they are
directly comparable. (\emph{Left}) CIFAR-10, 13 surrogates at the default
$r{=}0.9$, transferred to a naturally trained WRN-28-10 ($N{=}1{,}000$, 5 seeds).
(\emph{Right}) ImageNet, the 7 held-out surrogates of \cref{tab:heldout},
each transferred to the cross-family targets other than itself ($N{=}1{,}000$, 5
seeds). The controlled $\epsilon$-spectrum is not used here; the text says why.}
\label{fig:lgc_harm}
\end{figure}

\textbf{Evidence: primary and supporting.} The robust core of the empirical case
is \emph{directional}: DI harms the robust surrogate on all 10 attacks and helps
the standard one on 9 of 10 (\cref{tab:imagenet,tab:modern}). Because those ten
attacks are not independent, the strength of that pattern is what
\cref{tab:panel_stats} reports, through an image-block permutation test and a
pooled estimate over attacks, rather than a count.
That is the level at which CG-DI acts. The correlations
(\cref{fig:lgc_harm,tab:correlation}) are consistent with it but we read them only
as \emph{supporting}. On CIFAR-10, over 13 surrogates transferring to a target
outside the pool, LGC against the DI effect gives Pearson $-0.83$ ($p<0.001$) at
the default $r{=}0.9$; on the seven held-out ImageNet surrogates it gives
$-0.83$ ($p{=}0.022$) with Spearman $-0.79$ ($p{=}0.036$). Three caveats temper
these.

\begin{table}[t!]
\caption{LGC against DI sensitivity. The unit of the correlation is the
surrogate and $n$ is the number of surrogates; ASR is on the all-images basis,
with the target set disjoint from the surrogate pool in every row. ``LOO'' is the range of Pearson under
leave-one-out, which separates a relation from a single influential point. Every
row is a two-group separation rather than a graded relation, so Spearman is the
more conservative summary.}
\label{tab:correlation}
\centering
\small
\begin{tabular}{llcccc}
\toprule
Setting & $n$ & std / rob & Pearson & Spearman & LOO Pearson \\
\midrule
CIFAR-10, $r{=}0.9$ & 13 & 4 / 9 & $-0.83$ ($p{<}0.001$) & $-0.55$ ($p{=}0.049$) & $-0.86$ to $-0.75$ \\
CIFAR-10, $r{=}0.6$ & 13 & 4 / 9 & $-0.33$ ($p{=}0.27$) & $-0.13$ ($p{=}0.67$) & $-0.49$ to $-0.12$ \\
ImageNet, held out & 7 & 5 / 2 & $-0.83$ ($p{=}0.022$) & $-0.79$ ($p{=}0.036$) & $-0.90$ to $-0.80$ \\
ImageNet $\epsilon$-sweep & 6 & 1 / 5 & $-0.87$ ($p{=}0.025$) & $-0.26$ ($p{=}0.62$) & $-0.96$ to $\mathbf{-0.33}$ \\
\bottomrule
\end{tabular}
\end{table}

First, how much either coefficient rests on a single point differs sharply
between the two panels of \cref{fig:lgc_harm}, and this is the reason we use the held-out surrogates
rather than the controlled $\epsilon$-spectrum for the ImageNet panel. That sweep
contains only one standard model, and its correlation is carried by that one
point: $-0.87$ over all six, but $-0.33$ ($p{=}0.59$) with the standard point
removed, while removing any other point leaves it between $-0.85$ and $-0.96$.
Six points with a single high-leverage observation cannot support a claim about a
relation. The held-out panel has five standard surrogates spanning LGC
$0.39$--$0.90$, and leave-one-out keeps Pearson within $-0.80$ to $-0.90$.

That panel is not free of leverage either, and the difference is worth stating
precisely. Swin-B sits at LGC $0.39$, isolated from the next surrogate by $0.37$
on the axis, which gives it a leverage of $0.85$ and a Cook's distance of $4.2$.
Removing it barely moves the coefficient, from $-0.83$ to $-0.80$, but it does
move the $p$-value from $0.022$ to $0.054$. With seven points that is what one
should expect: the estimate is stable, the significance is not, and we therefore
rest the argument on the sign and the ordering rather than on a calibrated
$p$-value. Swin-B's low LGC is not a measurement artifact: it reappears at
$0.36$ in the independent measurement of
\cref{tab:frequency}; we read it as shifted-window attention producing locally
discontinuous gradients, which is a conjecture we do not test.

Second, both coefficients still reflect the \emph{coarse} standard-versus-robust
separation more than fine structure. The predictor is bimodal in both panels,
with nothing between the clusters, so a linear coefficient largely measures the
gap between them. Spearman is the more conservative summary, and on CIFAR-10 it
only just clears the line in \cref{tab:correlation} ($-0.55$, $p{=}0.049$).

Third, the CIFAR-10 coefficient depends on the transform: it is significant for
$r\in[0.8,0.97]$ and gone at $r{=}0.6$ ($-0.33$, $p{=}0.27$), where every
surrogate in the pool is harmed and there is no separation left to detect. We
report the default operator and the aggressive one side by side rather than
choosing between them. Within the robust region LGC also saturates
(\cref{tab:lgc_spectrum}) while the harm still varies several-fold, so LGC does
not resolve harm inside that regime.

In practice this supports a simple a-priori check from a single
gradient-consistency measurement: a surrogate with LGC $>0.95$ should be treated
as robust, and disabling diversity is the safer default for it. We attach no
typical loss to that check: the harm runs from $-0.8$ pp on a held-out ImageNet
source to $-8.1$ averaged over nine CIFAR-10 robust surrogates at $r{=}0.6$, so
any single headline number would misstate most settings.

\section{Application: Consistency-Guided Diversity Input}
\label{sec:application}

As a direct, falsifiable consequence of our findings, we propose \textbf{CG-DI}
(Consistency-Guided Diversity Input), a training-free rule that uses LGC to choose
the diversity setting per surrogate. It is a falsification test of LGC's
predictive claim: if an LGC-only switch, with its threshold, probe scale and
sample count fixed in advance, recovers the right regime
decision, LGC's utility extends beyond post-hoc explanation; if it fails, the
claim is refuted. Its simplicity is deliberate.

An attacker usually knows how their own surrogate was trained, so the probe has
to earn its place against simply reading that provenance. Over 27
model--dataset cases the two are each correct on 21 (\cref{sec:ablations}).

\subsection{Algorithm}
\label{sec:algorithm}

CG-DI detects the gradient regime with a lightweight LGC probe and makes a binary
decision: enable diversity ($p{=}0.8$) for standard-like models, disable it
($p{=}0$) for robust-like ones. \Cref{alg:cgdi} is a two-step process: (1) probe
LGC with $K$ local gradient queries; (2) compare against a threshold $\tau$.

\begin{algorithm}[!ht]
\caption{CG-DI: Consistency-Guided Diversity Input}
\label{alg:cgdi}
\small
\begin{algorithmic}[1]
\REQUIRE Surrogate $f$, input $x$, label $y$, threshold $\tau{=}0.92$
\STATE $g \leftarrow \nabla_x \mathcal{L}(f(x), y)$; $S \leftarrow 0$
\FOR{$k = 1$ to $K{=}5$}
\STATE $x'_k \leftarrow x + \mathcal{U}(-1/255, 1/255)$
\STATE $S \leftarrow S + \cos(g, \nabla_{x'_k} \mathcal{L}(f(x'_k), y))$
\ENDFOR
\STATE $p^* \leftarrow 0$ if $S/K > \tau$ else $0.8$
\RETURN Attack with diversity $p^*$
\end{algorithmic}
\end{algorithm}

\textbf{Two granularities, and which one we recommend.} \Cref{alg:cgdi} is stated
per image: it probes $x$ and chooses $p$ for that image. A \emph{model-level}
variant probes a calibration batch once, compares the mean LGC against
$\tau_{\text{op}}$, and uses the resulting $p$ for the whole attack. The two
coincide whenever a surrogate's per-image LGC sits entirely on one side of the
threshold, which is the usual case, and diverge for surrogates near it. We
recommend the model-level variant, and every table below states which one it
reports. The reason is \cref{sec:heldout}: on the two held-out surrogates closest
to $\tau_{\text{op}}$ the per-image variant splits its decisions and loses $7.4$
and $6.4$ pp against blind DI, while the model-level variant does not. The
model-level variant is also the cheaper one, since the probe runs once per
surrogate rather than once per image.

\textbf{What the binary choice costs: a continuous-$p$ sweep.} A binary
$p\in\{0,0.8\}$ is a coarse choice, and the natural check is to sweep $p$ and ask
what it gives up against the best grid value. We
sweep $p\in\{0,0.1,\dots,1.0\}$ for three sources against the four held-out
ImageNet targets (\cref{tab:continuous_p}).

The first thing the sweep shows is that there is no interior optimum to find. All
three curves are strictly monotone in $p$, with a \emph{sign-flipped slope}
between the low-LGC and the high-LGC sources, so the argmax is an endpoint in
every case and the quantity LGC predicts is the slope's sign. A smaller pilot
($N{=}500$, three seeds, one target) cannot separate the argmax from its
runner-up at all; at $N{=}1{,}000$ with five seeds and four targets it can, and
monotonicity is why the smaller sweep could not.

The useful question is therefore not where the optimum sits but what the binary
projection costs against it. On both robust sources CG-DI selects the argmax
exactly, at zero cost. On ResNet-50 it selects $p{=}0.8$ where the best grid
value is $p{=}1$, costing $0.95$ pp $[0.50, 1.42]$. Proximity to ``the continuous
optimum'' would presuppose an interior optimum that these curves do not have, so
the measured cost and its interval are the meaningful comparison.

\begin{table}[t!]
\caption{Continuous-$p$ sweep on ImageNet ($N{=}1{,}000$, 5 seeds, $r{=}0.9$),
pooled over four targets held out of the source pool, on the target
clean-correct basis; the CG-DI column is the model-level decision. Every curve is
monotone in $p$, so $p^*$ is an endpoint and LGC predicts the slope's sign rather
than an interior optimum. ``Cost'' is
ASR$(p^*)-$ASR$(p_{\text{CG-DI}})$ with a 95\% image-paired bootstrap interval.
ASR (\%).}
\label{tab:continuous_p}
\centering
\small
\begin{tabular}{lcccccccc}
\toprule
Source & LGC & $p{=}0$ & $p{=}0.5$ & $p{=}0.8$ & $p{=}1.0$ & $p^*$ & CG-DI & cost \\
\midrule
Standard RN-50 & 0.654 & 31.2 & 45.3 & 48.2 & \textbf{49.2} & $1.0$ & $0.8$ & $0.95$ $[0.50, 1.42]$ \\
Engstrom $\epsilon{=}4$ & 0.980 & \textbf{70.5} & 65.9 & 61.9 & 58.2 & $0.0$ & $0.0$ & $\mathbf{0.00}$ \\
Salman $\epsilon{=}2$ & 0.987 & \textbf{77.8} & 74.0 & 71.7 & 68.6 & $0.0$ & $0.0$ & $\mathbf{0.00}$ \\
\bottomrule
\end{tabular}
\end{table}

\textbf{Computational cost.} These $K{=}5$ extra gradient queries are on the
\emph{surrogate}, which the attacker controls and runs locally; CG-DI adds
zero target-model queries (a transfer attack never queries the target while
crafting), so it costs nothing against a query-limited black-box budget, the
constrained resource in API attacks. The only cost is surrogate compute: in
per-image mode CG-DI adds $\sim$75\% wall-clock overhead on a $T{=}10$ attack
(\cref{tab:runtime_cost}); and since LGC is consistent within a model
(\cref{sec:per_image}), pre-computing it once on a small calibration batch
($N{=}50$) and reusing the binary decision amortizes this to $\sim$5\% (and to
essentially zero across many attacks with the same surrogate).

\subsection{Main Results}
\label{sec:results}

CG-DI is a guardrail, so the right baseline to beat is \emph{the better of}
MI-FGSM ($p{=}0$) and DI-FGSM: it should match the stronger one per surrogate
rather than lose to blind DI on robust sources.

\begin{table}[t!]
\caption{CG-DI on CIFAR-10 ($N{=}1{,}000$, 5 seeds, $r{=}0.9$, blind DI at
$p{=}0.8$, source and target clean-correct basis), transferred to a naturally
trained WRN-28-10 and to a robust target, neither of which appears in the
surrogate pool. ``Off'' is the fraction of images the per-image rule disables DI
on; the last two columns are CG-DI minus the better of MI-FGSM and blind DI in
pp, averaged over the two targets. Against the robust target the standard
sources reach only $1.7$--$2.9\%$, so it says little about them.}
\label{tab:cifar_cgdi}
\centering
\small
\begin{tabular}{lcc|cccc|cc}
\toprule
& & & \multicolumn{4}{c|}{$\to$ WRN-28-10 (standard)} & \multicolumn{2}{c}{vs.\ best} \\
Source & LGC & off & MI-FGSM & blind DI & per-img & model & per-img & model \\
\midrule
Engstrom & 0.995 & 100\% & \textbf{20.93} & 13.63 & 20.93 & \textbf{20.93} & $+0.00$ & $+0.00$ \\
Rice & 0.997 & 100\% & \textbf{17.68} & 10.39 & 17.68 & \textbf{17.68} & $+0.00$ & $+0.00$ \\
\midrule
DenseNet-121 & 0.872 & 36\% & 89.74 & \textbf{92.31} & 90.92 & \textbf{92.31} & $-0.76$ & $+0.00$ \\
ResNet-50 & 0.904 & 52\% & 92.68 & 92.94 & 93.15 & \textbf{92.94} & $-0.15$ & $+0.00$ \\
ResNet-18 & 0.923 & 65\% & \textbf{90.79} & 89.21 & 89.89 & \textbf{90.79} & $-0.53$ & $-0.09$ \\
VGG-16 & 0.937 & 83\% & 75.27 & \textbf{78.32} & 75.85 & 75.27 & $-1.52$ & $-1.83$ \\
\midrule
\multicolumn{7}{r|}{\emph{mean over all six sources}} & $-0.49$ & $\mathbf{-0.32}$ \\
\bottomrule
\end{tabular}
\end{table}

\textbf{CIFAR-10.} \Cref{tab:cifar_cgdi} shows the guardrail behaving as designed
on the side it is for and costing something on the other. On both robust
surrogates it recovers the full $7.3$ pp that blind DI would have given away. The
standard side is where the fixed threshold is paid for, because
$\tau_{\text{op}}{=}0.92$ falls inside the standard surrogates' LGC range rather
than below it. The recommended model-level rule picks the better fixed setting
outright on two of the four and gives up $0.09$ and $1.83$ pp on the other two,
$-0.32$ pp on average. The per-image variant splits its decisions on 36--83\% of
the images and does slightly worse at $-0.49$ pp; \cref{sec:heldout} gives the
sharper reason for preferring the model-level rule. This is the price of a
fixed threshold on this dataset.

\textbf{ImageNet.} On ImageNet (\cref{tab:imagenet}; per-target breakdown in
\cref{sec:imagenet_extended}), CG-DI prevents the $10.3{\pm}0.3$ pp average ASR
degradation that blind DI causes on the robust source: for Engstrom it sets
$p{=}0$ and matches MI-FGSM's $76.0\%$; for ResNet50 it enables $p{=}0.8$ and
captures the majority of DI's gains. On the robust surrogates we evaluate, CG-DI
matches MI-FGSM and so does not lose to blind DI.

\subsection{A Held-Out Test}
\label{sec:heldout}

The results above are on surrogates the rule was developed against. To ask whether
it transfers, we froze $\tau_{\text{op}}{=}0.92$, the probe scale
$\epsilon_{chk}{=}1/255$, $K{=}5$ and the binary choice $p\in\{0,0.8\}$ at their
published values, wrote each surrogate's predicted sign to disk, took a local
SHA-256 of the prediction payload, and only then ran the attacks. The record
covers fifteen surrogates; the seven marked held out are the ones whose signs
were specified prospectively, and only those are scored here. There is no
external timestamp, so the guarantee is that the payload was fixed before the
source-side outcomes existed, not that a third party witnessed it. Seven
surrogates were used as sources for the first
time: five standard architectures and two adversarial-training recipes not in the
development set, ARES ConvNeXt-AT~\citep{liu2025comprehensive} and Singh
ConvStem-AT~\citep{singh2023revisiting}.

\begin{table}[t!]
\caption{Held-out evaluation of the CG-DI rule ($N{=}1{,}000$, 5 seeds, $r{=}0.9$,
four cross-family ImageNet targets, a surrogate never scored against itself). $D$
is the DI effect in pp on the source and target clean-correct basis, pooled over images
and targets; intervals are image-paired bootstrap. The CG-DI column is the
model-level decision; on these seven surrogates it coincides with the per-image
rule except on ViT-B/16 and ConvNeXt-B, discussed in the text. Predictions were
hashed before these runs.}
\label{tab:heldout}
\centering
\small
\begin{tabular}{llccrcc}
\toprule
Source & Type & LGC & CG-DI & $D$ ($p{=}0.8$) & 95\% CI & $D$ ($p{=}1$) \\
\midrule
ARES ConvNeXt-AT & Rob & 0.999 & $p{=}0$ & $-2.17$ & $[-3.16, -1.16]$ & $-3.49$ \\
Singh ConvStem-AT & Rob & 1.000 & $p{=}0$ & $-0.81$ & $[-1.74, +0.15]$ & $-1.85$ \\
\midrule
Swin-B & Std & 0.390 & $p{=}0.8$ & $+23.30$ & $[+21.26, +25.27]$ & $+24.10$ \\
DenseNet-121 & Std & 0.759 & $p{=}0.8$ & $+14.15$ & $[+12.73, +15.52]$ & $+13.62$ \\
InceptionV3 & Std & 0.814 & $p{=}0.8$ & $+9.76$ & $[+8.36, +11.18]$ & $+9.20$ \\
ConvNeXt-B & Std & 0.875 & $p{=}0.8$ & $+16.23$ & $[+14.78, +17.78]$ & $+16.96$ \\
ViT-B/16 & Std & 0.902 & $p{=}0.8$ & $+13.23$ & $[+11.56, +14.97]$ & $+13.16$ \\
\bottomrule
\end{tabular}
\end{table}

\Cref{tab:heldout} is the result: the rule matches the sign of the effect on all
seven surrogates, and six of the seven effects are individually resolved. The one
that is not is Singh ConvStem-AT, whose interval contains zero; we count it as a
correct sign and not as a confirmed harm. Two things the table does not show
belong here. First, the same frozen rule applied
\emph{per image} rather than per model disagrees with itself on the two standard
surrogates closest to the threshold: it disables DI on $61.7\%$ of images for
ViT-B/16 and $42.7\%$ for ConvNeXt-B, costing $7.4$ and $6.4$ pp against blind DI.
We evaluate and recommend the model-level rule, and every table states which rule
it reports. Second, the same protocol on CIFAR-10 puts $\tau_{\text{op}}$
inside the standard surrogates' LGC range ($0.872$--$0.937$ in
\cref{tab:cifar_cgdi}), so it classifies two of the four as robust-like and
disables DI on them. Only one of those two is an error: against the standard
target VGG-16 gains $+3.05$ with DI ($78.32$ against $75.27$), so switching it off
costs something, whereas ResNet-18 is harmed at $-1.58$ ($89.21$ against $90.79$)
despite standard training, and disabling DI there is the right call for a reason
training metadata would miss. We leave the threshold where it was frozen;
re-fitting it per dataset would turn a held-out test into a fitted one.

\textbf{The unseen dataset, reported in full.} CIFAR-100 was registered as a
third dataset before these runs, and it is where the rule does worst. At the
default $r{=}0.9$ the DI effect on
CIFAR-100 lies between $-0.10$ and $+1.45$ pp across the seven surrogates, that
is, the Scissors Effect has essentially vanished at that operating point, and the
rule is correct on 3 of 7, with only 4 of the 7 effects resolved away from zero
at all. Under aggressive resize ($r{=}0.6$) the effect
reappears weakly and the rule is correct on 5 of 7. \Cref{tab:cifar100_full}
gives all seven surrogates, their predictions and their intervals at both rates.
We do not count CIFAR-100 as
a successful dataset transfer. What it shows is a scope boundary: where the
effect is a percentage point or less on five of seven surrogates and never more
than $1.5$, a rule that predicts its sign has little to predict, and the probe's regime split carries no information about an
outcome that is itself indistinguishable from zero.

\subsection{Ablations}
\label{sec:ablations}

\textbf{Threshold sensitivity.} We ablate $\tau\in[0.80,0.98]$ on ImageNet
(target Swin-B, $N{=}500$, 3 seeds; \cref{sec:sensitivity}). The robust source
(Engstrom) holds $\sim$67.1\% across $\tau\in[0.80,0.96]$, dropping only at the
extreme $\tau{=}0.98$; the standard source (ResNet50) plateaus at $\sim$53.3\% for
$\tau\ge0.86$. This gives a broad safe zone $\tau\in[0.86,0.96]$ with our default
$\tau{=}0.92$ centered in it. The ablation routes each image separately, so it
measures the per-image variant; the model-level rule we recommend changes its
decision only when the mean LGC crosses $\tau$, and on ImageNet that gap is wide
(\cref{sec:prediction}). As long as a substantial LGC gap separates source
types, the exact $\tau$ has little practical impact.

\textbf{Why a probe rather than provenance.} We do not claim a case where the
probe beats provenance, and we did not find one. We have measured 27 model--dataset pairs: the thirteen CIFAR-10
surrogates, the seven held-out ImageNet sources and the seven CIFAR-100
sources. Every case is scored at $r{=}0.9$, with CIFAR-10 transferred to a
target outside its own surrogate pool. The two rules are each correct on 21. Each dataset is scored on
its own run's basis, all-images for CIFAR-10 and source-and-target clean-correct
for the other two, so this tally is the one place in the paper that pools
bases; it counts signs, which the basis does not change except where the effect
is already within noise. Five of each rule's six
misses are the same cases: CIFAR-10's ResNet-50, mildly harmed at $-1.24$ despite
standard training and with an LGC just below the threshold, and the four
CIFAR-100 surrogates where the effect is at most $1.5$ pp and three of the four
are small resolved positives. That leaves two
disagreements, once in each direction. Both are CIFAR-10 standard surrogates
whose LGC sits near $\tau_{\text{op}}$: the probe is right about ResNet-18, which
is harmed at $-3.16$ despite standard training, and wrong about VGG-16, which is
flat at $+0.06$. The tally is also setting-dependent, tying at 23 apiece if
CIFAR-100 is scored at $r{=}0.6$ instead. What remains is the narrower reason for
preferring a measurement: provenance is unverifiable for fine-tuned, distilled or
third-party checkpoints, so a rule that reads it cannot be applied where it is
missing. We have no measurement from such a setting and do not claim one.

\textbf{Smaller perturbation budget.} The Scissors Effect persists at
$\epsilon{=}4/255$ (\cref{tab:small_eps}), supporting that it is a property of
gradient geometry rather than of a specific budget.

\begin{table}[t!]
\caption{Small perturbation budget ($\epsilon{=}4/255$, ImageNet, DI-FGSM at
$p{=}1$, $N{=}1{,}000$, 5 seeds, all-images basis). The Scissors direction is
unchanged at a quarter of the main budget.}
\label{tab:small_eps}
\centering
\small
\begin{tabular}{lccc}
\toprule
Source & MI-FGSM & DI-FGSM & CG-DI \\
\midrule
Robust & \textbf{34.9$\pm$0.0\%} & 32.0$\pm$0.4\% & \textbf{34.9$\pm$0.0\%} \\
Standard & 32.5$\pm$0.0\% & \textbf{47.1$\pm$0.5\%} & 46.0$\pm$0.7\% \\
\bottomrule
\end{tabular}
\end{table}

\textbf{When does LGC alone mislead? A corner case.} The robustness spectrum of \cref{sec:eps_sweep,sec:recipe_panel} holds seven
robust sources: the five Salman $\epsilon$-sweep checkpoints plus Engstrom and
Mo2022. CG-DI classifies 6 of them correctly, along with the standard
reference. The single
miss is Salman $\epsilon{=}8/255$ \emph{ResNet-50} (LGC $0.80$, so CG-DI would
enable DI). This is a failure of the rule. Capacity
strain is our reading of why this checkpoint's LGC is atypical, not a controlled
finding: at the
same training $\epsilon{=}8/255$, Mo2022's \emph{ViT-B} yields LGC $1.00$
(correct), and Salman \etal's own clean accuracy for this checkpoint drops to
$54.5\%$ (a $21.6$-point fall from $\epsilon{=}0$), a documented capacity strain;
RobustBench and \citet{mo2022adversarial} likewise reserve higher-capacity
backbones for $\epsilon\ge8/255$. We report this case in \cref{sec:corner_case}.

\section{Discussion}
\label{sec:discussion}

\textbf{Implications.} Many modern attacks (\eg, Admix, SSA, SIA) stack
transformations expecting extra diversity to help. Our results show this is
unreliable once the surrogate is robust: a standard recipe can become weaker
simply because the source model changed. The point matters most for
\emph{defense evaluation}. A robustness number obtained with a robust surrogate
and DI left on is partly a measurement of the attack's own handicap, and the
error runs in the direction that flatters the defense.

Our controlled sweep also connects the effect to the
robustness--transferability trade-off studied by \citet{springer2021little} and
\citet{zhang2024why}: base MI-FGSM transferability reproduces their
little-robustness peak, and the DI effect rides on top of it as a separate axis
that changes sign while the base curve is still rising.

\textbf{Limitations.} (1)~The effect is several times smaller at $32{\times}32$
than on ImageNet, and its two sides behave differently there: the harm to robust
surrogates is present at every resize rate we tested, while the standard side's
gain appears only at moderate transform probability and vanishes by $p{=}1$. Our
central claim is therefore strongest on ImageNet, and any $32{\times}32$
statement has to name the transform setting it holds at. CIFAR-100 marks the
lower edge of that: at the default resize rate the effect spans $-0.1$ to
$+1.5$ pp across the seven surrogates and the probe is right about only 3 of
them, which we read as the effect running out rather than as the rule
transferring.
(2)~ImageNet uses three
robust surrogates spanning two architectures and two recipes; CIFAR-10 covers
nine including TRADES and proxy-distribution training, but broader coverage is
limited by public checkpoint availability. (3)~SSA interacts with DI on both
source types, and a probe that reads only the surrogate's gradient geometry
cannot detect a method-specific interference of that kind. (4)~A binary
$p\in\{0,0.8\}$ is a guardrail, not an optimizer: the ASR$(p)$ curves are
monotone, so the rule chooses between two endpoints rather than locating an
optimum, and it gives up $0.95$ pp $[0.50, 1.42]$ against the better endpoint on
the one standard surrogate where the two differ (\cref{tab:continuous_p}).
(5)~Ensemble surrogates mixing gradient regimes are unexplored. (6)~The probe can
mislead when capacity rather than robustness drives low gradient consistency, as
in the Salman $\epsilon{=}8$ ResNet-50 corner case.

\textbf{On theory.} Under an additive signal--noise gradient model with resize as
a linear operator, \cref{prop:scissors} (proved in \cref{sec:theory}, with
machine-checked algebra) shows DI helps iff $\mathrm{GCR}<\tau^\star$ and
necessarily hurts as $\mathrm{GCR}\to1$, so a crossover of the observed kind is
possible under a simplified model. Two things stop it from being more than that.
The linearization it assumes does not hold on real gradients, and the constants
it needs do not admit an estimator that survives a held-out test
(\cref{sec:measured_constants}). The model also extends formally to methods that
condition the gradient before DI (\cref{cor:ssa}), but a pre-registered
measurement shows its premise fails for SSA, whose effective regime shifts the
wrong way (\cref{sec:ssa}). Why SSA and DI compound as they do is open.

\textbf{Future work.} Two directions seem natural. First, a method-aware version
of the probe, able to recognise cases like SSA where a second transform family
is already acting on the gradient. Second, whether the effect appears in
multimodal systems; our preliminary CLIP results (\cref{sec:clip}) suggest the
probe remains informative beyond standard supervised vision models. Continuous
$p$-selection, ensemble surrogates, and a tighter link between gradient
consistency and loss-surface curvature are also left open.

\section{Conclusion}
\label{sec:conclusion}

We revisited an assumption largely taken for granted in transfer attacks: that
more input diversity always helps. It does not. For adversarially trained
surrogates, resize-based DI can substantially hurt transferability, producing the
opposite trend from standard models, a mismatch we call the Scissors Effect. The
scope is specific. The effect is strong and consistent on ImageNet and several
times smaller at $32{\times}32$, where the standard side's gain appears only at
DI's conventional transform probability; resize is its dominant component, about four
times the translation effect, though both are harmful; and a controlled robustness-strength sweep shows
the harm is graded rather than binary, appearing already in the little-robustness
regime. That places the phenomenon within the broader
robustness--transferability trade-off as a DI-specific specialization of it.

For why it happens, the measurement that carries the most weight is the
bias--variance decomposition: DI's averaging buys a large variance reduction on
noisy standard gradients and, with one exception, none on robust ones, while both pay a comparable
gradient displacement for the resize. The crossover theorem shows such a reversal is
possible under a simplified gradient model, and we report the distance between
that model and measured gradients rather than presenting it as the mechanism.
What remains unexplained is why the harm does not simply grow with transform
strength, and why SSA compounds with DI in the way it does.

The regime difference is compact enough to act on. A gradient-consistency probe,
frozen and locally hashed before the runs,
predicts the sign of the DI effect on seven unseen
ImageNet surrogates spanning two unseen training recipes, and selecting the
diversity setting from it avoids the $10.3$ pp average ASR drop that blind DI
causes on robust surrogates while keeping most of DI's benefit on standard ones.
It is a guardrail and not an optimizer, and it does not transfer as far as we
would like: it fails at per-image granularity, its threshold does not carry to
CIFAR-10, on CIFAR-100 it is correct on 3 of 7 in a setting where the effect
itself has largely vanished, and where training provenance is known and correct
it offers no measured advantage over reading it.

The practical takeaway does not depend on any of the mechanistic questions
staying open. DI should not be a default-on component of a transfer attack. When
it is left on with a robust surrogate, the attack that gets reported is weaker
than the attack that was available, and the model under evaluation looks more
robust than it is.

\subsubsection*{Broader Impact Statement}
This work studies when a standard component of transfer-based adversarial
attacks (input diversity) helps or hurts. Like all research on adversarial
transferability, the techniques could in principle be used to attack deployed
black-box models. However, our central message is cautionary and primarily
benefits \emph{defenders} and \emph{evaluators}: we show that leaving input
diversity on by default can make attacks from robust surrogates look weaker than
they actually are, leading to over-optimistic robustness estimates. Reporting
this effect, and providing a cheap gradient-geometry probe to detect it, helps
practitioners avoid systematically under-estimating the strength of transfer
attacks during defense evaluation. We use only publicly available models and
datasets and introduce no new attack capability beyond a guideline for setting an
existing hyper-parameter.

\bibliographystyle{tmlr}
\bibliography{references}

\begin{thebibliography}{54}
\providecommand{\natexlab}[1]{#1}
\providecommand{\url}[1]{\texttt{#1}}
\expandafter\ifx\csname urlstyle\endcsname\relax
  \providecommand{\doi}[1]{doi: #1}\else
  \providecommand{\doi}{doi: \begingroup \urlstyle{rm}\Url}\fi

\bibitem[Athalye et~al.(2018)Athalye, Engstrom, Ilyas, and
  Kwok]{athalye2018synthesizing}
Anish Athalye, Logan Engstrom, Andrew Ilyas, and Kevin Kwok.
\newblock Synthesizing robust adversarial examples.
\newblock In \emph{International Conference on Machine Learning (ICML)}, pp.\
  284--293. PMLR, 2018.

\bibitem[Carmon et~al.(2019)Carmon, Raghunathan, Schmidt, Duchi, and
  Liang]{carmon2019unlabeled}
Yair Carmon, Aditi Raghunathan, Ludwig Schmidt, John~C. Duchi, and Percy~S.
  Liang.
\newblock Unlabeled data improves adversarial robustness.
\newblock In \emph{Advances in Neural Information Processing Systems
  (NeurIPS)}, volume~32, 2019.

\bibitem[Croce \& Hein(2020)Croce and Hein]{croce2020reliable}
Francesco Croce and Matthias Hein.
\newblock Reliable evaluation of adversarial robustness with an ensemble of
  diverse parameter-free attacks.
\newblock In \emph{International Conference on Machine Learning (ICML)}, pp.\
  2206--2216, 2020.

\bibitem[Croce et~al.(2021)Croce, Andriushchenko, Sehwag, Debenedetti,
  Flammarion, Chiang, Mittal, and Hein]{croce2021robustbench}
Francesco Croce, Maksym Andriushchenko, Vikash Sehwag, Edoardo Debenedetti,
  Nicolas Flammarion, Mung Chiang, Prateek Mittal, and Matthias Hein.
\newblock {RobustBench}: A standardized adversarial robustness benchmark.
\newblock In \emph{Advances in Neural Information Processing Systems (NeurIPS),
  Datasets and Benchmarks Track}, 2021.

\bibitem[Deng et~al.(2009)Deng, Dong, Socher, Li, Li, and
  Fei-Fei]{deng2009imagenet}
Jia Deng, Wei Dong, Richard Socher, Li-Jia Li, Kai Li, and Li~Fei-Fei.
\newblock {ImageNet}: A large-scale hierarchical image database.
\newblock In \emph{IEEE Conference on Computer Vision and Pattern Recognition
  (CVPR)}, pp.\  248--255, 2009.

\bibitem[Dong et~al.(2018)Dong, Liao, Pang, Su, Zhu, Hu, and
  Li]{dong2018boosting}
Yinpeng Dong, Fangzhou Liao, Tianyu Pang, Hang Su, Jun Zhu, Xiaolin Hu, and
  Jianguo Li.
\newblock Boosting adversarial attacks with momentum.
\newblock In \emph{Proceedings of the IEEE/CVF Conference on Computer Vision
  and Pattern Recognition (CVPR)}, pp.\  9185--9193, 2018.

\bibitem[Dong et~al.(2019)Dong, Pang, Su, and Zhu]{dong2019evading}
Yinpeng Dong, Tianyu Pang, Hang Su, and Jun Zhu.
\newblock Evading defenses to transferable adversarial examples by
  translation-invariant attacks.
\newblock In \emph{Proceedings of the IEEE/CVF Conference on Computer Vision
  and Pattern Recognition (CVPR)}, pp.\  4312--4321, 2019.

\bibitem[Dosovitskiy et~al.(2021)Dosovitskiy, Beyer, Kolesnikov, Weissenborn,
  Zhai, Unterthiner, Dehghani, Minderer, Heigold, Gelly, Uszkoreit, and
  Houlsby]{dosovitskiy2021image}
Alexey Dosovitskiy, Lucas Beyer, Alexander Kolesnikov, Dirk Weissenborn,
  Xiaohua Zhai, Thomas Unterthiner, Mostafa Dehghani, Matthias Minderer, Georg
  Heigold, Sylvain Gelly, Jakob Uszkoreit, and Neil Houlsby.
\newblock An image is worth 16x16 words: Transformers for image recognition at
  scale.
\newblock In \emph{International Conference on Learning Representations
  (ICLR)}, 2021.

\bibitem[Engstrom et~al.(2019)Engstrom, Ilyas, Salman, Santurkar, and
  Tsipras]{engstrom2019robustness}
Logan Engstrom, Andrew Ilyas, Hadi Salman, Shibani Santurkar, and Dimitris
  Tsipras.
\newblock Robustness ({Python Library}), 2019.
\newblock URL \url{https://github.com/MadryLab/robustness}.

\bibitem[Fan et~al.(2025)Fan, Chen, Zhou, and Wang]{fan2025transferable}
Mingyuan Fan, Cen Chen, Wenmeng Zhou, and Yinggui Wang.
\newblock Transferable adversarial examples with {Bayesian} approach.
\newblock In \emph{Proceedings of the 20th ACM Asia Conference on Computer and
  Communications Security (ASIACCS)}, pp.\  517--529, 2025.

\bibitem[Ge et~al.(2023)Ge, Liu, Wang, Shang, and Liu]{ge2023boosting}
Zhijin Ge, Hongying Liu, Xiaosen Wang, Fanhua Shang, and Yuanyuan Liu.
\newblock Boosting adversarial transferability by achieving flat local maxima.
\newblock In \emph{Advances in Neural Information Processing Systems
  (NeurIPS)}, volume~36, pp.\  70141--70161, 2023.

\bibitem[Goodfellow et~al.(2015)Goodfellow, Shlens, and
  Szegedy]{goodfellow2015explaining}
Ian~J. Goodfellow, Jonathon Shlens, and Christian Szegedy.
\newblock Explaining and harnessing adversarial examples.
\newblock In \emph{International Conference on Learning Representations
  (ICLR)}, 2015.

\bibitem[Gowal et~al.(2020)Gowal, Qin, Uesato, Mann, and
  Kohli]{gowal2020uncovering}
Sven Gowal, Chongli Qin, Jonathan Uesato, Timothy Mann, and Pushmeet Kohli.
\newblock Uncovering the limits of adversarial training against norm-bounded
  adversarial examples.
\newblock \emph{arXiv preprint arXiv:2010.03593}, 2020.

\bibitem[Guo et~al.(2018)Guo, Rana, Ciss{\'e}, and van~der
  Maaten]{guo2018countering}
Chuan Guo, Mayank Rana, Moustapha Ciss{\'e}, and Laurens van~der Maaten.
\newblock Countering adversarial images using input transformations.
\newblock In \emph{International Conference on Learning Representations
  (ICLR)}, 2018.

\bibitem[He et~al.(2016)He, Zhang, Ren, and Sun]{he2016deep}
Kaiming He, Xiangyu Zhang, Shaoqing Ren, and Jian Sun.
\newblock Deep residual learning for image recognition.
\newblock In \emph{Proceedings of the IEEE/CVF Conference on Computer Vision
  and Pattern Recognition (CVPR)}, pp.\  770--778, 2016.

\bibitem[Huang et~al.(2017)Huang, Liu, Van Der~Maaten, and
  Weinberger]{huang2017densely}
Gao Huang, Zhuang Liu, Laurens Van Der~Maaten, and Kilian~Q. Weinberger.
\newblock Densely connected convolutional networks.
\newblock In \emph{Proceedings of the IEEE/CVF Conference on Computer Vision
  and Pattern Recognition (CVPR)}, pp.\  4700--4708, 2017.

\bibitem[Ilyas et~al.(2019)Ilyas, Santurkar, Tsipras, Engstrom, Tran, and
  Madry]{ilyas2019adversarial}
Andrew Ilyas, Shibani Santurkar, Dimitris Tsipras, Logan Engstrom, Brandon
  Tran, and Aleksander Madry.
\newblock Adversarial examples are not bugs, they are features.
\newblock In \emph{Advances in Neural Information Processing Systems
  (NeurIPS)}, volume~32, 2019.

\bibitem[Kim(2020)]{torchattacks2020}
Hoki Kim.
\newblock Torchattacks: A {PyTorch} repository for adversarial attacks.
\newblock \emph{arXiv preprint arXiv:2010.01950}, 2020.

\bibitem[Krizhevsky(2009)]{krizhevsky2009learning}
Alex Krizhevsky.
\newblock Learning multiple layers of features from tiny images.
\newblock Technical report, University of Toronto, 2009.
\newblock URL
  \url{https://www.cs.toronto.edu/~kriz/learning-features-2009-TR.pdf}.

\bibitem[Lin et~al.(2020)Lin, Song, He, Wang, and Hopcroft]{lin2020nesterov}
Jiadong Lin, Chuanbiao Song, Kun He, Liwei Wang, and John~E. Hopcroft.
\newblock Nesterov accelerated gradient and scale invariance for adversarial
  attacks.
\newblock In \emph{International Conference on Learning Representations
  (ICLR)}, 2020.

\bibitem[Liu et~al.(2025)Liu, Dong, Xiang, Yang, Su, Zhu, Chen, He, Xue, and
  Zheng]{liu2025comprehensive}
Chang Liu, Yinpeng Dong, Wenzhao Xiang, Xiao Yang, Hang Su, Jun Zhu, Yuefeng
  Chen, Yuan He, Hui Xue, and Shibao Zheng.
\newblock A comprehensive study on robustness of image classification models:
  Benchmarking and rethinking.
\newblock \emph{International Journal of Computer Vision}, 133:\penalty0
  567--589, 2025.

\bibitem[Liu et~al.(2021)Liu, Lin, Cao, Hu, Wei, Zhang, Lin, and
  Guo]{liu2021swin}
Ze~Liu, Yutong Lin, Yue Cao, Han Hu, Yixuan Wei, Zheng Zhang, Stephen Lin, and
  Baining Guo.
\newblock {Swin Transformer}: Hierarchical vision transformer using shifted
  windows.
\newblock In \emph{Proceedings of the IEEE/CVF International Conference on
  Computer Vision (ICCV)}, pp.\  10012--10022, 2021.

\bibitem[Liu et~al.(2022)Liu, Mao, Wu, Feichtenhofer, Darrell, and
  Xie]{liu2022convnet}
Zhuang Liu, Hanzi Mao, Chao-Yuan Wu, Christoph Feichtenhofer, Trevor Darrell,
  and Saining Xie.
\newblock A {ConvNet} for the 2020s.
\newblock In \emph{Proceedings of the IEEE/CVF Conference on Computer Vision
  and Pattern Recognition (CVPR)}, pp.\  11966--11976, 2022.

\bibitem[Long et~al.(2024)Long, Tao, Li, Lei, and Zhang]{long2024convergence}
Sheng Long, Wei Tao, Shuohao Li, Jun Lei, and Jun Zhang.
\newblock On the convergence of an adaptive momentum method for adversarial
  attacks.
\newblock In \emph{Proceedings of the AAAI Conference on Artificial
  Intelligence}, volume~38, pp.\  14132--14140, 2024.

\bibitem[Long et~al.(2022)Long, Zhang, Zeng, Gao, Liu, Zhang, and
  Song]{long2022frequency}
Yuyang Long, Qilong Zhang, Boheng Zeng, Lianli Gao, Xianglong Liu, Jian Zhang,
  and Jingkuan Song.
\newblock Frequency domain model augmentation for adversarial attack.
\newblock In \emph{European Conference on Computer Vision (ECCV)}, pp.\
  549--566. Springer, 2022.

\bibitem[Madry et~al.(2018)Madry, Makelov, Schmidt, Tsipras, and
  Vladu]{madry2018towards}
Aleksander Madry, Aleksandar Makelov, Ludwig Schmidt, Dimitris Tsipras, and
  Adrian Vladu.
\newblock Towards deep learning models resistant to adversarial attacks.
\newblock In \emph{International Conference on Learning Representations
  (ICLR)}, 2018.

\bibitem[Mo et~al.(2022)Mo, Wu, Wang, Guo, and Wang]{mo2022adversarial}
Yichuan Mo, Dongxian Wu, Yifei Wang, Yiwen Guo, and Yisen Wang.
\newblock When adversarial training meets vision transformers: Recipes from
  training to architecture.
\newblock In \emph{Advances in Neural Information Processing Systems
  (NeurIPS)}, volume~35, pp.\  18599--18611. Curran Associates, Inc., 2022.

\bibitem[Papernot et~al.(2016)Papernot, McDaniel, and
  Goodfellow]{papernot2016transferability}
Nicolas Papernot, Patrick McDaniel, and Ian Goodfellow.
\newblock Transferability in machine learning: from phenomena to black-box
  attacks using adversarial samples.
\newblock \emph{arXiv preprint arXiv:1605.07277}, 2016.

\bibitem[Peng et~al.(2023)Peng, Xu, Cornelius, Hull, Li, Duggal, Phute, Martin,
  and Chau]{peng2023robust}
ShengYun Peng, Weilin Xu, Cory Cornelius, Matthew Hull, Kevin Li, Rahul Duggal,
  Mansi Phute, Jason Martin, and Duen~Horng Chau.
\newblock Robust principles: Architectural design principles for adversarially
  robust {CNN}s.
\newblock In \emph{34th British Machine Vision Conference (BMVC)}. BMVA, 2023.

\bibitem[Radford et~al.(2021)Radford, Kim, Hallacy, Ramesh, Goh, Agarwal,
  Sastry, Askell, Mishkin, Clark, Krueger, and Sutskever]{radford2021learning}
Alec Radford, Jong~Wook Kim, Chris Hallacy, Aditya Ramesh, Gabriel Goh,
  Sandhini Agarwal, Girish Sastry, Amanda Askell, Pamela Mishkin, Jack Clark,
  Gretchen Krueger, and Ilya Sutskever.
\newblock Learning transferable visual models from natural language
  supervision.
\newblock In \emph{International Conference on Machine Learning (ICML)}, pp.\
  8748--8763. PMLR, 2021.

\bibitem[Rice et~al.(2020)Rice, Wong, and Kolter]{rice2020overfitting}
Leslie Rice, Eric Wong, and Zico Kolter.
\newblock Overfitting in adversarially robust deep learning.
\newblock In \emph{International Conference on Machine Learning (ICML)}, pp.\
  8093--8104. PMLR, 2020.

\bibitem[Salman et~al.(2020)Salman, Ilyas, Engstrom, Kapoor, and
  Madry]{salman2020adversarially}
Hadi Salman, Andrew Ilyas, Logan Engstrom, Ashish Kapoor, and Aleksander Madry.
\newblock Do adversarially robust {ImageNet} models transfer better?
\newblock In \emph{Advances in Neural Information Processing Systems
  (NeurIPS)}, volume~33, pp.\  3533--3545, 2020.

\bibitem[Sehwag et~al.(2022)Sehwag, Mahloujifar, Handina, Dai, Xiang, Chiang,
  and Mittal]{sehwag2022proxy}
Vikash Sehwag, Saeed Mahloujifar, Tinashe Handina, Sihui Dai, Chong Xiang, Mung
  Chiang, and Prateek Mittal.
\newblock Robust learning meets generative models: Can proxy distributions
  improve adversarial robustness?
\newblock In \emph{International Conference on Learning Representations
  (ICLR)}, 2022.
\newblock URL \url{https://openreview.net/forum?id=WVX0NNVBBkV}.

\bibitem[Simonyan \& Zisserman(2015)Simonyan and Zisserman]{simonyan2014very}
Karen Simonyan and Andrew Zisserman.
\newblock Very deep convolutional networks for large-scale image recognition.
\newblock In \emph{International Conference on Learning Representations
  (ICLR)}, 2015.

\bibitem[Singh et~al.(2023)Singh, Croce, and Hein]{singh2023revisiting}
Naman~D Singh, Francesco Croce, and Matthias Hein.
\newblock Revisiting adversarial training for {ImageNet}: Architectures,
  training and generalization across threat models.
\newblock In \emph{Advances in Neural Information Processing Systems
  (NeurIPS)}, volume~36, 2023.

\bibitem[Springer et~al.(2021)Springer, Mitchell, and
  Kenyon]{springer2021little}
Jacob~M. Springer, Melanie Mitchell, and Garrett~T. Kenyon.
\newblock A little robustness goes a long way: Leveraging robust features for
  targeted transfer attacks.
\newblock In \emph{Advances in Neural Information Processing Systems
  (NeurIPS)}, volume~34, pp.\  9759--9773, 2021.

\bibitem[Szegedy et~al.(2014)Szegedy, Zaremba, Sutskever, Bruna, Erhan,
  Goodfellow, and Fergus]{szegedy2013intriguing}
Christian Szegedy, Wojciech Zaremba, Ilya Sutskever, Joan Bruna, Dumitru Erhan,
  Ian Goodfellow, and Rob Fergus.
\newblock Intriguing properties of neural networks.
\newblock In \emph{International Conference on Learning Representations
  (ICLR)}, 2014.

\bibitem[Szegedy et~al.(2016)Szegedy, Vanhoucke, Ioffe, Shlens, and
  Wojna]{szegedy2016rethinking}
Christian Szegedy, Vincent Vanhoucke, Sergey Ioffe, Jonathon Shlens, and
  Zbigniew Wojna.
\newblock Rethinking the {Inception} architecture for computer vision.
\newblock In \emph{Proceedings of the IEEE/CVF Conference on Computer Vision
  and Pattern Recognition (CVPR)}, pp.\  2818--2826, 2016.

\bibitem[Tram{\`e}r et~al.(2018)Tram{\`e}r, Kurakin, Papernot, Goodfellow,
  Boneh, and McDaniel]{tramer2018ensemble}
Florian Tram{\`e}r, Alexey Kurakin, Nicolas Papernot, Ian Goodfellow, Dan
  Boneh, and Patrick McDaniel.
\newblock Ensemble adversarial training: Attacks and defenses.
\newblock In \emph{International Conference on Learning Representations
  (ICLR)}, 2018.

\bibitem[{Trustworthy-AI-Group}(2023)]{transferattack2023}
{Trustworthy-AI-Group}.
\newblock {TransferAttack}: A {PyTorch} framework for adversarial
  transferability, 2023.
\newblock URL \url{https://github.com/Trustworthy-AI-Group/TransferAttack}.

\bibitem[Wang et~al.(2020)Wang, Wu, Huang, and Xing]{wang2020high}
Haohan Wang, Xindi Wu, Zeyi Huang, and Eric~P. Xing.
\newblock High-frequency component helps explain the generalization of
  convolutional neural networks.
\newblock In \emph{Proceedings of the IEEE/CVF Conference on Computer Vision
  and Pattern Recognition (CVPR)}, pp.\  8684--8694, 2020.

\bibitem[Wang et~al.(2024)Wang, He, Wang, and Wang]{wang2024boosting}
Kunyu Wang, Xuanran He, Wenxuan Wang, and Xiaosen Wang.
\newblock Boosting adversarial transferability by block shuffle and rotation.
\newblock In \emph{Proceedings of the IEEE/CVF Conference on Computer Vision
  and Pattern Recognition (CVPR)}, pp.\  24336--24346, 2024.

\bibitem[Wang \& He(2021)Wang and He]{wang2021enhancing}
Xiaosen Wang and Kun He.
\newblock Enhancing the transferability of adversarial attacks through variance
  tuning.
\newblock In \emph{Proceedings of the IEEE/CVF Conference on Computer Vision
  and Pattern Recognition (CVPR)}, pp.\  1924--1933, 2021.

\bibitem[Wang et~al.(2021)Wang, He, Wang, and He]{wang2021admix}
Xiaosen Wang, Xuanran He, Jingdong Wang, and Kun He.
\newblock Admix: Enhancing the transferability of adversarial attacks.
\newblock In \emph{Proceedings of the IEEE/CVF International Conference on
  Computer Vision (ICCV)}, pp.\  16158--16167, 2021.

\bibitem[Wang et~al.(2023{\natexlab{a}})Wang, Zhang, and
  Zhang]{wang2023structure}
Xiaosen Wang, Zeliang Zhang, and Jianping Zhang.
\newblock Structure invariant transformation for better adversarial
  transferability.
\newblock In \emph{Proceedings of the IEEE/CVF International Conference on
  Computer Vision (ICCV)}, pp.\  4607--4619, 2023{\natexlab{a}}.

\bibitem[Wang et~al.(2023{\natexlab{b}})Wang, Pang, Du, Lin, Liu, and
  Yan]{wang2023better}
Zekai Wang, Tianyu Pang, Chao Du, Min Lin, Weiwei Liu, and Shuicheng Yan.
\newblock Better diffusion models further improve adversarial training.
\newblock In \emph{Proceedings of the 40th International Conference on Machine
  Learning (ICML)}, volume 202 of \emph{Proceedings of Machine Learning
  Research}, pp.\  36246--36263. PMLR, 2023{\natexlab{b}}.

\bibitem[Weng et~al.(2025)Weng, Luo, and Li]{weng2025improving}
Juanjuan Weng, Zhiming Luo, and Shaozi Li.
\newblock Improving transferable targeted adversarial attack via normalized
  logit calibration and truncated feature mixing.
\newblock \emph{IEEE Transactions on Information Forensics and Security},
  20:\penalty0 4595--4609, 2025.
\newblock \doi{10.1109/TIFS.2025.3563820}.

\bibitem[Xie et~al.(2019)Xie, Zhang, Zhou, Bai, Wang, Ren, and
  Yuille]{xie2019improving}
Cihang Xie, Zhishuai Zhang, Yuyin Zhou, Song Bai, Jianyu Wang, Zhou Ren, and
  Alan~L. Yuille.
\newblock Improving transferability of adversarial examples with input
  diversity.
\newblock In \emph{Proceedings of the IEEE/CVF Conference on Computer Vision
  and Pattern Recognition (CVPR)}, pp.\  2730--2739, 2019.

\bibitem[Xu et~al.(2020)Xu, Zhang, Luo, Xiao, and Ma]{xu2019frequency}
Zhi-Qin~John Xu, Yaoyu Zhang, Tao Luo, Yanyang Xiao, and Zheng Ma.
\newblock Frequency principle: Fourier analysis sheds light on deep neural
  networks.
\newblock \emph{Communications in Computational Physics}, 28\penalty0
  (5):\penalty0 1746--1767, 2020.

\bibitem[Yang et~al.(2024)Yang, Lin, Li, Zhao, Fan, Zhou, Wang, Liu, and
  Shen]{yang2024quantization}
Yulong Yang, Chenhao Lin, Qian Li, Zhengyu Zhao, Haoran Fan, Dawei Zhou, Nannan
  Wang, Tongliang Liu, and Chao Shen.
\newblock Quantization aware attack: Enhancing transferable adversarial attacks
  by model quantization.
\newblock \emph{IEEE Transactions on Information Forensics and Security},
  19:\penalty0 3265--3278, 2024.

\bibitem[Zagoruyko \& Komodakis(2016)Zagoruyko and
  Komodakis]{zagoruyko2016wide}
Sergey Zagoruyko and Nikos Komodakis.
\newblock Wide residual networks.
\newblock In \emph{British Machine Vision Conference (BMVC)}, 2016.

\bibitem[Zhang et~al.(2019)Zhang, Yu, Jiao, Xing, El~Ghaoui, and
  Jordan]{zhang2019theoretically}
Hongyang Zhang, Yaodong Yu, Jiantao Jiao, Eric~P. Xing, Laurent El~Ghaoui, and
  Michael~I. Jordan.
\newblock Theoretically principled trade-off between robustness and accuracy.
\newblock In \emph{International Conference on Machine Learning (ICML)}, pp.\
  7472--7482. PMLR, 2019.

\bibitem[Zhang et~al.(2024)Zhang, Hu, Zhang, Shi, Li, Liu, Wan, and
  Jin]{zhang2024why}
Yechao Zhang, Shengshan Hu, Leo~Yu Zhang, Junyu Shi, Minghui Li, Xiaogeng Liu,
  Wei Wan, and Hai Jin.
\newblock Why does little robustness help? a further step towards understanding
  adversarial transferability.
\newblock In \emph{IEEE Symposium on Security and Privacy (SP)}, pp.\
  3365--3384, 2024.

\bibitem[Zhu et~al.(2023)Zhu, Ren, Sui, Yang, and Jiang]{zhu2023boosting}
Hegui Zhu, Yuchen Ren, Xiaoyan Sui, Lianping Yang, and Wuming Jiang.
\newblock Boosting adversarial transferability via gradient relevance attack.
\newblock In \emph{Proceedings of the IEEE/CVF International Conference on
  Computer Vision (ICCV)}, pp.\  4741--4750, 2023.

\end{thebibliography}

\appendix
\section{Appendix}

\noindent\textbf{Code and data.} Code to reproduce the experiments, together with
the symbolic and numerical verification of the theory, is available at
\url{https://github.com/Avalon-S/ScissorsEffect}. Section, figure, and table numbers in
this appendix continue the main-paper numbering.

\section{A Bias--Variance Theorem for the Scissors Effect}
\label{sec:theory}

We make the bias--variance account of \cref{sec:bias_variance} precise and prove
\cref{prop:scissors}, which is restated below as \cref{thm:scissors}; the two are
the same statement and we refer to the appendix numbering from here on.
The model abstracts the attack as choosing a perturbation
direction from the surrogate gradient and transferring it to a target whose
vulnerable direction is aligned with the surrogate signal (consistent with the
observation that robust features transfer~\citep{salman2020adversarially}).

\subsection{Model and assumptions}

\begin{assumption}[Additive gradient model]
\label{ass:model}
At a fixed input $x$, the surrogate input-gradient is $g=\mu+\eta$, where
$\mu\in\mathbb{R}^n$ is a deterministic signal and $\eta$ is zero-mean noise with
$\mathrm{Cov}(\eta)=\sigma^2 I_n$. The gradient regime is summarized by the
signal-to-noise ratio $\rho:=\|\mu\|^2/(n\sigma^2)$.
\end{assumption}

\begin{assumption}[Resize-based DI]
\label{ass:di}
Input diversity replaces $g$ by $\bar g=\tfrac1m\sum_{i=1}^m g_i$ with
$g_i=R\mu+\zeta_i$ i.i.d., where $R\in\mathbb{R}^{n\times n}$ is the resize
operator and the noise passes through it, $\zeta_i=R\eta_i$ with $\eta_i$
isotropic, so $\mathrm{Cov}(\zeta_i)=\sigma^2 RR^{\top}$.
Write $c=\cos(R\mu,\mu)$, $g_r=\|R\mu\|^2/\|\mu\|^2$, and
$\kappa=\mathrm{tr}(RR^{\top})/(nm)$. We do \emph{not} assume $R$ is symmetric or
a contraction: neither property is used below, and the only structural condition
the argument needs is $\kappa/g_r<c^2<1$, stated explicitly in
\cref{thm:scissors}. A resize operator happens to satisfy the stronger
properties, but the crossover does not depend on them
(\cref{sec:theory_verify}).
\end{assumption}

\begin{assumption}[Transfer figure of merit]
\label{ass:fom}
The transferable target direction is proportional to $\mu$, and attack quality is
the signal-to-RMS alignment
$A(d)=\langle\mathbb{E}[d],\hat\mu\rangle/\sqrt{\mathbb{E}\|d\|^2}$, with
$\hat\mu=\mu/\|\mu\|$, evaluated at $d=g$ (no DI) or $d=\bar g$ (DI).
\end{assumption}

The metric $A$ uses only the first two moments of $d$, so it is exact: no
high-dimensional concentration is invoked. It is not the mean cosine
$\mathbb{E}[\cos(d,\mu)]$, from which it differs by a Jensen gap, and we claim no
inequality between them; what we can report is that the two track each other
almost exactly in simulation (\cref{fig:theory_mc} below).

\subsection{Lemmas and theorem}

\begin{lemma}[GCR--SNR identity]
\label{lem:lgc}
Under \cref{ass:model}, for independent draws $g,g'$,
\begin{equation}
  \mathrm{GCR}\;:=\;\frac{\mathbb{E}\langle g,g'\rangle}
  {\sqrt{\mathbb{E}\|g\|^2\,\mathbb{E}\|g'\|^2}}
  \;=\;\frac{\rho}{1+\rho},
  \label{eq:lgc_snr}
\end{equation}
which is strictly increasing in $\rho$.
\end{lemma}

\begin{proof}
Independence and $\mathbb{E}\eta=0$ give
$\mathbb{E}\langle g,g'\rangle=\langle\mathbb{E}g,\mathbb{E}g'\rangle=\|\mu\|^2$ and
$\mathbb{E}\|g\|^2=\|\mu\|^2+\operatorname{tr}(\sigma^2 I_n)=\|\mu\|^2+n\sigma^2$.
Dividing through by $n\sigma^2$ and using $\rho=\|\mu\|^2/(n\sigma^2)$ gives
\cref{eq:lgc_snr}.
\end{proof}

\begin{lemma}[Closed forms]
\label{lem:closed}
Under \cref{ass:model,ass:di,ass:fom},
\begin{equation}
  A_{\mathrm{noDI}}=\sqrt{\frac{\rho}{1+\rho}}=\sqrt{\mathrm{GCR}},
  \qquad
  A_{\mathrm{DI}}=\frac{c}{\sqrt{1+\kappa/(g_r\rho)}}.
  \label{eq:closed}
\end{equation}
\end{lemma}

\begin{proof}
For the baseline, $\mathbb{E}[g]=\mu$ gives $\langle\mathbb{E}[g],\hat\mu\rangle=\|\mu\|$
and $\mathbb{E}\|g\|^2=\|\mu\|^2+n\sigma^2$, so
\begin{equation*}
  A_{\mathrm{noDI}}=\frac{\|\mu\|}{\sqrt{\|\mu\|^2+n\sigma^2}}
  =\sqrt{\frac{\rho}{1+\rho}}.
\end{equation*}
For DI, $\mathbb{E}[\bar g]=R\mu$ gives
$\langle R\mu,\hat\mu\rangle=\mu^\top R\mu/\|\mu\|=c\sqrt{g_r}\,\|\mu\|$, while
$\mathbb{E}\|\bar g\|^2=\|R\mu\|^2+\tfrac{\sigma^2}{m}\operatorname{tr}(RR^{\top})
=g_r\|\mu\|^2+n\sigma^2\kappa$. Hence
\begin{equation*}
  A_{\mathrm{DI}}=\frac{c\sqrt{g_r}\,\|\mu\|}{\sqrt{g_r\|\mu\|^2+n\sigma^2\kappa}}
  =\frac{c}{\sqrt{1+\kappa/(g_r\rho)}}.
\end{equation*}
\end{proof}

\begin{theorem}[Scissors crossover; \cref{prop:scissors} restated and proved]
\label{thm:scissors}
Under \cref{ass:model,ass:di,ass:fom}, assume $c>0$ and $\kappa/g_r<c^2<1$, and
define
\begin{equation}
  \varphi(\rho):=\left(\frac{A_{\mathrm{DI}}}{A_{\mathrm{noDI}}}\right)^{2}
  =c^2\,\frac{\rho+1}{\rho+\kappa/g_r}.
  \label{eq:phi}
\end{equation}
Then $\varphi$ is strictly decreasing, with $\varphi(0^+)=c^2 g_r/\kappa>1$ and
$\varphi(\infty)=c^2<1$, so it has a unique root $\rho^\star>0$ given by
\begin{equation}
  \rho^\star=\frac{c^2-\kappa/g_r}{1-c^2}.
  \label{eq:rho_star}
\end{equation}
Consequently $A_{\mathrm{DI}}>A_{\mathrm{noDI}}$ iff $\rho<\rho^\star$, and
$A_{\mathrm{DI}}<A_{\mathrm{noDI}}$ iff $\rho>\rho^\star$. Equivalently, by
\cref{lem:lgc}, DI helps iff $\mathrm{GCR}<\tau^\star$ and hurts iff
$\mathrm{GCR}>\tau^\star$, where $\tau^\star=\rho^\star/(1+\rho^\star)$. In
particular, as $\mathrm{GCR}\to1$, DI necessarily hurts.
\end{theorem}

\begin{proof}
By \cref{lem:closed},
$\varphi(\rho)=c^2(1+1/\rho)/(1+\kappa/(g_r\rho))$, which simplifies to
\cref{eq:phi}. Differentiating,
\begin{equation*}
  \varphi'(\rho)=c^2\,\frac{(\rho+\kappa/g_r)-(\rho+1)}{(\rho+\kappa/g_r)^2}
  =c^2\,\frac{\kappa/g_r-1}{(\rho+\kappa/g_r)^2}<0,
\end{equation*}
since $\kappa/g_r<1$; hence $\varphi$ is strictly decreasing. The limits are
immediate, and $\varphi(0^+)=c^2 g_r/\kappa>1\Leftrightarrow\kappa/g_r<c^2$. Being
continuous and strictly decreasing from above $1$ to below $1$, $\varphi$ attains
the value $1$ at a unique $\rho^\star$; solving $c^2(\rho+1)=\rho+\kappa/g_r$ yields
\cref{eq:rho_star}. The $\mathrm{GCR}$ statement then follows from the strict
monotonicity of $\mathrm{GCR}(\rho)$ in \cref{lem:lgc}.
\end{proof}

\noindent The hypothesis $c>0$ is not cosmetic. $\varphi$ compares
\emph{squared} alignments, and $A_{\mathrm{DI}}$ is signed, so without it a
transform that reverses the signal direction can satisfy $\varphi>1$ while
$A_{\mathrm{DI}}<0<A_{\mathrm{noDI}}$. A $120^\circ$ rotation in two dimensions
($c={-}0.5$, $g_r{=}1$, $\kappa{=}0.05$, $\rho{=}0.01$) gives $\varphi\approx4.21$
with $A_{\mathrm{DI}}\approx-0.20$ against $A_{\mathrm{noDI}}\approx+0.10$. A
resize operator has $c>0$, so the hypothesis costs nothing here, but the
conclusion does not survive without it.

\Cref{ass:fom} takes the target's vulnerable direction to be $\mu$ itself, which is
a strong premise. The next proposition shows the theorem does not need it:
it survives for an \emph{arbitrary} target direction, under a weaker condition that
\cref{sec:alignment} measures directly.

\begin{proposition}[General target direction; relaxing \cref{ass:fom}]
\label{prop:general_target}
Let the target's vulnerable direction be an arbitrary unit vector $u$ (not assumed
equal to $\mu$), and assume the surrogate signal is positively aligned with it
both before and after resize: $\langle\mu,u\rangle>0$ and $\langle R\mu,u\rangle>0$.
Write $\gamma_\mu=\cos(\mu,u)>0$ and $\gamma_R=\cos(R\mu,u)>0$, so the figure of
merit uses $\langle\mathbb{E}[d],u\rangle$ in place of
$\langle\mathbb{E}[d],\hat\mu\rangle$. Then
\begin{equation}
  \varphi(\rho)=\Gamma\,\frac{g_r(\rho+1)}{g_r\rho+\kappa},
  \qquad \Gamma:=\frac{\gamma_R^2}{\gamma_\mu^2},
  \label{eq:phi_general}
\end{equation}
which is still strictly decreasing (as $\kappa<g_r$), with $\varphi(\infty)=\Gamma$.
Hence DI necessarily hurts as $\mathrm{GCR}\to1$ if and only if $\Gamma<1$,
i.e.\ $\gamma_R<\gamma_\mu$: resize moves the surrogate signal away from the
target's vulnerable direction. \Cref{thm:scissors} is the special case $u=\mu$
($\gamma_\mu{=}1$, $\gamma_R{=}c$, $\Gamma{=}c^2$).
\end{proposition}

\begin{proof}
With the target $u$, $A_{\mathrm{noDI}}^2=\gamma_\mu^2\,\rho/(\rho+1)$ and
$A_{\mathrm{DI}}^2=\gamma_R^2\,g_r\rho/(g_r\rho+\kappa)$, whose ratio is
\cref{eq:phi_general}. The factor $\Gamma$ does not depend on $\rho$, so
$\varphi'$ has the same sign as in \cref{thm:scissors} ($\propto\kappa-g_r<0$), and
$\varphi(\infty)=\Gamma$, $\varphi(0^+)=\Gamma g_r/\kappa$. All steps are
machine-verified (\cref{sec:theory_verify}).
\end{proof}

\noindent The two positivity hypotheses play the role $c>0$ plays in
\cref{thm:scissors}, and for the same reason: $\varphi$ is a ratio of
\emph{squared} alignments, so without them a transform that sends $R\mu$ to the
far side of $u$ can raise $\varphi$ while flipping the sign of
$A_{\mathrm{DI}}$. They also rule out the degenerate $\gamma_\mu{=}0$, at which
$\Gamma$ is undefined. Both hold in the regime the proposition is used for: a
surrogate whose clean gradient is already a useful attack direction on the target
has $\langle\mu,u\rangle>0$, and resize does not reverse it
($\gamma_R>0$).

\Cref{thm:scissors} fixes one operator $R$. DI draws a fresh transform at every
step, and the estimator used in \cref{sec:measured_constants} was built for that
case, so we state it as its own proposition. \Cref{thm:scissors} is the special
case of zero transform randomness.

\begin{proposition}[Random transforms]
\label{prop:random_transform}
Replace \cref{ass:di} by: the attack averages $m$ i.i.d.\ draws
$\bar g=\tfrac1m\sum_{i=1}^m T_i(\mu+\eta_i)$, where the transforms $T_i$ are
i.i.d.\ and independent of the noise $\eta_i$, all with finite second moments.
Write $M=\mathbb{E}[T]$, and
\begin{equation*}
  c=\cos(M\mu,\mu)>0,\quad g_r=\frac{\|M\mu\|^2}{\|\mu\|^2}>0,\quad
  v_r=\frac{\mathbb{E}\|T\mu-M\mu\|^2}{\|\mu\|^2},\quad
  t=\frac{\mathbb{E}\operatorname{tr}(T^{\top}T)}{n},
\end{equation*}
so that $v_r$ is the transform's own variance and $t$ the noise gain. Then
\begin{equation}
  \varphi(\rho)=\frac{c^2(1+\rho)}{a\rho+b},\qquad
  a=1+\frac{v_r}{m\,g_r},\qquad b=\frac{t}{m\,g_r},
  \label{eq:phi_random}
\end{equation}
and a unique positive crossover exists if and only if $b<c^2<a$, at
\begin{equation}
  \rho^\star=\frac{c^2-b}{a-c^2},\qquad
  \tau^\star=\frac{\rho^\star}{1+\rho^\star}=\frac{c^2-b}{a-b}.
  \label{eq:rho_star_random}
\end{equation}
Here $m$ is the number of i.i.d.\ transformed gradients the attack actually
averages in one update, not the number of draws used to estimate the constants.
Setting $v_r{=}0$ and $T\equiv R$ recovers \cref{thm:scissors} with
$\kappa=t/m$.
\end{proposition}

\begin{proof}
$\mathbb{E}[\bar g]=M\mu$ gives $\langle\mathbb{E}[\bar g],\hat\mu\rangle=c\sqrt{g_r}\|\mu\|$.
Independence of transforms and noise gives
$\mathbb{E}\|\bar g\|^2=\|M\mu\|^2+\tfrac1m(\mathbb{E}\|T\mu-M\mu\|^2+\sigma^2\,\mathbb{E}\operatorname{tr}(T^{\top}T))$,
so $\varphi=(A_{\mathrm{DI}}/A_{\mathrm{noDI}})^2$ is \cref{eq:phi_random} after
dividing by $\|\mu\|^2$ and substituting $\rho$. Since
$\varphi'(\rho)\propto b-a$, $\varphi$ is strictly decreasing iff $b<a$, with
$\varphi(0^+)=c^2/b$ and $\varphi(\infty)=c^2/a$; the crossover condition and
\cref{eq:rho_star_random} follow. As with \cref{thm:scissors}, $c>0$ is required
because $\varphi$ compares squared alignments.
\end{proof}

\begin{remark}[General target directions]
\label{rem:general_target}
\Cref{prop:general_target} is stated on $|\cos|$, so it is a claim about the
\emph{magnitude} of alignment and cannot distinguish a transform that preserves
alignment from one that reverses it; read together with $c>0$ in
\cref{thm:scissors}, the intended regime is the one where the signed alignment
stays positive.
The exact-alignment premise ``target $\propto\mu$'' is replaced by the weaker
condition $\gamma_R<\gamma_\mu$: a robust surrogate is
harmed by DI whenever resize reduces the signal's alignment with the target
direction. This requires no architectural match between surrogate and target.
Its sign is the empirical counterpart of what \cref{sec:alignment} measures: the
robust sources there show $D_{\text{sign}}<0$ (DI reduces source--target
sign-alignment), matching $\gamma_R<\gamma_\mu$. So the assumption the theorem really
needs is not assumed but observed.
\end{remark}

\begin{remark}[Why resize and not translation]
\label{rem:resize}
The bias factor enters \cref{thm:scissors} only through $c,g_r<1$, \ie\ through
$R\neq I$. Resize is a genuine low-pass contraction ($R\prec I$), giving $c,g_r<1$
and a finite crossover. A small, centered random translation, by contrast,
averages to $R\approx I$ (shifts cancel in expectation), so $c,g_r\approx1$ and the
bias factor $\approx1$: DI's translation component is then the smaller one, matching
the resize/translation decomposition in \cref{tab:transform}. The theorem also
accounts for the \emph{single sign change} in the $\epsilon$-sweep
(\cref{sec:eps_sweep}): measured GCR rises from $0.50$ at
$\epsilon_{\text{train}}{=}0$ to $0.99$ at $0.5/255$, crossing any threshold
between them, which is where the sweep's one sign change occurs. It does not
account for the graded magnitude, and GCR is not monotone in
$\epsilon_{\text{train}}$ across the robust range: it drifts down from $0.986$ to
$0.975$ over $\epsilon_{\text{train}}\in[0.5,4]/255$ and falls to $0.28$ at
$8/255$, the checkpoint of \cref{sec:corner_case}.
\end{remark}

\subsection{Corollary: a pre-method regime shift, and why it does not apply to SSA}
\label{sec:corollary}

The following corollary is what the model predicts for a method that conditions
the gradient before the DI step, and it is stated here because it is the natural
candidate explanation for the one apparent exception in \cref{tab:modern}, SSA.
SSA augments the input with Gaussian noise and a random DCT spectral down-scaling
before the gradient is taken; to first order this can be modelled as a symmetric
low-pass operator applied to the gradient, with its own ensemble averaging.

The corollary below is a correct consequence of that model, and its algebra is
machine-checked. Its premise, however, is not satisfied by SSA. Measuring the
effective gradient regime directly on real surrogates, under a prediction hashed
beforehand, every pre-method we tried \emph{lowers} it, SSA included and by an
amount set by its own strength parameter (\cref{sec:ssa}). The operator SSA
realises therefore does not satisfy $s_r/\kappa_S>1$, which is what the shift
requires. We state the corollary because the identity is used elsewhere, and do
not apply it to SSA.

\begin{corollary}[Pre-method regime shift]
\label{cor:ssa}
Let a method $M$ apply, before the DI step, a symmetric PSD operator $S$
($0\preceq S\preceq I$) to the gradient with $m_S$-fold averaging, so the
$M$-conditioned gradient is $g_M=S\mu+\xi$ with $\mathbb{E}\xi=0$ and
$\mathrm{Cov}(\xi)=\sigma^2 S^2/m_S$. Then its signal-to-noise ratio is
\begin{equation}
\rho_{\mathrm{eff}}=\frac{\|S\mu\|^2}{\sigma^2\,\mathrm{tr}(S^2)/m_S}
=\rho\cdot\frac{s_r}{\kappa_S},\qquad
s_r=\frac{\|S\mu\|^2}{\|\mu\|^2},\quad \kappa_S=\frac{\mathrm{tr}(S^2)}{n\,m_S},
\label{eq:rho_eff}
\end{equation}
and the effective consistency is
$\mathrm{GCR}_{\mathrm{eff}}=\rho_{\mathrm{eff}}/(1+\rho_{\mathrm{eff}})$. If $S$
removes more noise than signal, $\kappa_S<s_r$, then
$\mathrm{GCR}_{\mathrm{eff}}>\mathrm{GCR}$: $M$ moves the surrogate rightward along
the crossover axis of \cref{thm:scissors}. Applying \cref{thm:scissors} to $g_M$,
DI on top of $M$ helps iff $\mathrm{GCR}_{\mathrm{eff}}<\tau^\star_S$.

The shift alone does not fix the direction of the change in the DI effect. $S$
also changes the constants $c,g_r,\kappa$ that DI sees, hence the threshold, and
it leaves the residual noise anisotropic, which is outside the model
\cref{thm:scissors} assumes. Only in the special case where $S$ leaves those
constants unchanged, so that $\tau^\star_S=\tau^\star$, does a rightward shift
move a surrogate monotonically toward and then past the crossover. We use the
identity \cref{eq:rho_eff} and not that special case.
\end{corollary}

\begin{proof}
For independent draws,
\begin{equation*}
  \mathbb{E}\langle g_M,g_M'\rangle=\|S\mu\|^2,
  \qquad
  \mathbb{E}\|g_M\|^2=\|S\mu\|^2+\tfrac{\sigma^2}{m_S}\operatorname{tr}(S^2),
\end{equation*}
which give $\rho_{\mathrm{eff}}$ in \cref{eq:rho_eff} and, by \cref{lem:lgc}
applied to $g_M$, $\mathrm{GCR}_{\mathrm{eff}}=\rho_{\mathrm{eff}}/(1+\rho_{\mathrm{eff}})$.
The ratio $\rho_{\mathrm{eff}}/\rho=s_r/\kappa_S$ is immediate, and since
$x\mapsto x/(1+x)$ is increasing,
$\mathrm{GCR}_{\mathrm{eff}}>\mathrm{GCR}\Leftrightarrow s_r>\kappa_S$. The
help/hurt criterion is \cref{thm:scissors} applied to $g_M$ with
$\rho\to\rho_{\mathrm{eff}}$, which is why the threshold it is compared against is
$\tau^\star_S$ and not $\tau^\star$.
\end{proof}

\begin{remark}[the instantiation the model needs, and what SSA realises]
\label{rem:ssa}
For the corollary to apply to SSA, $S$ would have to be a DCT low-pass preserving
the low-frequency band where robust gradients concentrate while attenuating
broadband noise, giving $\kappa_S<s_r$ and a rightward shift. Measured directly,
the shift is leftward on every source we tested, so $\kappa_S>s_r$ for the
operator SSA actually realises. The likely reason is that SSA's Gaussian noise component and
the variance of its random spectral scaling both enter $\kappa_S$, and at the
strengths used they outweigh whatever the low-pass preserves; the effect grows
monotonically with SSA's scaling parameter, which is consistent with that reading.
We offer this as a description of the measurement, not as a replacement account.
\end{remark}

\subsection{Machine-checked verification}
\label{sec:theory_verify}

All algebraic steps above (the moments of \cref{lem:lgc,lem:closed}, the sign of
$\varphi'$, the limits, the root $\rho^\star$, the pre-method identity
$\rho_{\mathrm{eff}}=\rho\,s_r/\kappa_S$ of \cref{cor:ssa}, and the general-target
ratio $\varphi=\Gamma\,g_r(\rho{+}1)/(g_r\rho{+}\kappa)$ of
\cref{prop:general_target}) are verified symbolically with a computer-algebra
system; the verification is released with our code. The same script checks that
symmetry and contractivity are not load-bearing: with a deliberately
non-symmetric $R$ of spectral radius $1.30$ and $g_r{=}1.10$, so that neither
property holds, the sign of $A_{\mathrm{DI}}-A_{\mathrm{noDI}}$ still agrees with
$\rho\lessgtr\rho^\star$ at every probe point either side of the crossover.
As an independent check, we run a Monte-Carlo simulation with a DCT-based
low-pass resize operator $R$ ($n{=}64$, $m{=}10$ EOT samples) and a signal $\mu$
with mixed low/high-frequency content. These constants ($c{=}0.95$, $g_r{=}0.81$,
$\kappa{=}0.056$) are illustrative choices for the simulation, not measured from a
real surrogate, and they place the crossover at $\tau^\star{=}0.905$; a different
choice would move it (\eg, to $0.7$ or $0.99$). What the theorem guarantees is the
\emph{qualitative} behavior, not this number. With these constants, both the
theorem's $A$-metric and the directly measured $\mathbb{E}[\cos(d,\mu)]$ flip from
DI-helpful to DI-harmful at the same point, $\mathrm{GCR}\approx0.90$
(\cref{fig:theory_mc}); we note this happens to be close to CG-DI's empirically
chosen $\tau{=}0.92$, but we do not claim the theory predicts that value. The same
simulation adds a low-pass pre-operator $S$ ($s_r{=}0.79$, $\kappa_S{=}0.026$,
$m_S{=}20$, so $\rho_{\mathrm{eff}}/\rho{\approx}31\times$) and reproduces
\cref{cor:ssa} as an algebraic consequence of its own assumptions: it lifts a
standard surrogate's effective GCR from $0.55$ to $0.974$, flipping DI's effect
from helpful ($+0.19$) to near-zero ($-0.02$). This confirms the algebra
and nothing about SSA. Measured on real surrogates, SSA moves the effective
gradient regime in the opposite direction, so the operator it is modelled by does
not satisfy $s_r/\kappa_S>1$ (\cref{sec:ssa}). The crossover is visualized in
\cref{fig:theory_mc}.

\subsection{Why we do not offer a per-surrogate threshold}
\label{sec:measured_constants}

\Cref{prop:random_transform}'s constants can be estimated on real surrogates and
combined into a per-surrogate threshold, giving the rule ``DI helps iff
$\mathrm{GCR}<\tau^\star(f)$''. Scored on the surrogates used to form it, the rule
looks strong; that is a fit, not a test, since their DI direction is already
known. We therefore froze an implementation and scored it on seven surrogates
never used as attack sources.

That implementation, which we write $\widehat\tau_{\text{E6}}$ to keep it
separate from the proposition's $\tau^\star$, formed finite-sample plug-in
summaries from 20 transformed-gradient draws with the transform always applied.
Under it, six of the seven held-out surrogates had a finite crossover, with
thresholds from $0.011$ to $0.586$; Swin-B had none. It predicted harm for all
seven, and it is correct on two. The five standard sources gained between
$+9.76$ and $+23.30$ percentage points instead (\cref{tab:heldout}, at the
$p{=}0.8$ it was scored on). \Cref{tab:constants} gives the constants themselves
for all fifteen surrogates, so the estimator can be checked rather than taken on
report.

\begin{table}[h]
\caption{Measured \cref{prop:random_transform} constants on all fifteen ImageNet
surrogates ($r{=}0.9$, $m{=}20$ always-on draws, per-surrogate $N$ as in
\cref{tab:biasvar}), ordered by GCR within each group. $b=\kappa/g_r$ is the noise
gain and $a$ the transform-variance factor; a crossover exists iff $b<c^2<a$, and
$\widehat\tau_{\text{E6}}=(c^2-b)/(a-b)$ is the resulting plug-in threshold.
$^{\dagger}$ marks the seven surrogates held out of the pre-registered scoring.
The rule these constants define predicts DI harm whenever
$\mathrm{GCR}>\widehat\tau_{\text{E6}}$, which is every row with a finite
threshold; it is correct on two of the seven held-out cases.}
\label{tab:constants}
\centering
\small
\begin{tabular}{lcccccc}
\toprule
Surrogate & $c$ & $g_r$ & $b$ & $a$ & GCR & $\widehat\tau_{\text{E6}}$ \\
\midrule
\multicolumn{7}{l}{\textit{standard}} \\
Swin-B$^{\dagger}$ & $0.180$ & $0.388$ & $0.135$ & $1.000$ & $0.165$ & none \\
ResNet-50 & $0.386$ & $0.231$ & $0.122$ & $1.125$ & $0.498$ & $0.027$ \\
DenseNet-121$^{\dagger}$ & $0.439$ & $0.504$ & $0.113$ & $1.087$ & $0.629$ & $0.082$ \\
InceptionV3$^{\dagger}$ & $0.407$ & $1.486$ & $0.127$ & $1.054$ & $0.754$ & $0.042$ \\
ConvNeXt-B$^{\dagger}$ & $0.509$ & $0.444$ & $0.097$ & $1.090$ & $0.823$ & $0.164$ \\
ViT-B/16$^{\dagger}$ & $0.371$ & $0.258$ & $0.126$ & $1.186$ & $0.910$ & $0.011$ \\
\midrule
\multicolumn{7}{l}{\textit{robust}} \\
Salman $\epsilon{=}8$ & $0.440$ & $0.411$ & $0.059$ & $1.018$ & $0.280$ & $0.140$ \\
Engstrom & $0.516$ & $0.570$ & $0.063$ & $1.063$ & $0.974$ & $0.203$ \\
Salman $\epsilon{=}4$ & $0.412$ & $0.376$ & $0.080$ & $1.094$ & $0.975$ & $0.088$ \\
Salman $\epsilon{=}2$ & $0.485$ & $0.587$ & $0.071$ & $1.058$ & $0.981$ & $0.167$ \\
Salman $\epsilon{=}1$ & $0.568$ & $0.683$ & $0.062$ & $1.042$ & $0.984$ & $0.266$ \\
Salman $\epsilon{=}0.5$ & $0.620$ & $1.289$ & $0.057$ & $1.035$ & $0.986$ & $0.335$ \\
ARES ConvNeXt-AT$^{\dagger}$ & $0.683$ & $0.628$ & $0.040$ & $1.022$ & $0.999$ & $0.434$ \\
Singh ConvStem-AT$^{\dagger}$ & $0.777$ & $0.989$ & $0.025$ & $1.012$ & $0.999$ & $0.586$ \\
Mo2022 & $0.824$ & $0.850$ & $0.044$ & $1.024$ & $1.000$ & $0.648$ \\
\bottomrule
\end{tabular}
\end{table}

The mismatch between that implementation and the attack has to be stated,
because it changes what the failure means. The attack averages one transformed
gradient per update and applies the transform with probability $0.8$, whereas the
plug-in summaries used $m{=}20$ always-on draws; and the implementation estimates
$g_r$ from $\mathbb{E}\|\bar g\|^2$, which still contains a variance term that
$a$ then counts again. Re-evaluating \cref{eq:rho_star_random} at the
attack-matched $m{=}1$ leaves only one of the seven with a finite crossover at
all. What failed is therefore the plug-in predictive bridge from the model to
attack outcomes, not \cref{prop:random_transform} itself. We therefore have no
per-surrogate predictor to offer. We report the constants
(\cref{tab:constants}) and release the per-image moments they are formed from, so
that the estimator can be improved or refuted, but we do not ship a rule built on
them.

Two things survive, and they are what \cref{sec:theory} claims. The constants
are instantiable on real gradients, so the crossover the theorem describes is not
vacuous under its own assumptions: of the 15 surrogates measured, 13 have
$g_r\le1$, and 12 of those also satisfy $\kappa/g_r<c^2$. The exception is Swin-B
($c^2{=}0.032$ against $\kappa/g_r{=}0.135$), which has the lowest gradient
alignment in the set and no crossover under either estimator. The second thing
that survives is that the $\mathrm{GCR}\to1$ limit is near-definitional.
Neither yields a number. CG-DI's $\tau{=}0.92$ was and remains an empirical
choice (\cref{sec:sensitivity}), and its closeness to the illustrative
$\tau_{\text{sim}}{=}0.905$ of the simulation is a coincidence, not a prediction.

The linearization that the estimate assumes also does not hold. Modeling the
DI-averaged gradient as a fixed linear operator applied to the clean gradient
leaves a residual $\|\bar g - J_T^{\top} g\|/\|g\|$ of $1.04$--$1.28$ across the
same 13 surrogates, the subset of the 15 in \cref{tab:biasvar} whose resize
geometry can be estimated, so $R$ is not a small perturbation of the identity in
the direction the estimate needs.

\section{Experimental Hyperparameters}
\label{sec:hyperparameters}

All attacks are untargeted and optimized under the $L_\infty$ norm; the standard
configuration follows common transfer-attack practice (\cref{tab:params_attack}).

\begin{table}[h]
\caption{Hyperparameter settings for attacks.}
\label{tab:params_attack}
\centering
\begin{tabular}{lcc}
\toprule
Parameter & CIFAR-10 & ImageNet \\
\midrule
Perturbation budget ($\epsilon$) & $8/255$ & $16/255$ \\
Step size ($\alpha$) & $2/255$ & $2/255$ \\
Iterations ($T$) & 10 & 10 \\
LGC probe ($\epsilon_{chk}$) & $1/255$ & $1/255$ \\
LGC samples ($K$) & 5 & 5 \\
\bottomrule
\end{tabular}
\end{table}

\textbf{Baseline method settings.} MI-FGSM: momentum decay $\mu{=}1.0$. DI-FGSM:
transformation probability $p{=}0.5$ (default) unless swept; resize factor
$r{=}0.9$ (images resized to a random size in $[\lfloor224\times0.9\rfloor,224]$
pixels then zero-padded back; for CIFAR-10, resized into
$[\lfloor32\times0.9\rfloor,32]$ and padded back to $32$). TI-FGSM: Gaussian
kernel $k{=}15$ (ImageNet), $k{=}5$ (CIFAR-10).

\textbf{Compute.} All experiments ran on a server with an NVIDIA RTX 4090 (24\,GB)
and an Intel Xeon Platinum 8352V (16 vCPU) @ 2.10\,GHz with 120\,GB RAM, Python
3.10, CUDA 11.8.

\section{Model Architectures and Training Recipes}
\label{sec:arch}

\Cref{tab:arch} lists every surrogate, its architecture, and its training recipe.
All CIFAR-10 robust models are $L_\infty$, $\epsilon{=}8/255$ RobustBench
checkpoints. The robust sources used in the main analyses span two
architecture families (ResNet-50 and ViT-B) and several recipes (PGD-AT, TRADES,
semi-supervised, diffusion-augmented, ViT-aware AT), so the Scissors Effect cannot
be attributed to any single architecture or training procedure.

\begin{table}[h]
\caption{Surrogate architectures and training recipes. Robust accuracy (RA) for
CIFAR-10 robust models is the RobustBench~\citep{croce2021robustbench}
AutoAttack~\citep{croce2020reliable} value.}
\label{tab:arch}
\centering
\small
\begin{tabular}{lllll}
\toprule
Dataset & Model & Architecture & Recipe & RA (\%) \\
\midrule
\multirow{4}{*}{CIFAR-10 (Std)}
 & ResNet18 & ResNet-18 & natural & --- \\
 & ResNet50 & ResNet-50 & natural & --- \\
 & VGG16 & VGG-16 & natural & --- \\
 & DenseNet121 & DenseNet-121 & natural & --- \\
\midrule
\multirow{9}{*}{CIFAR-10 (Rob)}
 & Engstrom & ResNet-50 & PGD-AT & 49.3 \\
 & Rice & WRN-34-20 & PGD-AT (early stop) & 53.4 \\
 & Gowal & WRN-70-16 & PGD-AT (extra data) & 57.2 \\
 & Carmon & WRN-28-10 & semi-supervised (RST) & 59.5 \\
 & Peng & RaWRN-70-16 & synthetic data (50M) & 71.1 \\
 & Wang2023 & WRN-28-10 & diffusion data & 67.3 \\
 & Zhang (TRADES) & WRN-34-10 & TRADES & 53.1 \\
 & Sehwag & WRN-34-10 & proxy distribution & 60.3 \\
 & Sehwag-R18 & ResNet-18 & proxy distribution & 55.5 \\
\midrule
\multirow{6}{*}{ImageNet (Std)}
 & ResNet50 & ResNet-50 & natural & --- \\
 & ViT-B/16 & ViT-B/16 & natural & --- \\
 & DenseNet121 & DenseNet-121 & natural & --- \\
 & InceptionV3 & InceptionV3 & natural & --- \\
 & Swin-B & Swin-B & natural & --- \\
 & ConvNeXt-B & ConvNeXt-Base & natural & --- \\
\midrule
\multirow{3}{*}{ImageNet (Rob)}
 & Engstrom & ResNet-50 & PGD-AT, $\epsilon{=}4/255$ & --- \\
 & Salman ($\times5$) & ResNet-50 & PGD-AT, $\epsilon\in\{0.5,1,2,4,8\}/255$ & --- \\
 & Mo2022 & ViT-B & ViT-aware AT, $\epsilon{=}8/255$ & --- \\
\bottomrule
\end{tabular}
\end{table}

\section{Additional Experimental Details}
\label{sec:additional_exp}

All modern attacks were implemented via TransferAttack~\citep{transferattack2023}
following their official codebases, with default hyperparameters except the DI
probability $p$.

\subsection{Targeted Attack Generalization}
\label{sec:targeted}
We evaluate targeted MI-FGSM~\citep{dong2018boosting} on ImageNet ($N{=}1{,}000$,
5 seeds, $\epsilon{=}16/255$) with randomly assigned target classes
(\cref{tab:targeted}).

\begin{table}[h]
\caption{Targeted attack transfer ASR (\%), $N{=}1{,}000$, 5 seeds,
all-images basis. Success means the
target predicts the assigned target class, so the levels are far lower than in
the untargeted tables and are not comparable to them. DI harms robust surrogates
in this setting too.}
\label{tab:targeted}
\centering
\small
\begin{tabular}{llcccc}
\toprule
Surrogate & Target & Naked & +DI & $\Delta$ (pp) & $\Delta$ (rel.) \\
\midrule
ResNet50 (Std) & Swin-B & 0.2\% & 0.7\% & $+0.5$ & $+250\%$ \\
ResNet50 (Std) & IncV3 & 0.6\% & 1.2\% & $+0.6$ & $+100\%$ \\
\midrule
Engstrom (Rob) & Swin-B & 2.7\% & 1.8\% & $\mathbf{-0.9}$ & $-33\%$ \\
Engstrom (Rob) & IncV3 & 5.2\% & 3.6\% & $\mathbf{-1.6}$ & $-31\%$ \\
\bottomrule
\end{tabular}
\end{table}

Despite the inherently low absolute ASR of targeted cross-architecture transfer,
the direction is consistent: DI helps the standard surrogate ($\sim+0.6$ pp) and
harms the robust one ($-1.2$ pp avg, $\sim-32\%$ relative). The Scissors Effect is
thus independent of attack objective (untargeted vs.\ targeted).

\FloatBarrier
\section{Spectral Visualization of Gradient Geometry}
\label{sec:spectral}

We analyze the frequency spectra of input gradients via 2D FFT
(\cref{fig:fft_spectra}).

\begin{figure}[!ht]
    \centering
    \includegraphics[width=0.85\linewidth]{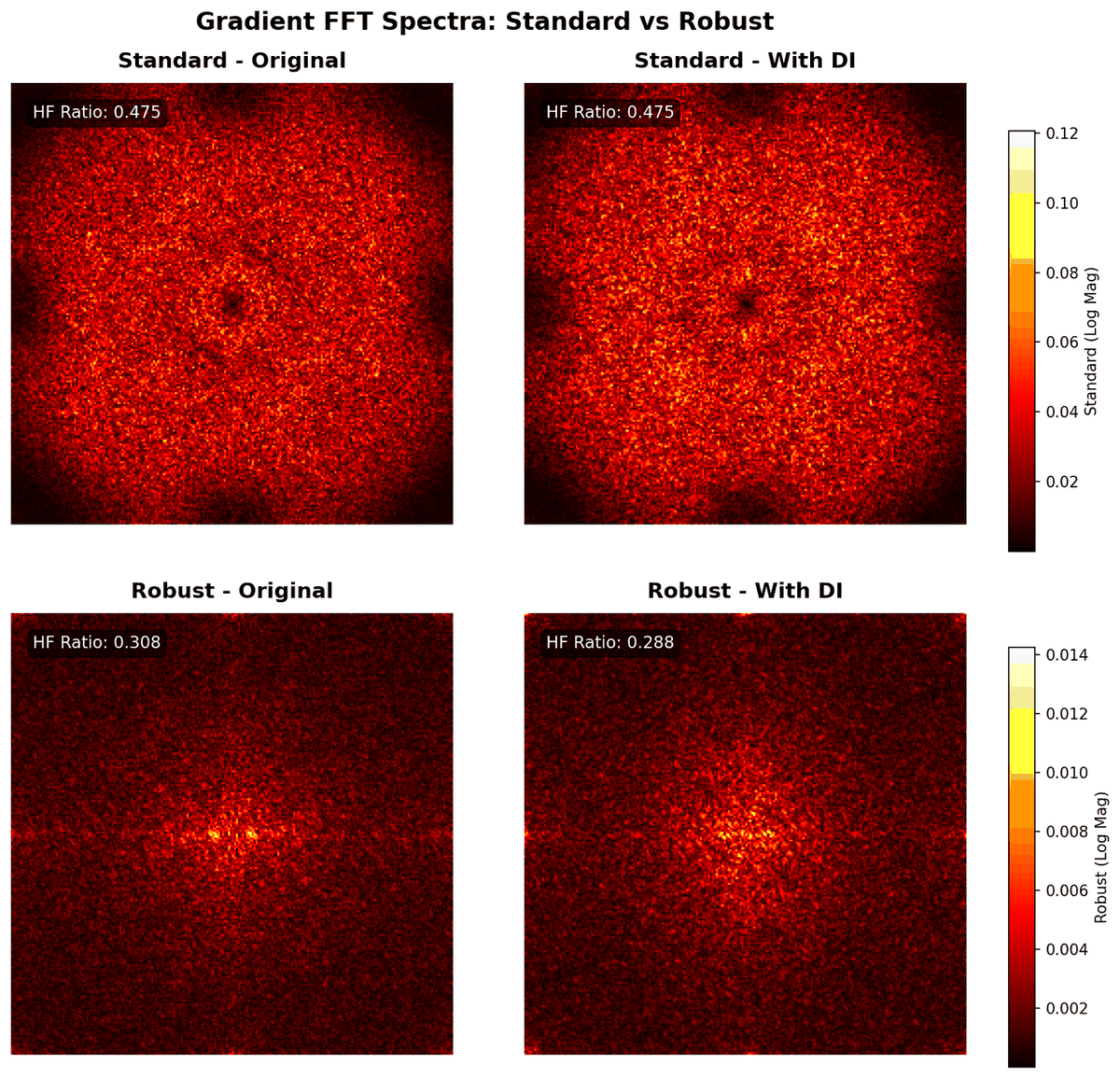}
    \caption{2D FFT spectra of gradient maps. (\emph{Top row}) Standard model
    (ResNet18): diffuse, broadband content. (\emph{Bottom row}) Robust model
    (Engstrom): concentrated low-frequency energy. Left/right columns compare
    original vs.\ DI-augmented inputs. The two rows carry \emph{separate}
    color scales (note the differing colorbars), so brightness is comparable
    within a row and not across rows. DI leaves the standard model's
    high-frequency share unchanged to three decimals ($0.475$ both sides) and
    lowers the robust model's ($0.308\to0.288$), which is the spectral form of
    the asymmetry this section is about.}
    \label{fig:fft_spectra}
\end{figure}

\textbf{High-frequency ratio.} From the 2D FFT $G(f)$ we form the log-magnitude
spectrum $S(f)=\log(1+|G(f)|)$, take its radially-averaged profile $P(r)$
(\cref{fig:radial_profile}), and define
\begin{equation}
    R_{HF} = \frac{\sum_{r > r_{\max}/2} P(r)}{\sum_r P(r)},
\end{equation}
where $r$ is the radial spatial frequency and $r_{\max}$ the maximum resolvable
frequency. The log transform prevents dominant low-frequency components from
masking the high-frequency signal.

The radial profile (\cref{fig:radial_profile}) shows a $\sim$1.7$\times$ HF gap
(standard mean $\approx0.48$ vs.\ robust
$\approx0.28$), explaining why DI's resize, which corrupts low-frequency signals via
interpolation, is particularly harmful for robust models. Among standard models,
attention-based architectures (ViT-B/16: $R_{HF}{=}0.37$) produce smoother
gradients than CNNs (ResNet50: $0.55$), consistent with their higher LGC.

\begin{figure}[!ht]
    \centering
    \includegraphics[width=0.9\linewidth]{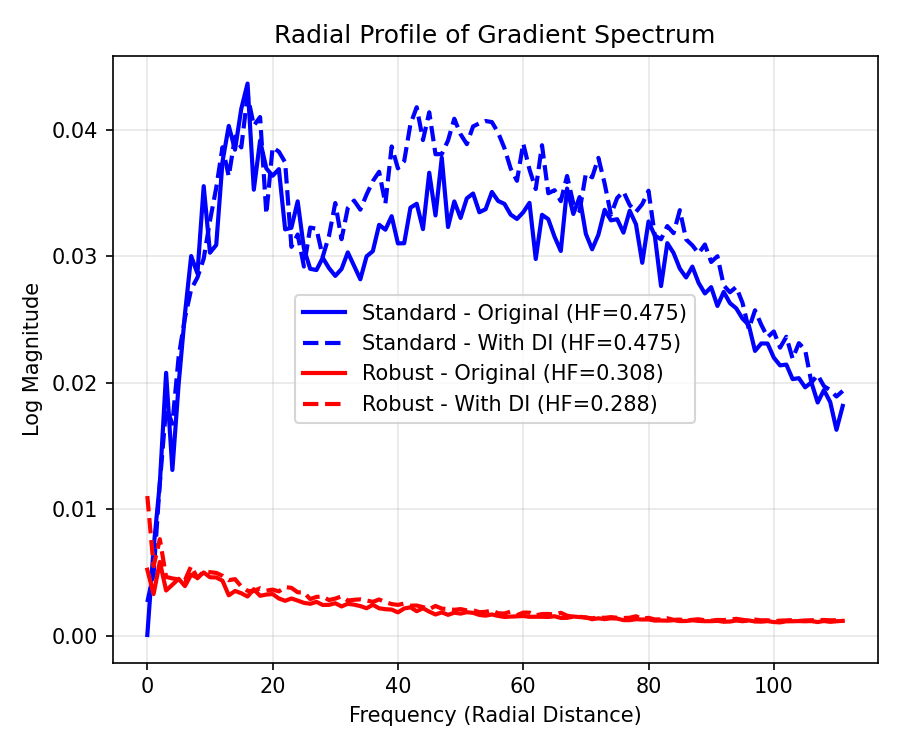}
    \caption{Radial spectral profile. Standard models (\emph{blue}) have higher
    log-magnitude at high frequencies; robust models (\emph{red}) concentrate at
    low frequencies. Each curve is the profile of the batch-averaged gradient,
    so the figure carries no uncertainty band.}
    \label{fig:radial_profile}
\end{figure}

\section{Per-Image LGC Computation}
\label{sec:per_image}

Although \cref{alg:cgdi} defines LGC per image, its variability depends on model
type. Robust models have very low per-image variance (std $<0.02$ over 1{,}000
images); standard models vary more widely (\eg, ResNet50, model-level mean
$\approx0.64$), but their mean stays well below $\tau{=}0.92$. Where the
model-level mean sits far from the threshold it is a reliable stand-in for the
per-image statistic, and one can then estimate the binary decision on a small
calibration batch ($N{=}50$), amortizing the $K{=}5$ queries across all attack
iterations.

That stand-in fails near the threshold, which is where it would matter most. On
the two held-out surrogates whose mean LGC lies closest to $\tau_{\text{op}}$ the
per-image rule disagrees with the model-level one on a large fraction of images
and costs $7.4$ and $6.4$ pp against blind DI (\cref{sec:heldout}). We therefore
evaluate and recommend the model-level rule and do not present the per-image
variant as equivalent to it.

\section{Optimal \texorpdfstring{$p^*$}{p-star} Identifiability Analysis}
\label{sec:pstar_analysis}

Two questions about the argmax $p^*$ of the ASR$(p)$ curve are easy to conflate,
and only one of them has an affirmative answer.

\emph{Can the location of the best grid value be determined?} Yes, at adequate
sample size. At $N{=}1{,}000$ with five seeds and four held-out targets
(\cref{tab:continuous_p}) the curves are strictly monotone in $p$ for all three
sources, so the best grid value is an endpoint. A smaller sweep does not show
this: at $N{=}500$ with three seeds and a single target, the best grid value beats
its runner-up by only $+0.08$, $+0.33$ and $+0.73$ pp, with every interval
containing zero.

\emph{Can neighbouring grid values near the optimum be told apart?} No, and this
is a property of a smooth curve near its endpoint rather than a defect of the
measurement. It is also not needed for anything the paper claims.

The consequence is that $p^*$ is the wrong quantity to correlate LGC against,
since on a monotone curve such a correlation only re-states the sign of the
slope, which is what \cref{tab:correlation} reports directly. The operationally
meaningful quantity is what the binary rule gives up against the endpoint: zero
on both robust sources and $0.95$ pp $[0.50, 1.42]$ on ResNet-50.

\section{Controlled \texorpdfstring{$\epsilon$}{epsilon}-Sweep: Second Target (ViT-B)}
\label{sec:eps_vitb}

\Cref{tab:eps_sweep_vitb} repeats the controlled $\epsilon$-sweep of
\cref{sec:eps_sweep} against a second target (ViT-B). The large-scale picture
matches: DI is beneficial on the standard surrogate and harmful
once $\epsilon_{\text{train}}\ge2/255$, with the harm deepening thereafter. The
two smallest robustness levels do not resolve. $D$ is $-0.7$ at $0.5$ and $+0.1$
at $1$, so the sequence is not monotone there, and neither value is separated
from zero at $N{=}500$, where the sampling uncertainty on an ASR near $90\%$ is
about $1.3$ pp. We read these two cells as ``no resolved effect'' rather than as
locating a crossover.

\begin{table}[h]
\caption{Controlled $\epsilon$-sweep, target ViT-B ($N{=}500$, 3 seeds,
$\epsilon_{\text{attack}}{=}16/255$, all-images basis; uncertainty here is
dominated by the image sample, not by seeds). $D$ is in pp; it is not resolved at
$\epsilon_{\text{train}}\le1$ and is clearly negative from $2$ onward; compare
the Swin-B target in \cref{tab:eps_sweep}.}
\label{tab:eps_sweep_vitb}
\centering
\small
\begin{tabular}{lccc}
\toprule
$\epsilon_{\text{train}}$ ($\times255$) & MI-FGSM & DI-FGSM & $D=$ DI$-$MI \\
\midrule
0 (Standard) & 34.4 & 50.4 & $+16.0$ \\
0.5 & \textbf{91.8} & 91.1 & $-0.7$ \\
1 & 90.0 & 90.1 & $+0.1$ \\
2 & 87.8 & 82.1 & $-5.7$ \\
4 & 78.8 & 71.4 & $-7.4$ \\
8 & 62.6 & 52.9 & $\mathbf{-9.7}$ \\
\bottomrule
\end{tabular}
\end{table}

\section{Cross-Recipe Replication Panel}
\label{sec:recipe_panel}

\Cref{tab:recipe_panel} gives the cross-recipe panel of \cref{sec:eps_sweep}: five
robust ImageNet surrogates spanning four adversarial-training recipes (PGD-AT,
ViT-aware AT, ARES adversarial training, ConvStem AT) and four architectures
(ResNet-50, ViT-B, ConvNeXt-B, ViT-B-ConvStem), each attacked against two targets
(Swin-B and ConvNeXt-B), with the per-source LGC measured alongside ($N{=}500$, 3
seeds, $\epsilon_{\text{attack}}{=}16/255$). DI harms every robust surrogate on all
nine cross-family pairs and benefits the standard surrogate on both; the sign of
$D$ follows the LGC regime, not the recipe or architecture. This establishes that
the Scissors direction replicates beyond the Salman ResNet-50 / PGD-AT family and
to a second target; it is a replication of the sign, not a second controlled
strength sweep (no public fixed-backbone $\epsilon$-spectrum exists outside
Salman's). The one source--target pair sharing an architecture family
(ARES ConvNeXt-B source against the ConvNeXt-B target) is marked and excluded from
the nine cross-family count.

\begin{table}[h]
\caption{Cross-recipe replication panel ($N{=}500$, 3 seeds, all-images basis,
targets Swin-B and
ConvNeXt-B). $D=$ DI$-$MI in pp. DI harms all five robust surrogates across four
recipes and four architectures on both targets; the standard surrogate benefits;
the sign tracks LGC. $\dagger$ same-family pair (excluded from the cross-family
count).}
\label{tab:recipe_panel}
\centering
\small
\begin{tabular}{lllccc}
\toprule
Source & Recipe & Arch & LGC & $D$ (Swin-B) & $D$ (ConvNeXt-B) \\
\midrule
Standard RN50 & --- & RN50 & 0.63 & $\mathbf{+13.0}$ & $\mathbf{+15.9}$ \\
\midrule
Engstrom & PGD-AT & RN50 & 0.98 & $-12.3$ & $-18.7$ \\
Salman & PGD-AT & RN50 & 0.98 & $-10.9$ & $-15.2$ \\
Mo2022 & ViT-AT & ViT-B & 1.00 & $-6.1$ & $-6.7$ \\
ARES & ARES-AT & ConvNeXt-B & 1.00 & $-5.5$ & $-7.1^{\dagger}$ \\
Singh & ConvStem-AT & ViT-B-CS & 1.00 & $-4.5$ & $-5.3$ \\
\bottomrule
\end{tabular}
\end{table}

\section{Source--Target Overlap: Same-Family Control}
\label{sec:self_transfer_app}

\Cref{tab:self_transfer} gives the full numbers for the same-family
($\mathrm{RN50}\to\mathrm{RN50}$) control of \cref{sec:self_transfer}. The DI
effect's sign is set by the source regime, not by source--target family
overlap: a standard source benefits from DI even against robust RN50 targets,
while a robust source is harmed even against a standard RN50 target.

\begin{table}[h]
\caption{Same-family control ($N{=}500$, 3 seeds, all-images basis). $D=$
DI$-$MI in pp. Cross-family baselines (against Swin-B) repeated for reference. Source
and target here share an architecture family but are distinct checkpoints; this
is the one place in the paper where that is deliberate.}
\label{tab:self_transfer}
\centering
\small
\begin{tabular}{lllccc}
\toprule
Source & Target & Relation & MI & DI & $D$ \\
\midrule
Standard RN50 & Swin-B & cross-family (baseline) & 41.2 & 54.2 & $+13.0$ \\
Standard RN50 & Engstrom RN50 & same-family (robust tgt) & 37.4 & 38.7 & $+1.3$ \\
Standard RN50 & Salman $\epsilon{=}2$ RN50 & same-family (robust tgt) & 34.0 & 35.1 & $+1.1$ \\
\midrule
Salman $\epsilon{=}2$ RN50 & Swin-B & cross-family (baseline) & 78.4 & 65.9 & $-12.5$ \\
Salman $\epsilon{=}2$ RN50 & Standard RN50 & same-family (std tgt) & 88.8 & 77.2 & $\mathbf{-11.6}$ \\
\bottomrule
\end{tabular}
\end{table}

\section{LGC Across the Robustness Spectrum, and a Corner Case}
\label{sec:corner_case}

We compute LGC on 7 robust ImageNet sources spanning the controlled
$\epsilon$-spectrum plus Engstrom and Mo2022, together with the naturally trained
ResNet-50 as a standard reference point (clean-correct subsets, $K{=}5$,
$\sigma{=}1/255$; \cref{tab:lgc_spectrum}). LGC tracks DI sensitivity
quantitatively: against the DI effect $D$ from \cref{tab:eps_sweep}, Pearson
$r(\text{LGC},D)=-0.87$, $p{=}0.025$ ($n{=}6$; \cref{fig:lgc_scatter}). That
coefficient should not be read as evidence of a graded relation: this sweep holds
the architecture and recipe fixed and varies only $\epsilon_{\text{train}}$, so it
contains a single standard model, and dropping that one point takes Pearson to
$-0.33$ ($p{=}0.59$). The correlation panel we report in the main text
(\cref{fig:lgc_harm}) therefore uses the held-out ImageNet surrogates, which
contain five standard models. CG-DI is
correct on 6 of the 7 robust sources and on the standard reference.

\begin{table}[h]
\caption{LGC across the robustness spectrum ($K{=}5$ probes at
$\sigma{=}1/255$, source clean-correct subset): 7 robust ImageNet sources plus the
naturally trained ResNet-50 as a standard reference (first row), which is not
counted among the robust sources. CG-DI is correct on 6 of the 7 robust sources
and on the reference. Bold marks the two extremes of the LGC column; the single
miss (Salman $\epsilon{=}8$ ResNet-50) is discussed in the text.}
\label{tab:lgc_spectrum}
\centering
\small
\begin{tabular}{lllcc}
\toprule
Source & Arch & $\epsilon_{\text{train}}$ & LGC & CG-DI \\
\midrule
Standard & RN50 & 0 & 0.65 & $p{=}0.8$\ \checkmark \\
Salman & RN50 & 0.5 & 0.99 & $p{=}0$\ \checkmark \\
Salman & RN50 & 1 & 0.99 & $p{=}0$\ \checkmark \\
Salman & RN50 & 2 & 0.99 & $p{=}0$\ \checkmark \\
Salman & RN50 & 4 & 0.98 & $p{=}0$\ \checkmark \\
Engstrom & RN50 & 4 & 0.98 & $p{=}0$\ \checkmark \\
Mo2022 & ViT-B & 8 & \textbf{1.00} & $p{=}0$\ \checkmark \\
Salman & RN50 & 8 & \textbf{0.80} & $p{=}0.8$\ (corner case) \\
\bottomrule
\end{tabular}
\end{table}

\begin{figure}[h]
\centering
\includegraphics[width=0.6\linewidth]{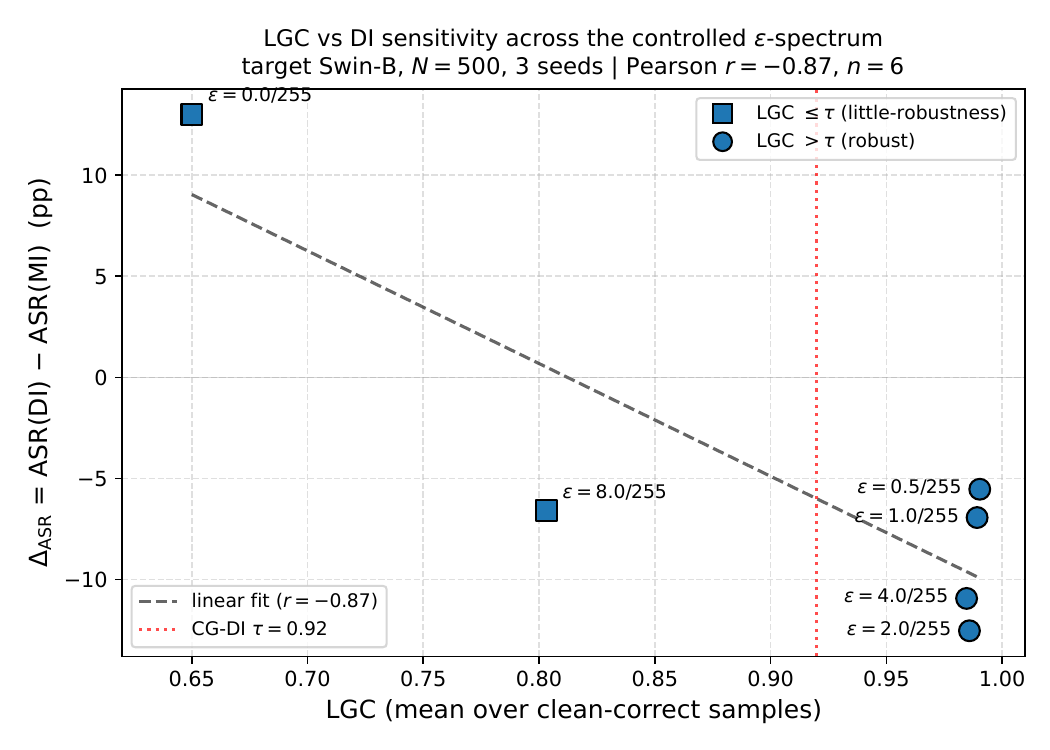}
\caption{LGC vs.\ the DI effect $D$ across the controlled $\epsilon$-spectrum
($N{=}500$, 3 seeds, all-images basis; Pearson $r{=}-0.87$, $p{=}0.025$,
$n{=}6$ surrogates, which is the unit of the correlation). The CG-DI threshold
$\tau{=}0.92$ is marked. Only one of the six is a standard model, and the
coefficient depends on it: without that point Pearson is $-0.33$ ($p{=}0.59$).
This is a description of six points, not an estimate we would extrapolate.}
\label{fig:lgc_scatter}
\end{figure}

\textbf{The corner case.} Salman $\epsilon{=}8/255$ ResNet-50 has LGC $0.80$, so
CG-DI would (incorrectly) enable DI. It is a failure of the rule, and we report
it as one. Capacity strain is our reading of \emph{why} the checkpoint's LGC is
atypical, and the three grounds below are circumstantial rather than a controlled
test. (i)~At the same training
$\epsilon{=}8/255$, Mo2022's ViT-B yields LGC $1.00$ and CG-DI is correct, though
that comparison changes the training recipe along with the backbone and so does
not isolate capacity. (ii)~Salman \etal's own reported clean
accuracy for this checkpoint falls from $76.1\%$ ($\epsilon{=}0$) to $54.5\%$
($\epsilon{=}8$), a documented $21.6$-point capacity strain. (iii)~RobustBench and
\citet{mo2022adversarial} reserve higher-capacity backbones (WRN-70-16, ViT-B) for
$\epsilon\ge8/255$, using ResNet-50 only for $\epsilon\le4/255$.

What fails here is the rule we ship. The capacity strain lowers this checkpoint's
LGC from the ${\approx}0.98$ typical of robust models to $0.80$; since
$0.80<\tau{=}0.92$, CG-DI routes it as standard and enables DI on a surrogate DI
harms. That is a failure of the single fixed threshold, and the threshold stays
as it was frozen.

A per-surrogate threshold is not the repair. Tested on seven surrogates it was
not fitted to, the plug-in rule of \cref{sec:measured_constants} is correct on
only 2 of 7, so invoking it to recover this case would import its failures
elsewhere. It is also the wrong comparison to make: $\widehat\tau_{\text{E6}}$ is
defined on GCR and the routing statistic is computed from \cref{eq:lgc}, two
quantities \cref{tab:thresholds} keeps deliberately distinct. For the record the rule does place this checkpoint on the harmful side: its
measured GCR of $0.280$ exceeds its $\widehat\tau_{\text{E6}}$ of $0.140$
(\cref{tab:constants}). It does the same for all seven held-out surrogates,
five of which DI helps, so a single agreeing case is not evidence.
Re-fitting a threshold on the cases it fails would in any event convert a
held-out test into a fitted one, which is why $\tau$ was frozen before these
runs.

\section{CIFAR-100 in Full}
\label{sec:cifar100_full}

\Cref{tab:cifar100_full} is the per-surrogate record behind the CIFAR-100 summary
of \cref{sec:heldout}, at the same resolution as the datasets the rule succeeds
on. Two features explain the tally. First, the effect itself is small: at $r{=}0.9$ every
surrogate lies within $1.5$ pp of zero and three of the seven intervals contain
zero, so for those three there is no sign to get right. Second, the six robust
surrogates all sit at LGC ${\ge}0.994$, far above $\tau_{\text{op}}$, so the rule
predicts harm for all of them; where the measured effect is instead a small
positive number, the rule is scored wrong even though the magnitude is
negligible. Aggressive resize ($r{=}0.6$) restores the sign on two more
surrogates without restoring the effect size.

\begin{table}[h]
\caption{CIFAR-100 in full ($p{=}0.8$, $N{=}1{,}000$, 5 seeds, two to three CIFAR-100
targets per source, all-images basis). $D=$ DI$-$MI in percentage points, pooled
over targets and seeds; intervals are a $10{,}000$-replicate bootstrap over
(target, seed) pairs. ``pred'' is the frozen rule's decision from LGC alone;
``res.'' marks whether the $r{=}0.9$ interval excludes zero. The last column is
the rule's correctness at $r{=}0.9$ and $r{=}0.6$.}
\label{tab:cifar100_full}
\centering
\small
\begin{tabular}{llcccccc}
\toprule
Source & Type & LGC & pred & $D$ ($r{=}0.9$) & res. & $D$ ($r{=}0.6$) & ok \\
\midrule
Standard WRN-28-10 & Std & $0.868$ & $p{=}0.8$ & $+1.16$ $[+0.87,+1.45]$ & yes & $+2.18$ $[+2.01,+2.36]$ & Y/Y \\
Rice & Rob & $0.994$ & $p{=}0$ & $+1.45$ $[+1.16,+1.70]$ & yes & $+0.59$ $[+0.37,+0.81]$ & n/n \\
Hendrycks & Rob & $0.996$ & $p{=}0$ & $-0.02$ $[-0.20,+0.16]$ & no & $-1.45$ $[-2.00,-0.93]$ & Y/Y \\
Rebuffi & Rob & $0.997$ & $p{=}0$ & $-0.10$ $[-0.37,+0.21]$ & no & $-0.95$ $[-1.64,-0.26]$ & Y/Y \\
Pang & Rob & $0.997$ & $p{=}0$ & $+0.13$ $[-0.04,+0.29]$ & no & $-0.45$ $[-1.07,+0.16]$ & n/Y \\
Cui & Rob & $0.998$ & $p{=}0$ & $+0.70$ $[+0.39,+1.05]$ & yes & $+0.24$ $[-0.34,+0.80]$ & n/n \\
Wang & Rob & $0.998$ & $p{=}0$ & $+0.34$ $[+0.09,+0.58]$ & yes & $-0.06$ $[-0.78,+0.64]$ & n/Y \\
\midrule
\multicolumn{7}{l}{\emph{rule correct}} & 3/7, 5/7 \\
\bottomrule
\end{tabular}
\end{table}

\section{Multimodal Surrogate: CLIP Transfer}
\label{sec:clip}

We attack from CLIP ViT-B/32~\citep{radford2021learning} to robust CIFAR-10
targets (\cref{tab:clip}). CLIP's gradients are ``standard-like'' (LGC
$\approx0.72$; \cref{tab:frequency}), likely because contrastive pretraining does
not induce AT-level gradient smoothness, so CG-DI enables moderate diversity
($p{=}0.8$) and reaches near-oracle performance.

\begin{table}[h]
\caption{CLIP ViT-B/32 $\to$ robust CIFAR-10 targets ($N{=}1{,}000$, 5 seeds,
all-images basis). CLIP's LGC is standard-like, so CG-DI enables diversity, and $p{=}0.8$ is
the better of the two settings it chooses between.}
\label{tab:clip}
\centering
\small
\begin{tabular}{lcccc}
\toprule
Target & $p$=0 & $p$=0.5 & $p$=1.0 & CG-DI \\
\midrule
Engstrom & 13.9\% & 14.1\% & 14.5\% & 14.3\% \\
Rice & 16.1\% & 16.5\% & 16.6\% & 16.6\% \\
Wang-WRN & 8.4\% & 8.9\% & 9.0\% & 9.0\% \\
\bottomrule
\end{tabular}
\end{table}

\section{Extended ImageNet Architecture Breakdown}
\label{sec:imagenet_extended}

\Cref{tab:imagenet_extended} gives the per-target CG-DI results underlying the
averaged ImageNet metrics in \cref{tab:imagenet}, plus an extra target (ResNet50)
for the robust source.

\begin{table}[h]
\caption{ImageNet per-architecture breakdown ($N{=}1{,}000$, 5 seeds, all-images
basis; the per-target detail behind \cref{tab:imagenet}). CG-DI selects $p{=}0.8$
for the standard source and $p{=}0$ for the robust one, on every target.
$\Delta$ is the DI effect in pp, ASR$(p{=}1)-$ASR$(p{=}0)$, not CG-DI's gain. Bold marks
the best of the three settings in each row: on the standard source CG-DI's
$p{=}0.8$ is within about a point of $p{=}1$ rather than above it.}
\label{tab:imagenet_extended}
\centering
\small
\begin{tabular}{llcccc}
\toprule
Source & Target & $p$=0 & $p$=1 & CG-DI & $\Delta$ \\
\midrule
\multirow{4}{*}{ResNet50} & IncV3 & 51.9 & \textbf{66.5} & 66.1 & +14.6 \\
 & ViT & 34.9 & \textbf{49.2} & 48.1 & +14.3 \\
 & Swin & 40.9 & \textbf{54.1} & 53.5 & +13.2 \\
 & ConvNeXt & 48.3 & 64.6 & \textbf{64.9} & +16.3 \\
\midrule
\multirow{5}{*}{Engstrom} & IncV3 & \textbf{86.0} & 80.9 & \textbf{86.0} & -5.1 \\
 & ResNet50 & \textbf{81.3} & 68.7 & \textbf{81.3} & -12.6 \\
 & ViT & \textbf{82.3} & 75.5 & \textbf{82.3} & -6.8 \\
 & Swin & \textbf{68.6} & 55.5 & \textbf{68.6} & -13.1 \\
 & ConvNeXt & \textbf{66.9} & 50.7 & \textbf{66.9} & -16.2 \\
\bottomrule
\end{tabular}
\end{table}

\section{Model Preprocessing Details}
\label{sec:preprocessing}

\Cref{tab:preprocessing} documents input sizes and normalization for ImageNet
targets. All use standard ImageNet normalization; InceptionV3 requires
$299\times299$, others $224\times224$. CIFAR-10 RobustBench models use native
$32\times32$ inputs; torchvision architectures used for cross-architecture
experiments are upscaled to $224\times224$.

\begin{table}[h]
\centering
\caption{Model preprocessing configuration for ImageNet targets.}
\label{tab:preprocessing}
\small
\begin{tabular}{lccc}
\toprule
Model & Input size & Norm mean & Norm std \\
\midrule
ResNet50 & 224$\times$224 & (0.485, 0.456, 0.406) & (0.229, 0.224, 0.225) \\
InceptionV3 & \textbf{299$\times$299} & (0.485, 0.456, 0.406) & (0.229, 0.224, 0.225) \\
ViT-B/16 & 224$\times$224 & (0.485, 0.456, 0.406) & (0.229, 0.224, 0.225) \\
Swin-B & 224$\times$224 & (0.485, 0.456, 0.406) & (0.229, 0.224, 0.225) \\
ConvNeXt-B & 224$\times$224 & (0.485, 0.456, 0.406) & (0.229, 0.224, 0.225) \\
Engstrom (Robust) & 224$\times$224 & (0.485, 0.456, 0.406) & (0.229, 0.224, 0.225) \\
\bottomrule
\end{tabular}
\end{table}

\section{Additional Analysis}
\label{sec:additional_analysis}

\subsection{Computational Cost of LGC}
\label{sec:lgc_cost}
\Cref{tab:runtime_cost} benchmarks LGC ($K{=}5$) against a 10-step MI-FGSM attack.
Per-batch LGC adds $\sim$76\% overhead, reduced to $\sim$5\% with pre-computation
on a small calibration set (LGC is stable across set sizes;
\cref{tab:calibration_stability}).

\begin{table}[h]
\centering
\small
\caption{Runtime overhead of the CG-DI probe (batch size 32, $T{=}10$ attack
steps, wall-clock on one GPU, mean over 5 repetitions). Overhead is relative to
the same attack without the probe.}
\label{tab:runtime_cost}
\begin{tabular}{lcc}
\toprule
Component & Time (s) & Overhead \\
\midrule
LGC ($K{=}5$) & 0.239 & --- \\
Attack ($T{=}10$) & 0.316 & +75.6\% \\
\midrule
\multicolumn{3}{c}{\textit{Amortized ($N{=}1000$)}} \\
\midrule
Pre-comp (50 imgs) & 0.5 & --- \\
Attack & 10.1 & \textbf{+5.0\%} \\
\bottomrule
\end{tabular}
\end{table}

\begin{table}[h]
\centering
\caption{Stability of model-level LGC estimation across calibration set sizes
($K{=}5$ probes at $\sigma{=}1/255$). Each row is one fixed calibration subset of
the stated size; the interval is the mean and standard deviation \emph{over
images} within it, so it measures the spread of per-image LGC and does not shrink
with the subset size. What the table shows is that the mean is insensitive to the
size, not that the estimate has been resampled.}
\label{tab:calibration_stability}
\begin{tabular}{cc}
\toprule
Calibration set size & Estimated LGC \\
\midrule
$N{=}50$ & $0.9815 \pm 0.0091$ \\
$N{=}1000$ & $0.9812 \pm 0.0098$ \\
\bottomrule
\end{tabular}
\end{table}

\FloatBarrier
\subsection{Impact of Translation (TI-FGSM) on Robust Models}
\label{sec:translation_impact}
We compare a translation-only variant (TI-FGSM, kernel 5) to standard DI and the
naked baseline (\cref{tab:translation_impact}). Kernel-based translation
invariance degrades robust-model attacks by $-18.4$ pp, even more than standard DI
($-6.9$ pp), because the convolution kernel over-smooths already-consistent robust
gradients. This differs from the mild random (EOT) translation used in the
main-paper decomposition (\cref{tab:transform}), and reinforces that robust
gradients are highly sensitive to spatial structure.

\begin{table}[h]
\centering
\caption{Input transformations on a robust surrogate (Engstrom $\to$ Swin-B,
$N{=}500$, 3 seeds, all-images basis).}
\label{tab:translation_impact}
\begin{tabular}{llc}
\toprule
Method & Transform & ASR (\%) \\
\midrule
CG-DI ($p{=}0$) & none (naked) & \textbf{67.6} \\
DI-FGSM ($p{=}0.5$) & resize + translation & 60.7 \\
TI-FGSM & translation only (kernel) & 49.2 \\
\bottomrule
\end{tabular}
\end{table}

\subsection{Interpolation Mode Ablation (Ruling Out Artifacts)}
\label{sec:interp_ablation}
To rule out that the effect is tied to the bilinear interpolation in DI's resize,
we ablate three modes, namely bilinear (default), bicubic, and antialiased, while sweeping
$p\in\{0,0.3,0.5,0.7,1.0\}$, using \texttt{torchattacks}~\citep{torchattacks2020}
DIFGSM and modifying only the interpolation call (3 seeds, ImageNet,
$N{=}1{,}000$; \cref{tab:interp_ablation}).

\begin{table}[h]
\centering
\caption{Interpolation mode ablation ($N{=}1{,}000$, 3 seeds, all-images basis).
Average transfer ASR (\%) across the four ImageNet targets;
$\Delta$ is in pp relative to $p{=}0$ (44.0\% Standard, 76.0\% Robust).}
\label{tab:interp_ablation}
\small
\begin{tabular}{ll ccccc}
\toprule
Source & Interp mode & $p\!=\!0$ & $p\!=\!0.3$ & $p\!=\!0.5$ & $p\!=\!0.7$ & $p\!=\!1.0$ \\
\midrule
\multirow{3}{*}{Standard}
 & Bilinear  & 44.0 & 52.8\scs{+8.8} & 55.6\scs{+11.6} & 57.2\scs{+13.2} & 58.9\scs{+14.9} \\
 & Bicubic   & 44.0 & 51.7\scs{+7.7} & 54.2\scs{+10.2} & 56.0\scs{+12.0} & 57.3\scs{+13.3} \\
 & Antialias & 44.0 & 52.3\scs{+8.3} & 55.3\scs{+11.3} & 57.5\scs{+13.5} & 59.2\scs{+15.2} \\
\midrule
\multirow{3}{*}{Robust}
 & Bilinear  & 76.0 & 73.7\scs{-2.3} & 71.8\scs{-4.2} & 69.8\scs{-6.2} & 65.6\scs{-10.4} \\
 & Bicubic   & 76.0 & 75.1\scs{-0.9} & 74.6\scs{-1.4} & 73.7\scs{-2.3} & 70.9\scs{-5.1} \\
 & Antialias & 76.0 & 73.7\scs{-2.3} & 71.8\scs{-4.2} & 69.4\scs{-6.6} & 65.0\scs{-11.0} \\
\bottomrule
\end{tabular}
\end{table}

All three modes show the same qualitative pattern (DI helps Standard $+13$--$15$ pp,
harms Robust $-5$ to $-11$ pp): the \emph{direction} is invariant across all three
interpolation filters. The \emph{magnitude}, however, is not: bicubic roughly
halves the robust harm ($-5.1$ pp vs.\ bilinear $-10.4$ pp), while standard surrogates
are largely insensitive to interpolation choice ($<2$ pp variation). We read this as
support for, not against, the gradient-geometry account. A pure interpolation
artifact would have no reason to keep a fixed sign
across filters; what we instead see is a sign that never flips and a magnitude that
scales with how aggressively the filter suppresses high-frequency content: bicubic has a
gentler low-pass rolloff than bilinear, so on a robust surrogate's low-frequency
gradient it injects less resize bias and does less harm. That is exactly the
dependence the low-pass-bias mechanism (\cref{sec:bias_variance}) predicts: the harm
is set by the resize operator's contraction strength, not by an interpolation
idiosyncrasy. The Scissors Effect is thus a property of gradient geometry interacting
with a low-pass resize, with magnitude tunable by the filter and direction fixed.

\subsection{A Per-Image Negative Control: Sign-Alignment Does Not Mediate Individual Flips}
\label{sec:alignment_mediation}
The sign-alignment evidence of \cref{sec:alignment} is a \emph{population} statement:
averaged over images, DI moves a robust surrogate's gradient the wrong way and a
standard one's the right way. A natural follow-up question is whether it also acts as
a \emph{per-image} mediator: do the specific images whose transfer DI breaks have a
larger negative sign-alignment shift than the images it leaves alone? We tested this
directly and report a clean negative result, because it bounds what the
sign-alignment column should be read to claim.

For each clean-correct image (Engstrom$\to$Swin-B, $N{=}500$, 311 correct on both
models, 3 seeds, $\epsilon{=}16/255$) we paired (i) the per-image sign-alignment shift
$d_{\text{sign}}$ induced by DI (EOT$=10$, $r{=}0.9$, the same metric as
\cref{sec:alignment}) with (ii) the per-image transfer outcome under an actual attack
(MI-FGSM vs.\ DI-FGSM at $p{=}1$). The attack side reproduces the Scissors Effect at
the outcome level (net transfer change $-14.9$ pp robust, $+16.1$ pp standard; DI flips
$18.5\%$ of robust images from success to failure). But $d_{\text{sign}}$ does
not distinguish the flipped images: those DI breaks have mean
$d_{\text{sign}}{=}-0.0001$, statistically indistinguishable from the unchanged images
($+0.0001$; Mann--Whitney $p{=}0.43$), the correlation between $d_{\text{sign}}$ and a
signed outcome ($+1$ helped, $-1$ hurt) is null ($r{=}0.02$, $p{=}0.48$), and
$P(\text{DI hurts})$ is flat across $d_{\text{sign}}$ quintiles ($0.17$--$0.21$).

We read this straightforwardly: the clean-point sign-alignment shift is a
\emph{directional, population-level} signal, not a per-image cause of which
individual attack transfers, since the transferred outcome depends on the full
ten-step attack trajectory and not on the initial gradient alignment. Two things
in the main text follow from it. The quantitative case rests on the directly
measured bias--variance decomposition (\cref{tab:biasvar}), not on the norm ratio
of \cref{tab:alignment}, which can move for reasons other than variance
reduction; and the per-image link to ASR is supplied by the trajectory statistics
of \cref{tab:trajectory}, which is where the null above is resolved as a
timescale mismatch. The Scissors
mechanism is \emph{distributional}, a property of the surrogate's gradient regime
as a whole, and we do not claim it reduces to a per-image mediator. This negative
control is reported because the per-image test is the first one a skeptical
reader would run.

\subsection{Remark on High Baseline ASR}
\label{sec:high_asr_remark}
The high baseline ASR of the robust source (Engstrom) against standard targets
(avg 76.0\%) is a documented phenomenon: adversarially trained models learn
perceptually aligned, transferable features~\citep{salman2020adversarially}, and
robust models are known to be strong surrogates for transfer
attacks~\citep{springer2021little}. Our values use unmodified
RobustBench~\citep{croce2021robustbench} checkpoints.

\subsection{Threshold Sensitivity Analysis}
\label{sec:sensitivity}
We ablate $\tau\in[0.80,0.98]$ on ImageNet (target Swin-B, $N{=}500$, 3 seeds;
\cref{fig:sensitivity}) and identify a safe zone $\tau\in[0.86,0.96]$. The robust
source (Engstrom) holds $\sim$67.1\% across $\tau\in[0.80,0.96]$, dropping only at
$\tau{=}0.98$ (64.7\%); the standard source (ResNet50) rises from 51.5\%
($\tau{=}0.80$) to a $\sim$53.3\% plateau for $\tau\ge0.86$. Because CG-DI computes
LGC per image (or on a calibration batch), the model-level mean ($\approx0.64$ for
ResNet50) does not place all images below $\tau$; as $\tau$ decreases, more
individual images exceed it and are assigned $p{=}0$, explaining the gradual
standard-source transition.

\begin{figure}[h]
\centerline{\includegraphics[width=0.72\linewidth]{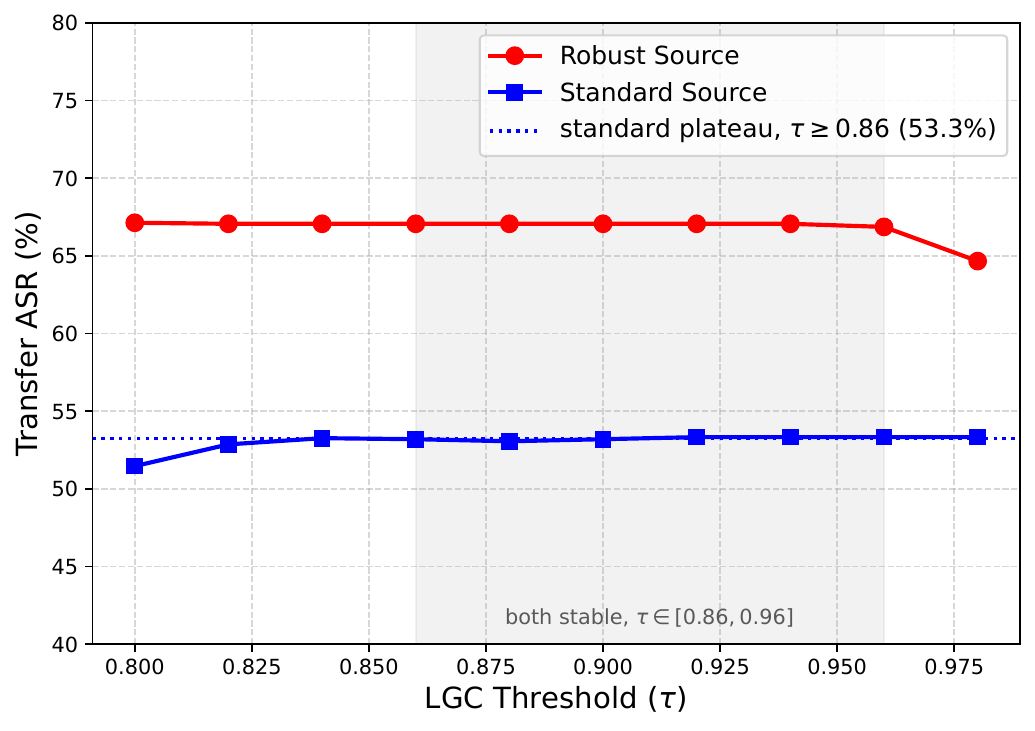}}
\caption{Threshold sensitivity of the \emph{per-image} rule on ImageNet (target
Swin-B, $N{=}500$, 3 seeds, all-images basis).
Transfer ASR vs.\ the LGC threshold $\tau$ for both source types (Engstrom in
red, ResNet-50 in blue). Both are simultaneously stable across
$\tau\in[0.86,0.96]$, with the default $\tau{=}0.92$ centred in that range. The
curves are smooth because each image is routed separately; the model-level rule
we recommend is a step function of $\tau$, so this figure does not measure its
robustness.}
\label{fig:sensitivity}
\end{figure}

\end{document}